\documentclass{article}

\usepackage[final]{arxiv}

\usepackage{graphicx}
\usepackage{subcaption}
\usepackage{adjustbox}
\usepackage{booktabs}
\usepackage{multirow}
\usepackage{wrapfig}

\newcommand{\best}[1]{\textbf{#1}}
\newcommand{\second}[1]{\underline{#1}}
\usepackage{subcaption}
\usepackage[utf8]{inputenc} % allow utf-8 input
\usepackage[T1]{fontenc}    % use 8-bit T1 fonts
\usepackage{hyperref}       % hyperlinks
\usepackage{url}            % simple URL typesetting
\usepackage{booktabs}       % professional-quality tables
\usepackage{amsfonts}       % blackboard math symbols
\usepackage{nicefrac}       % compact symbols for 1/2, etc.
\usepackage{microtype}      % microtypography
\usepackage{xcolor}         % colors
\definecolor{mich-blue-high}{HTML}{0027CC}
\long\def\blue#1{\bgroup\color{mich-blue-high}#1\egroup}
\long\def\red#1{\bgroup\color{red}#1\egroup}
\usepackage{multirow}
\usepackage{amsthm,amssymb,amsmath}
\newtheorem{theorem}{Theorem}[section]
\newtheorem{corollary}{Corollary}[theorem]
\newtheorem{lemma}[theorem]{Lemma}
\newtheorem{assumption}[theorem]{Assumption}
\newtheorem{proposition}[theorem]{Proposition}
\newtheorem{remark}[theorem]{Remark}

\usepackage{algorithm,algorithmic}
\usepackage{cleveref}
\usepackage{amsfonts}
\usepackage{wrapfig}
\usepackage{booktabs}  % for better tables
\usepackage{graphicx}
\graphicspath{{figs/}}

\usepackage{makecell}
\usepackage{adjustbox}
\usepackage[table]{xcolor}

\usepackage{subcaption}
\usepackage{placeins}

\definecolor{tracebg}{RGB}{232,242,252}
\definecolor{neuralbg}{RGB}{241,245,249}
\definecolor{diffbg}{RGB}{245,241,249}
\definecolor{cnnbg}{RGB}{226,242,228}
\definecolor{headergray}{RGB}{248,248,248}
\definecolor{traceblue}{RGB}{31,91,150}

\title{
Trajectory Constraints for Imaging Inverse Problems
}

\author{Chaoyan Huang\textsuperscript{1,2},~Haijie Yuan\textsuperscript{1}, Bharath Subramanian\textsuperscript{1}, Saiprasad Ravishankar\textsuperscript{1,3}\\ 
  \textsuperscript{1}Department of Computational Mathematics, Science, \& Engineering, Michigan State University 
  \textsuperscript{2}Department of Electrical Engineering and Computer Science, University of Michigan\\
  \textsuperscript{3}Department of Biomedical Engineering, Michigan State University \\
  \texttt{\{huang345, yuanhai2, subram89, ravisha3\}@msu.edu, \{chaoyanh\}@umich.edu}
}

\begin{document}

\maketitle

\begin{abstract}
%Diffusion-based and iterative methods have become effective tools for solving imaging inverse problems. Their reconstruction process naturally forms a trajectory of intermediate estimates. Although these intermediate estimates define a reconstruction trajectory, most methods do not explicitly regularize the transitions between consecutive states. To address this limitation, we introduce TRACE, a training-free \textbf{TRA}jectory-\textbf{C}onstrained r\textbf{E}construction framework that stabilizes the reconstruction path by coupling adjacent states along the trajectory. This gives a trajectory-level model that can be interpreted as a sequence of proximal updates. Since the exact proximal update is generally intractable, we approximate it with a neural mapping. This yields a diffusion-like reconstruction process with an explicit coupling between neighboring states. We provide a stability analysis showing that temporal coupling bounds trajectory variation and that this control is preserved under untrained network updates. Experiments on linear and nonlinear image reconstruction tasks show that TRACE improves reconstruction quality. Trajectory-level analyses and ablations confirm that temporal coupling directly affects state transitions along the reconstruction path.
Diffusion models have achieved strong performance in imaging inverse problems, but typically rely on domain-matched pretrained models and external training data. Moreover, errors introduced at intermediate reverse-sampling steps can accumulate and degrade the final reconstruction, with their propagation strongly influenced by how successive states are related. We propose TRACE, a training-free \textbf{TRA}jectory-\textbf{C}onstrained r\textbf{E}construction framework for imaging inverse problems. TRACE replaces the pretrained diffusion model with an untrained overparameterized network optimized directly from the observed measurements and organizes reconstruction as a sequence of intermediate states. At each step, TRACE enforces stochastic proximal consistency (SPC) by inferring each reconstruction from a stochastically perturbed predecessor under measurement consistency, while proximally constraining the resulting estimate to its corresponding unperturbed state. This formulation admits a proximal interpretation, from which we derive stability bounds characterizing the variation between consecutive states under approximate network updates. A local Jacobian analysis further characterizes how SPC regularizes the network-parameterized reconstruction. Experiments on natural-image restoration and CT reconstruction demonstrate competitive performance against both untrained and pretrained neural-prior methods. We further show that the proposed trajectory constraint can be incorporated into existing diffusion-based inverse solvers to improve reconstruction performance. 
\end{abstract}

\section{Introduction}

Many imaging tasks can be formulated as inverse problems, where the goal is to recover an unknown signal $\mathbf{x}_0$ from incomplete or indirect measurements $\mathbf{y}$ \citep{ravishankar2020image,lustig2007sparse,yang2008superresolution}. Because inverse problems are often non-injective, multiple signals may agree with the same measurements, making prior information essential. In a Bayesian formulation, reconstruction is characterized by the posterior
$p(\mathbf{x}_0\mid\mathbf{y}) \propto p(\mathbf{y}\mid\mathbf{x}_0)p(\mathbf{x}_0)$,
where $p(\mathbf{y}\mid\mathbf{x}_0)$ is the measurement model and $p(\mathbf{x}_0)$ is the signal prior. 

Diffusion models provide a powerful learned representation of this prior  \citep{ho2020ddpm,song2021score}. In score-based diffusion, reconstruction proceeds through a sequence of progressively denoised states $\mathbf{x}_T,\ldots,\mathbf{x}_0$ generated by a discretized reverse-time process. 
For inverse problems, the measurements additionally guide these updates, producing a reconstruction trajectory whose intermediate states gradually evolve toward a measurement-consistent solution \citep{chung2023dps,zhang2025improving}. Importantly, each reverse update is conditioned on the preceding state, thereby linking successive reconstructions along the trajectory. However, this approach has two limitations. First, diffusion-based inverse problem solvers typically rely on domain-matched pretrained models and external training data \citep{wang2023ddnm,kawar2022ddrm,alkhouri2025sitcom}. Second, errors introduced at intermediate reverse-sampling steps can accumulate across the reconstruction trajectory and degrade the final reconstruction \citep{alkhouri2025sitcom}. Recent work further shows that the manner in which successive states are related can critically affect the evolution of this trajectory \citep{zhang2025improving}. These drawbacks motivate a training-free reconstruction framework that adapts directly to the observed measurements without external training data, while explicitly regulating how successive reconstruction states evolve along the trajectory.

This motivates our first question: \textit{Can we retain the benefits of the diffusion-based reconstruction without relying on external training data?}
Note that deep Image Prior (DIP) offers a promising starting point by adapting an untrained network directly to the observed measurements \citep{ulyanov2018dip}. Moreover, DIP and diffusion sampling exhibit a related coarse-to-fine frequency progression, with lower-frequency structure emerging before finer details \citep{shi2021measuringcontrollingspectralbias,lee2025beta}. This similarity makes the untrained neural prior a natural candidate to replace the pretrained diffusion model in a diffusion-inspired reconstruction process. However, this replacement does not resolve how the trajectory of intermediate states should evolve. In diffusion-based solvers, the pretrained reverse model not only provides a learned prior but also governs the evolution between successive states. This raises our second question: \textit{How should the resulting intermediate reconstructions be constrained?}

To address these questions, we propose TRACE, a training-free \textbf{tra}jectory-\textbf{c}onstrained r\textbf{e}construction framework for imaging inverse problems.
TRACE replaces the pretrained diffusion model with an instance-adaptive untrained neural network optimized directly from the observed measurements. At each step, TRACE perturbs the current estimate and updates the network using both measurement fidelity and a penalty on the change between consecutive states. 
% TRACE models reconstruction as a trajectory of intermediate states, where stochastic perturbations are introduced within constrained transitions between consecutive reconstructions. 
The construction admits a proximal interpretation. After fixing one state, updating its predecessor corresponds to a proximal step involving measurement fidelity, implicit regularization, and stochastic
proximal consistency to the neighboring state. %TRACE approximates this proximal step using an untrained neural mapping optimized directly from the measurements. 
Our contributions are as follows.

\begin{itemize}
    \item We introduce TRACE, which retains the progressive stochastic reconstruction process of diffusion-based methods while using an instance-adaptive untrained network optimized directly from the measurements. This removes the dependence on external pretraining and enables explicit control of the reconstruction trajectory.

    \item %\textbf{Theoretical characterization of perturbation and coupling:}
    We derive a proximal interpretation of TRACE and establish stability bounds showing how stochastic proximal consistency controls transitions between consecutive states under approximate neural updates. We further develop a local Jacobian analysis that characterizes the stochastic proximal consistency and relate this interaction to the spectral bias of the untrained network.
    \item %\textbf{Extensive evaluation across inverse problems:}
    We evaluate TRACE across a range of imaging inverse problems and demonstrate competitive reconstruction performance against leading baselines. We further show that trajectory consistency can be incorporated into existing diffusion-based solvers to improve reconstruction performance.
\end{itemize}

\section{Background}

\paragraph{Problem setup.}
Consider the linear inverse problem $\mathbf{y}=\mathcal{A}\mathbf{x}$, where $\mathbf{x}\in\mathbb{R}^n$ is the unknown signal, $\mathbf{y}\in\mathbb{R}^m$ denotes the observed measurements, and $\mathcal{A}\in\mathbb{R}^{m\times n}$ is a known forward operator. In many practical settings, $m<n$, making the problem underdetermined and allowing multiple signals to be consistent with the same measurements. Regularization is therefore commonly used to incorporate prior assumptions about the unknown signal, restrict the solution space, and promote stable reconstructions \citep{benning2018modernregularizationmethodsinverse}. A standard variational formulation is

\[
\min_{\mathbf{x}}
\frac{1}{2}\|\mathbf{y}-\mathcal{A}\mathbf{x}\|_2^2
+
R(\mathbf{x}).
\]

where the first term enforces consistency with the observed measurements and $R$ is a regularizer encoding prior knowledge about $\mathbf{x}$. 
Existing reconstruction methods differ substantially in how such prior knowledge is represented and incorporated into the reconstruction. Existing reconstruction methods differ substantially in how prior knowledge is represented and incorporated into the reconstruction. Next, we provide a brief overview of three lines of work particularly relevant to our setting.

Proximal methods iteratively update the optimization variable by minimizing an objective together with a quadratic penalty that discourages large deviations from the preceding iterate \citep{moreau1965proximite,rockafellar1976monotone}. This gives a natural mechanism for controlling the evolution of an iterative reconstruction. Plug-and-play (PnP) methods replace the explicit proximal step with a generic denoiser, allowing classical denoisers such as BM3D and learned CNN denoisers to act as implicit image priors within model-based iterative reconstruction \citep{venkatakrishnan2013plug,Meinhardt_2017,zhang2021plug,Kamilov_2023}. Since a generic denoiser need not correspond to the proximal operator of an explicit regularizer, substantial work has studied the fixed-point and convergence properties of PnP algorithms, as well as constructions that recover an explicit optimization interpretation \citep{chan2016plugandplayadmmimagerestoration,ryu2019plugandplaymethodsprovablyconverge,hurault2022gradientstepdenoiserconvergent,hurault2022proximaldenoiserconvergentplugandplay,fang2024whatspriorlearnedproximal}. More recent work has connected PnP with stochastic and score-based formulations \citep{NEURIPS2021_6e289439,renaud2026plugandplayimagerestorationstochastic,Park_2025}. Diffusion models learn a data distribution by reversing a prescribed noising process using a time-dependent score or denoising model \citep{ho2020denoising,song2021score,karras2022elucidating}. For inverse problems, a pretrained diffusion model can serve as a generative prior, with the reverse sampling process modified to approximately sample from the posterior distribution conditioned on the observed measurements \citep{song2022medical,chung2023dps}. Existing approaches differ primarily in how measurement information is incorporated during sampling. Some methods approximate the noisy likelihood or posterior score \citep{chung2023dps,song2023pseudoinverse}, while others exploit the structure of linear forward operators through spectral-domain or projection-based updates \citep{kawar2021snips,kawar2022ddrm,wang2023ddnm}. Another line of work interleaves diffusion sampling with optimization or data-consistency steps \citep{zhu2023diffpir,alkhouri2025sitcom}. More recently, DAPS decouples consecutive sampling steps through a noise-annealing procedure, allowing larger corrections between successive states and improving robustness to errors accumulated during reverse sampling \citep{zhang2025improving}. Deep image prior (DIP) optimizes an overparameterized neural network directly from the observed measurements without requiring a training dataset \citep{ulyanov2018dip}. The network architecture and its optimization dynamics induce an implicit image prior that enables reconstruction from a single degraded observation. Subsequent methods have extended this framework through sequential optimization, self-guided inputs, diffusion-guided inputs, and transferable network parameters \citep{alkhouri2024image,liang2025analysis,liang2025sequential,liangugodit}. Related work has also combined untrained neural priors with diffusion-inspired reconstruction procedures \citep{luo2025selfdiffusion}. TRACE is positioned at the intersection of these three directions. It uses an untrained neural prior for instance-adaptive reconstruction, follows a stochastic trajectory inspired by diffusion, and introduces an explicit stochastic proximal consistency between consecutive states that admits a proximal interpretation.

\section{Method}

Motivated by the reverse-trajectory structure of diffusion models, we represent reconstruction as a sequence of states
$\{\mathbf{x}_t\}_{t=0}^{T}$, where $\mathbf{x}_T$ denotes the initial state and $\mathbf{x}_0$ the final reconstruction.
We adopt the backward indexing convention of diffusion and retain a prescribed noise schedule along this trajectory, while removing the pretrained diffusion model. To regulate the evolution of the resulting trajectory, we penalize large changes between neighboring reconstruction states.
Incorporating this proximal relation into the classical regularized inverse-problem formulation
$f(\mathbf{x})+\lambda R(\mathbf{x})$
gives

\begin{equation}
\min_{\{\mathbf{x}_t\}_{t=0}^T}
\sum_{t=0}^{T}
\left(
f(\mathbf{x}_t)
+
\lambda_t R(\mathbf{x}_t)
\right)
+
\sum_{t=0}^{T-1}
\frac{\beta_t}{2}
\|\mathbf{x}_{t}-\mathbf{x}_{t+1}\|_2^2,
\label{eq:trajectory_obj}
\end{equation}
where
$f(\mathbf{x})=\frac{1}{2}\|\mathcal{A}\mathbf{x}-\mathbf{y}\|_2^2$
enforces measurement consistency, $R(\cdot)$ denotes a regularization functional, and $\beta_t>0$ controls the strength of proximal consistency
between neighboring reconstruction states.

We do not solve \eqref{eq:trajectory_obj} as a full joint optimization problem over all states.
Instead, TRACE uses a backward one-step approximation induced by this trajectory energy.
Starting from $\mathbf{x}_T$, we construct the trajectory sequentially for $t=T-1,\ldots,0$ by fixing the already available state $\mathbf{x}_{t+1}$ and updating $\mathbf{x}_t$ through
\begin{equation}
\mathbf{x}_t
=
\arg\min_{\mathbf{x}}
f(\mathbf{x})
+
\lambda_t R(\mathbf{x})
+
\frac{\beta_t}{2}
\|\mathbf{x}-\mathbf{x}_{t+1}\|_2^2.
\label{eq:prox_subproblem}
\end{equation}
Letting
$F_t(\mathbf{x})=f(\mathbf{x})+\lambda_tR(\mathbf{x})$,
\eqref{eq:prox_subproblem} takes the form of a proximal update
\citep{moreau1965proximite,rockafellar1976monotone}
\begin{equation}
\mathbf{x}_t
=
P_t(\mathbf{x}_{t+1})
:=
\operatorname{prox}_{\frac{1}{\beta_t}F_t}
(\mathbf{x}_{t+1}).
\label{eq:prox_update}
\end{equation}

Since $R(\cdot)$ is not explicitly specified, the proximal map in
\eqref{eq:prox_update} cannot be evaluated directly.
Motivated by deep plug-and-play methods, where proximal updates are replaced by
neural denoisers \citep{zhang2021plug,wu2024extrapolated,huang2025deep},
we approximate this step using an untrained neural mapping.
Unlike standard plug-and-play methods that employ a fixed pretrained denoiser,
the network is optimized directly from the observed measurements at each
trajectory step, providing an instance-specific implicit regularization.

Rather than restricting this neural transition to the deterministic
reference state $\mathbf{x}_{t+1}$, TRACE extends the proximal
relation to a stochastic neighborhood of this reference.
Specifically, at each trajectory step, we sample a perturbed
input from a Gaussian neighborhood centered at
$\mathbf{x}_{t+1}$:
\begin{equation}
\mathbf{u}_t
=
\mathbf{x}_{t+1}
+
\sigma_t\boldsymbol{\epsilon}_t,
\qquad
\boldsymbol{\epsilon}_t\sim\mathcal{N}(\mathbf{0},\mathbf{I}),
\label{eq:perturbed_state}
\end{equation}
where $\sigma_t$ is determined by the prescribed noise schedule.
The next trajectory state is parameterized as
$\mathbf{x}_t=D_{\theta_t}(\mathbf{u}_t)$.

We summarize the resulting TRACE construction through two design principles that govern each trajectory step:
\begin{itemize}
    \item[\colorbox{yellow!20}{\textbf{P1}}]
    \textbf{Measurement Consistency.}
    The intermediate reconstruction is adapted to agree with the observed measurements.

    \item[\colorbox{cyan!15}{\textbf{P2}}]
    \textbf{Stochastic Proximal Consistency.}
    The resulting reconstruction is stochastically perturbed according to the prescribed noise schedule, while the updated reconstruction is constrained to remain close to its unperturbed state.

\end{itemize}

Together, these principles lead to the practical TRACE objective
\begin{equation}
\mathcal{L}_{\theta_t}
=
\colorbox{yellow!20}{
$\displaystyle
\frac{1}{2}
\left\|
\mathcal{A}
D_{\theta_t}\!\left(
\mathbf{u}_t
\right)
-\mathbf{y}
\right\|_2^2
$}
+
\colorbox{cyan!15}{
$\displaystyle
\frac{\beta_t}{2}
\left\|
D_{\theta_t}\!\left(
\mathbf{u}_t
\right)
-\mathbf{x}_{t+1}
\right\|_2^2
$},
\label{eq:training_loss}
\end{equation}

The trajectory is initialized by an uncoupled DIP reconstruction. Given a random input $\mathbf{z}$, the network weights $\theta_T^0$ are randomly initialized as in DIP \citep{ulyanov2018dip} and optimized for $K$ iterations:
\begin{equation}
\theta_T^K
\approx
\arg\min_{\theta}
\frac{1}{2}
\|\mathcal{A}D_{\theta}(\mathbf{z})-\mathbf{y}\|_2^2,
\qquad
\mathbf{x}_T=D_{\theta_T^K}(\mathbf{z}).
\label{eq:initial_dip}
\end{equation}
For each subsequent trajectory step, the network is warm-started from the preceding optimized weights, $\theta_t^0=\theta_{t+1}^K$, where $k=0,\ldots,K$ indexes the inner optimization iterations. The complete procedure is summarized in Algorithm~\ref{alg:trajectory_diffusion}.

\begin{algorithm}[t]
\caption{TRACE: \textbf{TRA}jectory-\textbf{C}onstrained r\textbf{E}construction framework}
\label{alg:trajectory_diffusion}
\begin{algorithmic}[1]
\STATE \textbf{Input:} $\mathbf{y}$, $\mathbf{x}_T$, $\theta_T^K$, trajectory steps $T$, inner network iterations $K$, forward operator $\mathcal{A}$, perturbation $\{\sigma_t\}_{t=0}^{T-1}$,  parameters $\{\beta_t\}_{t=0}^{T-1}$.
\FOR{$t=T-1,\ldots,0$}
\STATE Sample $\boldsymbol{\epsilon}_t\sim\mathcal{N}(0,I)$, 
$\mathbf{u}_t=\mathbf{x}_{t+1}+\sigma_t\boldsymbol{\epsilon}_t$.
\STATE Initialize $\theta_t^0=\theta_{t+1}^K$.
    \FOR{$k=1,\ldots,K$}
        \STATE {Update $\theta_t^k$ from $\theta_t^{k-1}$ by one optimizer step on $\mathcal{L}_{\theta_t}$ in \eqref{eq:training_loss}.}
        %\STATE Update $\theta_t^k$ by minimizing $\mathcal{L}_{\theta_t^k}$ in \eqref{eq:training_loss}. \sai{Forgot to add period at end of sentences. I added. Why is there superscript k for loss here? Do you mean 1 step of optimizer like GD to minimize $L_{\theta}$ in (6)?}
    \ENDFOR
    \STATE Set $\mathbf{x}_t = D_{\theta_t^K}(\mathbf{u}_t)$.
\ENDFOR
\RETURN $\mathbf{x}_0$
\end{algorithmic}
\end{algorithm}

\section{Theoretical Analysis}
\label{sec:theory}
\subsection{Effect of Proximal Consistency}
\label{sec:simple_analysis}
{
We analyze the local update in \eqref{eq:prox_subproblem} and study how
the proximal consistency term affects the change between two consecutive
reconstruction states. Recall that
$F_t(\mathbf{x})=f(\mathbf{x})+\lambda_t R(\mathbf{x})$, 
so that the local objective can be written as
\begin{equation}
\label{eq:local_energy}
    \Phi_t(\mathbf{x};\mathbf{x}_{t+1})
    =
    F_t(\mathbf{x})
    +
    \frac{\beta_t}{2}
    \|\mathbf{x}-\mathbf{x}_{t+1}\|_2^2.
\end{equation}
In TRACE, the practical state $\mathbf{x}_t=D_{\theta_t^K}(\mathbf{u}_t)$ is obtained from a finite neural optimization and therefore need not exactly minimize \eqref{eq:local_energy}. We measure the discrepancy between the ideal proximal model in \eqref{eq:prox_subproblem} and the implemented
neural update by the suboptimality gap
\begin{equation}
\label{eq:prox_gap}
    \varepsilon_t
    :=
    \Phi_t(\mathbf{x}_t;\mathbf{x}_{t+1})
    -
    \inf_{\mathbf{x}}
    \Phi_t(\mathbf{x};\mathbf{x}_{t+1}).
\end{equation}

\begin{proposition}
\label{prop:trajectory_control}
Suppose that $F_t$ is bounded below and that $F_t(\mathbf{x}_{t+1})<\infty$. Let
$\underline{F}_t:=\inf_{\mathbf{x}}F_t(\mathbf{x})>-\infty$. 
For any $\beta_t>0$, the state $\mathbf{x}_t$ satisfies
\begin{equation}
\label{eq:trajectory_energy_bound}
    \frac{\beta_t}{2}
    \|\mathbf{x}_t-\mathbf{x}_{t+1}\|_2^2
    \leq
    F_t(\mathbf{x}_{t+1})
    -
    F_t(\mathbf{x}_t)
    +
    \varepsilon_t.
\end{equation}
Consequently,
\begin{equation}
\label{eq:trajectory_distance_bound}
    \|\mathbf{x}_t-\mathbf{x}_{t+1}\|_2
    \leq
    \sqrt{
    \frac{
    2\left(
    F_t(\mathbf{x}_{t+1})
    -
    \underline{F}_t
    +
    \varepsilon_t
    \right)
    }{
    \beta_t
    }}.
\end{equation}
\end{proposition}

\begin{proof}
From the definition of $\varepsilon_t$ in
\eqref{eq:prox_gap},
\begin{equation}
\Phi_t(\mathbf{x}_t;\mathbf{x}_{t+1})
    =
    \inf_{\mathbf{x}}
    \Phi_t(\mathbf{x};\mathbf{x}_{t+1})
    +
    \varepsilon_t
    \leq
    \Phi_t(\mathbf{x}_{t+1};\mathbf{x}_{t+1})
    +
    \varepsilon_t
    =
    F_t(\mathbf{x}_{t+1})
    +
    \varepsilon_t,\nonumber
\end{equation}
where the second inequality follows by using
$\mathbf{x}_{t+1}$ as a feasible point.
Substituting \eqref{eq:local_energy} gives
% \begin{equation}\nonumber
%     F_t(\mathbf{x}_t)
%     +
%     \frac{\beta_t}{2}
%     \|\mathbf{x}_t-\mathbf{x}_{t+1}\|_2^2
%     \leq
%     F_t(\mathbf{x}_{t+1})
%     +
%     \varepsilon_t,
% \end{equation}
% which gives 
\eqref{eq:trajectory_energy_bound}.
Since
$F_t(\mathbf{x}_t)\geq\underline{F}_t$,
\eqref{eq:trajectory_distance_bound} follows directly.
\end{proof}

% Proposition~\ref{prop:trajectory_control} gives a direct bound on the change between consecutive reconstruction states induced by proximal consistency. %When the objective value and optimization error are fixed, a larger $\beta_t$ gives a smaller upper bound on $\|\mathbf{x}_t-\mathbf{x}_{t+1}\|_2$. 
% This behavior is also consistent with the trajectory observed in Fig.~\ref{fig:trajectory_stability}, where the state difference $\Delta_t=\|\mathbf{x}_t-\mathbf{x}_{t+1}\|_2$ gradually decreases as the reconstruction proceeds, while both PSNR and SSIM improve and eventually stabilize. This suggests that TRACE progressively reduces large state changes while still allowing the reconstruction to be refined.

Proposition~\ref{prop:trajectory_control} provides direct control on the change between two consecutive reconstruction states, up to the local suboptimality $\varepsilon_t$.
It does not require the realized trajectory difference to decrease monotonically across reconstruction steps, since the local objective, coupling strength, and approximation error may vary with $t$. In Appendix~\ref{app:abs}, we empirically compare the coupled and uncoupled trajectories under the same stochastic perturbation schedule.

%The result provides a simple interpretation of temporal coupling as a method for controlling the reconstruction trajectory.

% Stronger stability results under additional assumptions are provided in Appendix~\ref{app:proofs}.
When $\varepsilon_t=0$, the result reduces to the exact proximal update in \eqref{eq:prox_subproblem}. Stronger results for the ideal proximal trajectory and the propagation of approximation errors are given in Appendix~\ref{app:temporal_coupling}.

\begin{remark}
\label{rem:no_coupling}
The bound above requires $\beta_t>0$. When $\beta_t=0$, the local objective in \eqref{eq:prox_subproblem} contains no term involving $\|\mathbf{x}-\mathbf{x}_{t+1}\|_2$, and the local objective alone does not control the change relative to $\mathbf{x}_{t+1}$.
\end{remark}

\subsection{Local Analysis of Stochastic Proximal Consistency}
While Proposition~\ref{prop:trajectory_control} characterizes how proximal
consistency controls the change between consecutive reconstruction states,
SPC applies this constraint when the network is driven by a stochastically
perturbed trajectory state. We next analyze the local regularization induced
by this stochastic proximal formulation.

\begin{proposition}
\label{prop:interaction}
Fix a trajectory state $x$ and network parameters $\theta$ where
$J_D(x)=\left.\frac{\partial D_\theta(u)}{\partial u}\right|_{u=x}$
denotes the input Jacobian of $D_\theta$ at $x$.
Consider the first-order local approximation
\[
D_\theta(x+\sigma\epsilon)
\approx
D_\theta(x)+\sigma J_D(x)\epsilon,
\qquad
\epsilon\sim\mathcal N(0,I).
\]
Then the expected locally linearized TRACE objective satisfies
\[
\begin{aligned}
\mathbb E_\epsilon[\mathcal L]
\approx{}&
\underbrace{
\frac{1}{2}\|A D_\theta(x)-y\|_2^2
+
\frac{\beta}{2}\|D_\theta(x)-x\|_2^2
}_{\text{unperturbed TRACE objective}}
+
\underbrace{
\frac{\sigma^2}{2}\|A J_D(x)\|_F^2
}_{\text{measurement-space sensitivity}}
+
\underbrace{
\frac{\beta\sigma^2}{2}\|J_D(x)\|_F^2
}_{\text{SPC-induced regularization}}.
\end{aligned}
\]
\end{proposition}

The full proof is given in Appendix \ref{app:interaction_derivation}. 
The first two terms in Proposition~\ref{prop:interaction} recover the unperturbed TRACE objective, while the remaining terms characterize the local effect of stochastic perturbation. The measurement-space sensitivity term arises through the data-fidelity term and controls perturbation sensitivity after application of the forward operator. In contrast, the SPC-induced regularization acts directly on the full Jacobian, thereby controlling local perturbation sensitivity in the reconstruction space. To understand how this regularization acts on measurement-visible and measurement-invisible directions, consider
\[
R_{\mathrm{pert}}
=
\frac{\sigma^2}{2}\|A J_D(x)\|_F^2
+
\frac{\beta\sigma^2}{2}\|J_D(x)\|_F^2.
\]
Let $P_{\mathcal R}$ and $P_{\mathcal N}$ denote the orthogonal projections
onto $\operatorname{range}(A^\top)$ and $\operatorname{null}(A)$,
respectively. Since $P_{\mathcal R}+P_{\mathcal N}=I$ and
$A P_{\mathcal N}=0$,
\[
\begin{aligned}
R_{\mathrm{pert}}
={}&
\frac{\sigma^2}{2}
\|A P_{\mathcal R}J_D(x)\|_F^2
+
\frac{\beta\sigma^2}{2}
\|P_{\mathcal R}J_D(x)\|_F^2
+
\frac{\beta\sigma^2}{2}
\|P_{\mathcal N}J_D(x)\|_F^2.
\end{aligned}
\]
The stochastic extension of the proximal relation induces an additional null-space regularization term
\(\frac{\beta\sigma^2}{2}\|P_{\mathcal N}J_D(x)\|_F^2\),
which is absent when proximal coupling is removed. This motivates the null-space analysis in
Section~\ref{sec:nullspace_exp}, where we further examine the source of the
reconstruction improvement by evaluating the PSNR of the reconstructions
projected onto $\operatorname{null}(A)$. We also approximately estimate the
denoiser Jacobian and verify that SPC suppresses its null-space component,
consistent with the additional regularization predicted by the analysis and
absent under stochastic perturbation alone.

Another way to interpret this proposition is through spectral bias. Untrained network
priors are known to learn low-frequency image components before
higher-frequency details \citet{shi2021measuringcontrollingspectralbias}. Jacobian regularization has been used to promote
local smoothness by reducing sensitivity to input perturbations
\citep{hoffman2019robustlearningjacobianregularization},
while controlling function smoothness has also been shown to
influence the spectral bias of untrained network priors
\citep{liu2024architectureagnosticuntrainednetworkpriors}. The SPC-induced
term directly regularizes this input-output Jacobian. Thus, beyond the
regularization induced by stochastic perturbation alone, SPC provides an
additional mechanism for modulating the network's inherent spectral bias, with
strength jointly controlled by $\beta$ and $\sigma^2$. This motivates the spectral-bias analysis in
Section~\ref{sec:spectralbias_exp}, where we use frequency-band
correspondence to empirically examine how SPC changes the
low-to-high frequency progression relative to TRACE ablation variants.

}

\section{Experiments}
\label{sec:experiments}

We design the experiments around the two questions motivating TRACE: whether competitive reconstruction can be achieved without relying on domain-matched pretrained priors, and how explicitly constraining successive states affects the reconstruction process.

We first consider CT reconstruction, where several competing methods have access to domain-specific pretrained diffusion models or supervised reconstruction networks. 
We then evaluate natural-image inverse problems, where strong pretrained diffusion and CNN priors are readily available.
Together, these settings evaluate whether TRACE can remain competitive without external training across domains where competing methods benefit from domain-specific learned priors.

Beyond final reconstruction performance, we provide controlled analyses of stochastic perturbation and proximal coupling in the appendix to examine how SPC affects the reconstruction trajectory and how the observed behavior relates to Section~\ref{sec:theory}.
Finally, we add the same SPC constraint to several pretrained diffusion-based inverse solvers to test whether its effect extends to other reconstruction frameworks.
Additional ablations, sensitivity studies, and runtime comparisons are provided in the Appendix \ref{app:additional_ablation}.

\subsection{Experimental Setup}
\label{sec:experimental_setup}
We evaluate TRACE on two CT reconstruction tasks, sparse-view and limited-angle CT, and three natural-image restoration tasks: inpainting, super-resolution, and deblurring. For CT, we use 24 abdominal slices from the AAPM dataset. Sparse-view CT uses 60 projection views over 180°, whereas limited-angle CT uses 120 projection views within the angular range 0°--119°. For natural-image restoration, we use 14 test images from BSDS68 and Set14. 
We compare TRACE with plug-and-play methods using pretrained CNN
denoisers: PnP-DnCNN~\cite{zhang2017beyond},
PnP-FFDNet~\cite{zhang2018ffdnet}, and
IRCNN~\cite{zhang2017learning}. For CT, we also evaluate the
supervised networks RED-CNN~\cite{chen2017low} and
EDCNN~\cite{liang2020edcnn}.
Methods using pretrained diffusion priors include
SITCOM~\cite{alkhouri2025sitcom}, DPS~\cite{chung2023dps},
DAPS~\cite{zhang2025improving},
RED-Diff~\cite{mardani2024variational},
PnP-DM~\cite{wu2024principled}, and
DDS~\cite{chung2023decomposed}.
The hybrid method uDiG-DIP~\cite{liang2025sequential}
combines diffusion purification with DIP optimization.
Methods that optimize a randomly initialized network for each test
instance include DIP~\cite{ulyanov2018dip},
aSeqDIP~\cite{alkhouri2024image},
Self-Guided DIP~\cite{liang2025analysis}, and
SDI~\cite{luo2025selfdiffusion}.
We also include UGoDIT~\cite{liangugodit}, which uses transferable
weights learned from a group of degraded measurements.

For each task, the evaluated methods use the same test images,
measurement realizations, and forward operator. Unless stated
otherwise, TRACE uses $T=40$ trajectory states and $K=150$ inner
optimization steps per state. Further experimental settings and
runtime comparisons are provided in Appendix~\ref{app:algo}.

\subsection{CT Reconstruction}
\label{sec:ct_results}
We first consider CT reconstruction, where the compared methods use different sources of prior information.
DDS and uDiG-DIP use diffusion models pretrained on CT images, while RED-CNN and EDCNN are supervised on CT data.
IRCNN uses a CNN denoiser pretrained on natural images, and SITCOM uses a diffusion model pretrained on natural images. TRACE and other compared neural-prior methods, in contrast, are optimized directly from the observed measurements and do not require external training data or a pretrained model.
This setting allows us to compare training-free reconstruction with both domain-matched learned priors and generic pretrained priors.

\begin{table}[t]
\centering
\small
\setlength{\tabcolsep}{3.6pt}
\renewcommand{\arraystretch}{1.08}

\caption{
Average results on CT reconstruction. Best and second-best values in each row are shown in \textbf{bold} and \underline{underlined}, respectively.
%Higher is better for PSNR and SSIM; lower is better for LPIPS.
DDS and uDiG-DIP use CT-domain pretrained diffusion priors, whereas
SITCOM uses a natural-image pretrained diffusion prior.
IRCNN uses a pretrained natural-image CNN denoiser, while RED-CNN and
EDCNN are supervised CT reconstruction networks.
%LPIPS uses the AlexNet variant on the grayscale slice replicated to three channels and is reported for reference.
}
\label{tab:ct}

\begin{adjustbox}{max width=\textwidth}
\begin{tabular}{
@{}
ll
ccc
@{\hspace{5pt}}
ccc
@{\hspace{5pt}}
cccc
>{\columncolor{tracebg}}c
@{}
}
\toprule

% ------------------------------------------------------------
% Group header
% ------------------------------------------------------------
\rowcolor{headergray}
\multirow{2}{*}{\textbf{Setting}}
&
\multirow{2}{*}{\textbf{Metric}}
&
\multicolumn{3}{c}{\cellcolor{cnnbg}\textbf{CNN-based methods}}
&
\multicolumn{3}{c}{\cellcolor{diffbg}\textbf{Diffusion-prior methods}}
&
\multicolumn{5}{c}{\cellcolor{neuralbg}\textbf{Neural-prior methods}}
\\

\cmidrule(lr){3-5}
\cmidrule(lr){6-8}
\cmidrule(lr){9-13}

\rowcolor{headergray}
&
&
IRCNN \hspace{-0.05in}
&
RED-CNN \hspace{-0.05in}
&
EDCNN \hspace{-0.05in}
&
DDS \hspace{-0.05in}
&
uDiG-DIP \hspace{-0.05in}
&
SITCOM \hspace{-0.05in}
&
DIP \hspace{-0.05in}
&
ASeqDIP \hspace{-0.05in}
&
\makecell{Self-Guided\\DIP} \hspace{-0.05in}
&
SDI \hspace{-0.05in}
&
\cellcolor{tracebg}\textcolor{traceblue}{\textbf{TRACE}}
\\

\midrule

% ------------------------------------------------------------
% Sparse-view
% ------------------------------------------------------------
\multirow{3}{*}{\textbf{Sparse-view}}
&
PSNR$\uparrow$
&
29.44
&
29.32
&
29.11
&
\best{30.55}
&
27.79
&
21.44
&
25.25
&
27.15
&
27.43
&
24.36
&
\second{29.55}
\\

&
SSIM$\uparrow$
&
\second{0.873}
&
0.841
&
0.825
&
0.832
&
0.807
&
0.746
&
0.691
&
0.822
&
0.809
&
0.789
&
\best{0.894}
\\

&
LPIPS$\downarrow$
&
\second{0.055}
&
0.062
&
\best{0.050}
&
0.124
&
0.138
&
0.174
&
0.211
&
0.127
&
0.103
&
0.153
&
0.057
\\

\midrule

% ------------------------------------------------------------
% Limited-angle
% ------------------------------------------------------------
\multirow{3}{*}{\textbf{Limited-angle}}
&
PSNR$\uparrow$
&
26.22
&
24.91
&
23.73
&
25.51
&
25.12
&
18.67
&
23.36
&
25.64
&
\second{26.33}
&
25.36
&
\best{27.34}
\\

&
SSIM$\uparrow$
&
\best{0.806}
&
0.766
&
0.719
&
0.797
&
0.733
&
0.600
&
0.626
&
0.760
&
0.782
&
0.768
&
\second{0.804}
\\

&
LPIPS$\downarrow$
&
0.108
&
0.115
&
\second{0.104}
&
\best{0.082}
&
0.196
&
0.240
&
0.276
&
0.164
&
0.113
&
0.151
&
0.128
\\

\bottomrule
\end{tabular}
\end{adjustbox}
\end{table}

Table~\ref{tab:ct} shows that TRACE remains competitive with methods using externally learned priors. For sparse-view CT, DDS achieves the highest PSNR, while TRACE obtains the second-best one.
DDS benefits from a domain-matched pretrained diffusion prior, whereas TRACE uses no CT training data or pretrained reconstruction model. Despite this difference, TRACE remains within $1.00$ dB of DDS and achieves the highest SSIM.
For limited-angle CT, TRACE obtains the highest PSNR, $1.01$ dB above the second-best one, while our SSIM is close to the best value obtained by IRCNN.

The comparison with SITCOM further illustrates the role of domain matching. SITCOM performs strongly on natural images when its pretrained prior is matched to the image domain, as shown in the next subsection~\ref{sec:natural_results}.
When the same natural-image prior is directly applied to CT, its PSNR drops to $21.44$ dB for sparse-view and $18.67$ dB for limited-angle reconstruction. This is different from DDS, which uses a CT-pretrained diffusion prior and performs strongly on sparse-view CT.
TRACE does not require a domain-specific pretrained model and is optimized directly from the measurements. 

\providecommand{\ctcell}[3]{}
\renewcommand{\ctcell}[3]{%
\begin{minipage}[t]{0.155\textwidth}
    \centering
    \includegraphics[width=\linewidth]{#1}\par
    \vspace{-0.04in}
    {\scriptsize #2}\par
    \vspace{-0.04in}
    {\scriptsize #3}
\end{minipage}%
}

\providecommand{\ctcelltrace}[3]{}
\renewcommand{\ctcelltrace}[3]{%
\begin{minipage}[t]{0.155\textwidth}
    \centering
    \includegraphics[width=\linewidth]{#1}\par
    \vspace{-0.04in}
    {\scriptsize\textcolor{traceblue}{\textbf{#2}}}\par
    \vspace{-0.04in}
    {\scriptsize\textbf{#3}}
\end{minipage}%
}

\begin{figure}[t]
\centering
\setlength{\tabcolsep}{1.5pt}
\renewcommand{\arraystretch}{1.0}

\begin{tabular}{@{}cccccc@{}}

\ctcell{
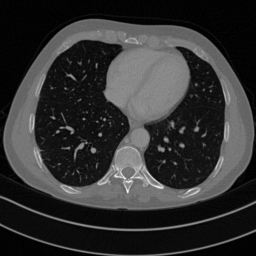}
{Ground Truth (GT)}
{}
&
\ctcell{
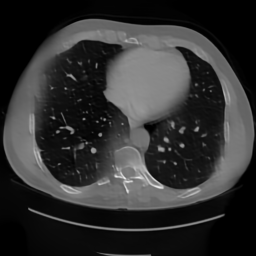}
{IRCNN}
{\fontsize{6.4}{7.5}\selectfont \underline{24.74} / 0.779 / 0.131}
&
\ctcell{
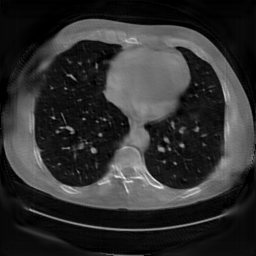}
{RED-CNN}
{\fontsize{6.4}{7.5}\selectfont 23.74 / 0.761 / 0.136}
&
\ctcell{
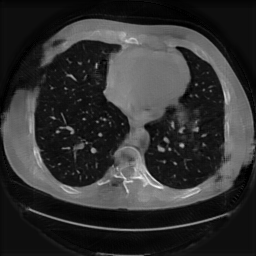}
{EDCNN}
{\fontsize{6.4}{7.5}\selectfont 22.51 / 0.724 / 0.127}
&
\ctcell{
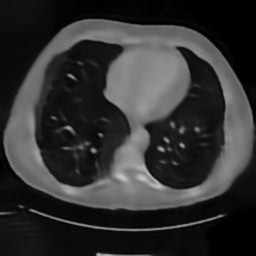}
{uDiG-DIP}
{\fontsize{6.4}{7.5}\selectfont 24.47 / 0.687 / 0.211}
&
\ctcell{
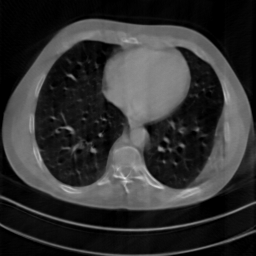}
{DDS}
{\fontsize{6.4}{7.5}\selectfont 24.71 / \underline{0.783} / \textbf{0.098}}
\\[0.23in]

\ctcell{
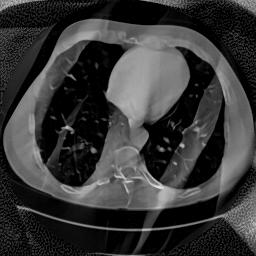}
{SITCOM}
{\fontsize{6.4}{7.5}\selectfont 17.31 / 0.451 / 0.324}
&
\ctcell{
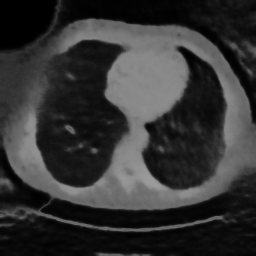}
{DIP}
{\fontsize{6.4}{7.5}\selectfont 21.51 / 0.573 / 0.304}
&
\ctcell{
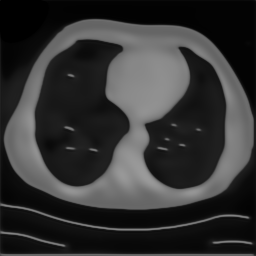}
{ASeqDIP}
{\fontsize{6.4}{7.5}\selectfont 22.25 / 0.657 / 0.236}
&
\ctcell{
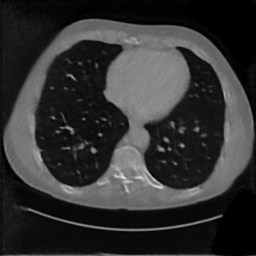}
{Self-Guided DIP}
{\fontsize{6.4}{7.5}\selectfont 24.58 / 0.730 / 0.129}
&
\ctcell{
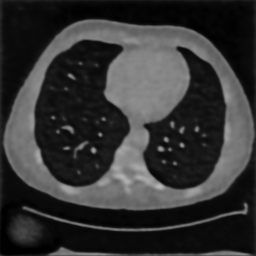}
{SDI}
{\fontsize{6.4}{7.5}\selectfont 23.41 / 0.698 / 0.196}
&
\ctcelltrace{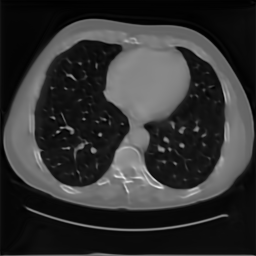}
{TRACE}
{\fontsize{6.4}{7.5}\selectfont 25.75 / 0.824 / {\mdseries\underline{0.116}}}

\end{tabular}
\caption{
Representative limited-angle CT reconstructions (PSNR$\uparrow$ / SSIM$\uparrow$/ LPIPS$\downarrow$). 
TRACE preserves anatomical structures more clearly and achieves the highest PSNR and SSIM in this example.
}
\label{fig:ct_qual}

\end{figure}

Figure~\ref{fig:ct_qual} shows a representative limited-angle reconstruction.
The cross-domain SITCOM reconstruction largely fails to recover the underlying CT anatomy, with severe structural distortions caused by the mismatch between its natural-image prior and the CT domain.
DDS, despite using a CT-domain pretrained diffusion prior, still exhibits visible streak and texture artifacts. TRACE has better preserved anatomical structures and achieves the highest PSNR and SSIM. The averages in Table~\ref{tab:ct} are computed over 24 test slices, additional settings and results are provided in Appendix~\ref{app:additional_ct}.

\subsection{Natural-Image Inverse Problems}
\label{sec:natural_results}

\begin{table}[t]
\centering
\small
\setlength{\tabcolsep}{3.0pt}
\renewcommand{\arraystretch}{1.08}

\caption{
Average results on natural-image inverse problems.
Best and second-best values for each setting and metric are shown in
\textbf{bold} and \underline{underlined}, respectively.
UGoDIT$^{*}$ uses transferable weights learned from unsupervised group training.
PnP-DnCNN and PnP-FFDNet plug pretrained CNN denoisers into a
plug-and-play proximal-gradient solver.
}
\label{tab:natural_main_extended}

\begin{adjustbox}{max width=\textwidth}
\begin{tabular}{
@{}
lll
cc
@{\hspace{5pt}}
ccccc
@{\hspace{5pt}}
ccccc
>{\columncolor{tracebg}}c
@{}
}
\toprule

% ============================================================
% Group header
% ============================================================
\rowcolor{headergray}
\multirow{2}{*}{\textbf{Task}}
&
\multirow{2}{*}{\textbf{Setting}}
&
\multirow{2}{*}{\textbf{Metric}}
&
\multicolumn{2}{c}{\cellcolor{cnnbg}\textbf{CNN-denoiser methods}}
&
\multicolumn{5}{c}{\cellcolor{diffbg}\textbf{Diffusion-prior methods}}
&
\multicolumn{6}{c}{\cellcolor{neuralbg}\textbf{Neural-prior methods}}
\\

\cmidrule(lr){4-5}
\cmidrule(lr){6-10}
\cmidrule(lr){11-16}

\rowcolor{headergray}
&
&
&
PnP-DnCNN \hspace{-0.05in}
&
PnP-FFDNet \hspace{-0.05in}
&
SITCOM \hspace{-0.05in}
&
DPS \hspace{-0.05in}
&
DAPS \hspace{-0.05in}
&
RED-diff \hspace{-0.05in}
&
PnP-DM \hspace{-0.05in}
&
DIP \hspace{-0.05in}
&
ASeqDIP \hspace{-0.05in}
&
\makecell{Self-Guided\\DIP} \hspace{-0.05in}
&
SDI \hspace{-0.05in}
&
UGoDIT$^{*}$ \hspace{-0.05in}
&
\cellcolor{tracebg}\textcolor{traceblue}{\textbf{TRACE}}
\\

\midrule

% ============================================================
% Inpainting
% ============================================================

\multirow{6}{*}{\textbf{Inpainting}}
&
\multirow{3}{*}{Missing 50\%}
&
PSNR$\uparrow$
&
31.45
&
29.67
&
33.60
&
30.80
&
31.48
&
31.82
&
29.11
&
29.08
&
31.84
&
\best{33.91}
&
32.53
&
30.93
&
\second{33.77}
\\

&
&
SSIM$\uparrow$
&
0.943
&
0.926
&
0.940
&
0.928
&
0.922
&
0.929
&
0.882
&
0.832
&
0.919
&
\best{0.951}
&
0.938
&
0.894
&
\second{0.948}
\\

&
&
LPIPS$\downarrow$
&
0.0668
&
0.1010
&
0.0372
&
0.0373
&
0.0556
&
0.0762
&
0.1128
&
0.2299
&
0.0773
&
0.0337
&
\best{0.0315}
&
0.0778
&
\second{0.0316}
\\

\cmidrule(lr){2-16}

&
\multirow{3}{*}{Missing 70\%}
&
PSNR$\uparrow$
&
28.53
&
26.82
&
29.31
&
28.00
&
28.79
&
29.11
&
25.82
&
27.94
&
29.14
&
\second{30.61}
&
29.27
&
28.96
&
\best{30.68}
\\

&
&
SSIM$\uparrow$
&
0.889
&
0.864
&
0.855
&
0.876
&
0.876
&
0.891
&
0.837
&
0.805
&
0.871
&
\best{0.909}
&
0.886
&
0.864
&
\second{0.903}
\\

&
&
LPIPS$\downarrow$
&
0.1306
&
0.1765
&
0.0883
&
0.0718
&
0.1024
&
0.1171
&
0.1618
&
0.2499
&
0.1166
&
0.0787
&
\second{0.0600}
&
0.0938
&
\best{0.0553}
\\

\midrule

% ============================================================
% Super-resolution
% ============================================================

\multirow{6}{*}{\textbf{SR}}
&
\multirow{3}{*}{$\times 2$}
&
PSNR$\uparrow$
&
\second{32.07}
&
30.98
&
\best{32.55}
&
29.27
&
30.42
&
29.65
&
25.99
&
26.98
&
28.51
&
28.42
&
31.28
&
29.63
&
32.01
\\

&
&
SSIM$\uparrow$
&
\best{0.933}
&
0.924
&
0.929
&
0.887
&
0.893
&
0.885
&
0.835
&
0.748
&
0.827
&
0.777
&
0.920
&
0.872
&
\second{0.931}
\\

&
&
LPIPS$\downarrow$
&
0.0812
&
0.1055
&
\best{0.0416}
&
\second{0.0515}
&
0.0700
&
0.1254
&
0.1446
&
0.3285
&
0.2519
&
0.1374
&
0.0772
&
0.0926
&
0.0755
\\

\cmidrule(lr){2-16}

&
\multirow{3}{*}{$\times 4$}
&
PSNR$\uparrow$
&
27.09
&
26.78
&
\best{27.50}
&
25.08
&
26.32
&
24.63
&
23.80
&
26.11
&
25.88
&
22.26
&
26.32
&
25.91
&
\second{27.47}
\\

&
&
SSIM$\uparrow$
&
0.805
&
\second{0.808}
&
0.806
&
0.773
&
0.795
&
0.762
&
0.762
&
0.723
&
0.744
&
0.478
&
0.783
&
0.755
&
\best{0.817}
\\

&
&
LPIPS$\downarrow$
&
0.2661
&
0.2846
&
\second{0.1655}
&
\best{0.1304}
&
0.1902
&
0.2762
&
0.2533
&
0.3828
&
0.3500
&
0.5732
&
0.2713
&
0.2373
&
0.2484
\\

\midrule

% ============================================================
% Deblurring
% ============================================================

\multirow{6}{*}{\textbf{Deblurring}}
&
\multirow{3}{*}{Motion}
&
PSNR$\uparrow$
&
27.95
&
27.00
&
\best{33.89}
&
\second{30.03}
&
27.82
&
25.61
&
24.78
&
24.36
&
24.58
&
29.21
&
29.25
&
25.24
&
29.94
\\

&
&
SSIM$\uparrow$
&
0.795
&
0.778
&
\best{0.932}
&
\second{0.879}
&
0.824
&
0.782
&
0.774
&
0.652
&
0.654
&
0.829
&
0.832
&
0.696
&
0.863
\\

&
&
LPIPS$\downarrow$
&
0.2386
&
0.3007
&
\best{0.0588}
&
\second{0.0875}
&
0.1571
&
0.2385
&
0.2240
&
0.4522
&
0.4140
&
0.1760
&
0.1899
&
0.2968
&
0.1840
\\

\cmidrule(lr){2-16}

&
\multirow{3}{*}{Nonlinear}
&
PSNR$\uparrow$
&
24.22
&
24.11
&
\best{26.35}
&
20.53
&
24.40
&
23.59
&
22.55
&
23.59
&
22.85
&
23.11
&
24.69
&
23.29
&
\second{25.52}
\\

&
&
SSIM$\uparrow$
&
0.645
&
0.643
&
\best{0.733}
&
0.595
&
\second{0.730}
&
0.725
&
0.713
&
0.612
&
0.569
&
0.492
&
0.665
&
0.600
&
0.708
\\

&
&
LPIPS$\downarrow$
&
0.4432
&
0.5140
&
\best{0.2591}
&
0.3119
&
\second{0.2780}
&
0.3147
&
0.3094
&
0.5192
&
0.4977
&
0.5298
&
0.4156
&
0.4323
&
0.3986
\\

\bottomrule
\end{tabular}
\end{adjustbox}
\end{table}

\providecommand{\imgcell}[3]{}
\renewcommand{\imgcell}[3]{%
\begin{minipage}[t]{0.185\linewidth}
    \centering
    \includegraphics[width=\linewidth]{#1}\par
    \vspace{-0.04in}
    {\scriptsize #2}\par
    \vspace{-0.04in}
    {\scriptsize #3}
\end{minipage}%
}

\begin{figure}[t]
\centering
\setlength{\tabcolsep}{2pt}
\renewcommand{\arraystretch}{1.0}

\begin{tabular}{ccccc}
\imgcell{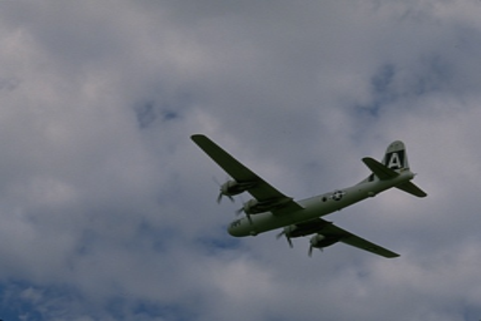}
{(a) GT}
{ }
&
\imgcell{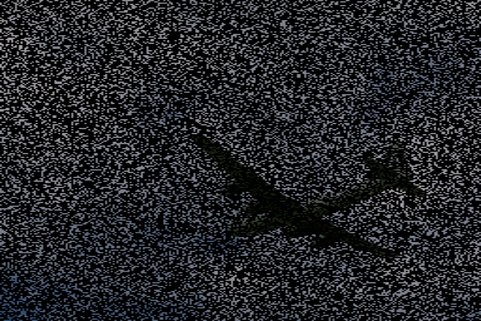}
{(b) Observation}
{8.58 / 0.037 / 1.012}
&
\imgcell{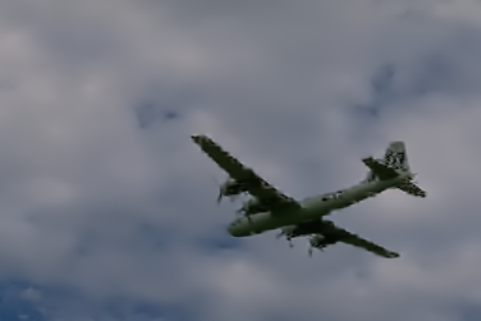}
{(c) PnP-DnCNN}
{35.11 / 0.969 / 0.1187}
&
\imgcell{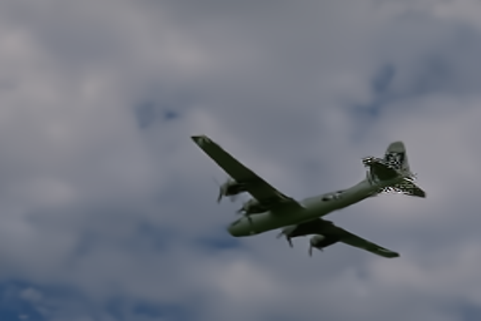}
{(d) PnP-FFDNet}
{34.14 / 0.970 / 0.1128}
&
\imgcell{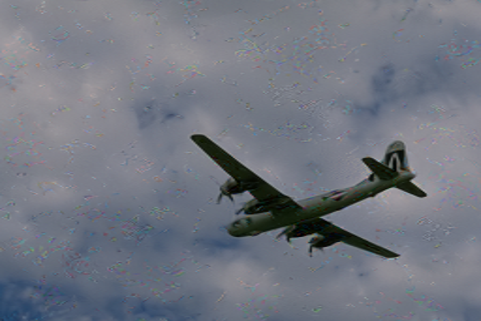}
{(e) SITCOM}
{34.05 / 0.837 / 0.3052}
\\[0.2in]

\imgcell{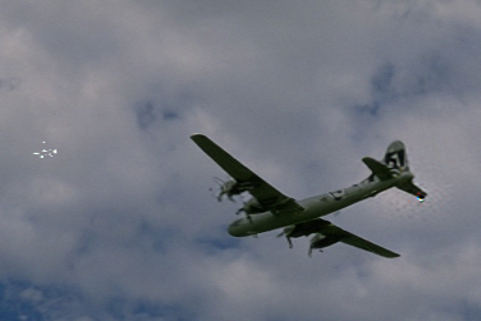}
{(f) DPS}
{35.24 / 0.912 / 0.0674}
&
\imgcell{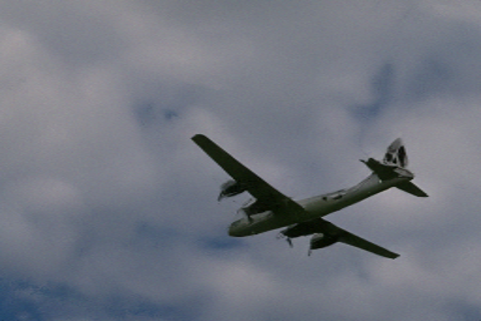}
{(g) DAPS}
{36.67 / 0.872 / 0.0767}
&
\imgcell{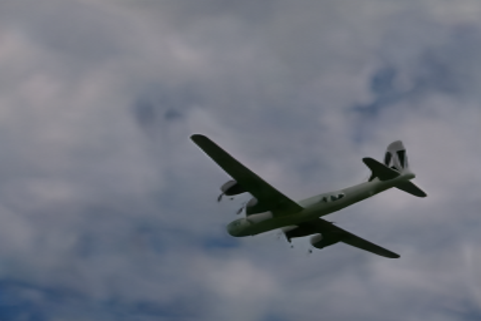}
{(h) RED-diff}
{37.70 / 0.918 / 0.0965}
&
\imgcell{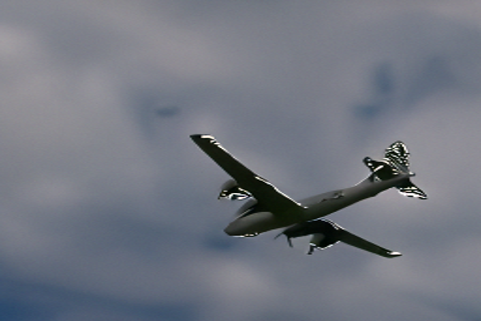}
{(i) PnP-DM}
{27.16 / 0.848 / 0.1435}
&
\imgcell{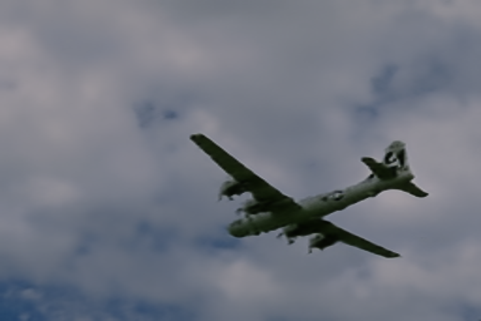}
{(j) DIP}
{36.89 / \underline{0.971} / 0.0803}
\\[0.2in]

\imgcell{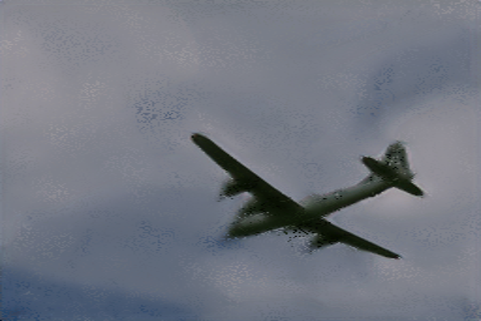}
{(k) ASeqDIP}
{31.72 / 0.892 / 0.2943}
&
\imgcell{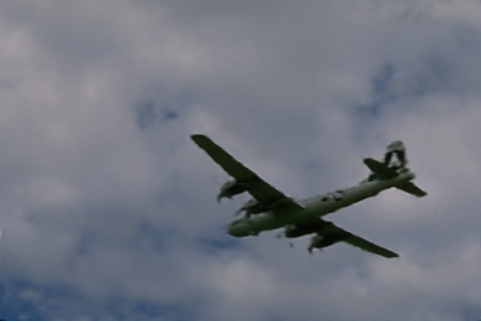}
{(l) Self-Guided DIP}
{37.16 / 0.978 / \textbf{0.0425}}
&
\imgcell{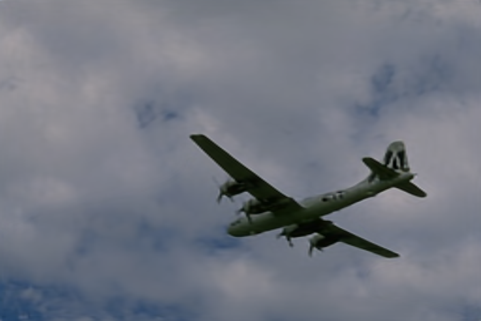}
{(m) SDI}
{\underline{39.70} / \textbf{0.982} / 0.0550}
&
\imgcell{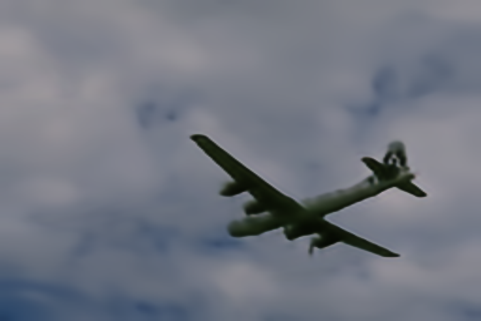}
{(n) UGoDIT}
{36.20 / 0.962 / 0.1236}
&
\imgcell{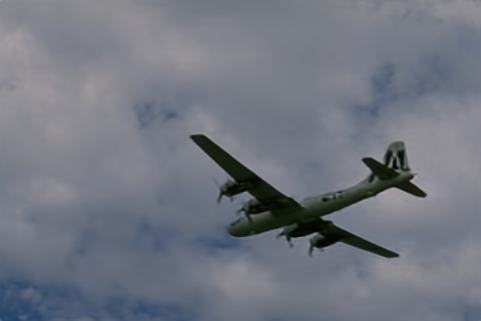}
{(o) \textbf{\textcolor{traceblue}{TRACE}}}
{\textbf{39.95} / \textbf{0.982} / \underline{0.0523}}
\end{tabular}

\vspace{0.35em}
\caption{Inpainting result (PSNR$\uparrow$ / SSIM$\uparrow$ / LPIPS$\downarrow$) with $70\%$ missing pixels.
}
\label{fig:inpaint70_example}
\end{figure}

We consider natural-image inverse problems, where pretrained diffusion and CNN priors are readily available.
We examine whether removing external pretraining leads to a substantial loss in reconstruction quality.
We evaluate image inpainting, super-resolution, and deblurring, with the results summarized in Table~\ref{tab:natural_main_extended}.
Among methods that are optimized without external training, TRACE achieves the highest PSNR in five of the six settings.
The advantage is more apparent in the harder reconstruction settings. For $70\%$ missing inpainting, TRACE achieves the best PSNR and LPIPS among all compared methods, while remaining close to the best SSIM. For $\times4$ super-resolution, TRACE obtains $27.47$ dB, nearly matching SITCOM at $27.50$ dB, and achieves the highest SSIM.

The pretrained methods remain stronger in several cases.
In particular, SITCOM achieves the best results for both deblurring settings. These results are consistent with the CT comparison (Table \ref{tab:ct}): a well-matched pretrained prior can provide a clear advantage, whereas TRACE does not require such a model and remains competitive across different inverse problems.

Figure~\ref{fig:inpaint70_example} shows a representative result for $70\%$ missing inpainting.
TRACE recovers the missing structures with fewer visible artifacts and achieves the highest PSNR. Additional results for inpainting, super-resolution, and
deblurring are provided in Appendix~\ref{app:moreresults}.

\subsection{Generality to Pretrained Diffusion Solvers}
\label{sec:diffusion_transfer}
Although TRACE uses an untrained neural network, the trajectory constraint is not specific to this parameterization. To examine whether the same idea extends beyond TRACE, we augment five pretrained diffusion-based inverse solvers with the proximal
state-consistency term used in SPC.
For each method, we compare the original solver (Base) with its SPC-augmented version (+SPC).
We consider Gaussian deblurring on FFHQ and multi-coil MRI reconstruction on fastMRI~\cite{knoll2020fastmri} at $4\times$ and $8\times$ acceleration.
Experimental settings are provided in Appendix~\ref{app:spc_details}.

\begin{table}[t]
\centering
\small
\setlength{\tabcolsep}{3.2pt}
\renewcommand{\arraystretch}{1.10}

\caption{
Effect of SPC on pretrained diffusion inverse solvers.
$\Delta$ denotes the change in mean PSNR (dB), and the number in parentheses reports how many test images or slices obtain higher PSNR with SPC out of the total number evaluated. 
}
\label{tab:spc_transfer}

\begin{adjustbox}{max width=\columnwidth}
\begin{tabular}{@{}l ccc ccc ccc@{}}
\toprule

&
\multicolumn{3}{c}{\textbf{FFHQ Deblur}}
&
\multicolumn{3}{c}{\textbf{MRI $4\times$}}
&
\multicolumn{3}{c}{\textbf{MRI $8\times$}}
\\

\cmidrule(lr){2-4}
\cmidrule(lr){5-7}
\cmidrule(lr){8-10}

\textbf{Solver}
&
\textbf{Base}
&
\textbf{+SPC}
&
\textbf{$\Delta$}
&
\textbf{Base}
&
\textbf{+SPC}
&
\textbf{$\Delta$}
&
\textbf{Base}
&
\textbf{+SPC}
&
\textbf{$\Delta$}
\\

\midrule

DPS
& 24.30
& \textbf{25.95}
& +1.64\,{\scriptsize(17/20)}
& 33.40
& \textbf{33.42}
& +0.02\,{\scriptsize(5/10)}
& 29.43
& 29.43
& +0.00\,{\scriptsize(6/10)}
\\

DAPS
& 28.12
& \textbf{28.32}
& +0.20\,{\scriptsize(19/20)}
& 31.33
& \textbf{31.70}
& +0.37\,{\scriptsize(10/10)}
& 29.02
& \textbf{29.33}
& +0.31\,{\scriptsize(10/10)}
\\

RED-diff
& 27.47
& \textbf{27.60}
& +0.13\,{\scriptsize(16/20)}
& 30.67
& \textbf{31.04}
& +0.37\,{\scriptsize(10/10)}
& 28.42
& \textbf{28.71}
& +0.28\,{\scriptsize(8/10)}
\\

PnP-DM
& 25.39
& \textbf{25.51}
& +0.12\,{\scriptsize(13/20)}
& 29.35
& \textbf{29.54}
& +0.19\,{\scriptsize(9/10)}
& 27.58
& \textbf{27.72}
& +0.14\,{\scriptsize(9/10)}
\\

SITCOM
& 27.91
& \textbf{28.55}
& +0.64\,{\scriptsize(19/20)}
& 34.54
& \textbf{34.59}
& +0.05\,{\scriptsize(10/10)}
& 31.97
& \textbf{32.08}
& +0.11\,{\scriptsize(9/10)}
\\

\midrule

\rowcolor{tracebg}
\textbf{Mean}
&
&
&
\textbf{+0.55}\,{\scriptsize(84/100)}
&
&
&
\textbf{+0.20}\,{\scriptsize(44/50)}
&
&
&
\textbf{+0.17}\,{\scriptsize(42/50)}
\\

\bottomrule
\end{tabular}
\end{adjustbox}
\end{table}

Table~\ref{tab:spc_transfer} reports the effect of adding the state-consistency constraint to the five pretrained diffusion solvers.
The average gain is $0.55$ dB on FFHQ deblurring, with larger improvements for DPS ($+1.64$ dB) and SITCOM ($+0.64$ dB).
For MRI, the average gains are $0.20$ dB at $4\times$ acceleration and $0.17$ dB at $8\times$ acceleration.
The improvement is smaller for DPS on MRI, while DAPS and RED-diff show consistent gains at both acceleration factors.
Although SPC does not improve every individual reconstruction, positive changes are observed on most test cases.
Across the five solvers, SPC improves PSNR on 84 of the 100 FFHQ deblur evaluations, 44 of the 50 evaluations for $4\times$ MRI, and 42 of the 50 for $8\times$ MRI.

These results show that the trajectory constraint is not tied to the untrained network used in TRACE. This supports the broader role of trajectory constraint for controlling intermediate reconstruction states.

\section{Conclusion}

We presented TRACE, a training-free framework for imaging inverse problems that reconstructs the image through a sequence of intermediate states rather than relying on a pretrained diffusion model. TRACE uses stochastic proximal consistency to constrain successive states, and the resulting formulation admits both a proximal interpretation and a local Jacobian analysis. 
Experiments on natural-image restoration and CT reconstruction show that TRACE remains competitive with methods using pretrained or supervised priors while requiring no external training data.
We also find that the same state-consistency constraint can improve existing diffusion inverse solvers.
These results suggest that explicitly modeling the reconstruction trajectory provides a useful alternative to relying only on a learned prior or a final reconstruction objective.

\section{Acknowledgments}
This work was supported in part by the
National Science Foundation (NSF) grants CCF-2212065, ECCS-2436945, and NSF CAREER CCF-2442240. 

\bibliography{references}
%%%%%%%%%%%%%%%%%%%%%%%%%%%%%%%%%%%%%%%%%%%%%%%%%%%%%%%%%%%%
\newpage
\appendix
\onecolumn
\par\noindent\rule{\textwidth}{1pt}
\begin{center}
{\Large \bf Appendix}
\end{center}
\vspace{-0.1in}
\par\noindent\rule{\textwidth}{1pt}

\section*{Appendix}
The appendix is organized as follows. %\colorbox{yellow}{\color{red}check if i'm correct, for everyone}
\begin{itemize}
    \item Appendix~\ref{app:theory} provides additional theoretical results for the proximal formulation, including the effect of the coupling strength, the relation between local optimization error and the exact proximal update, error propagation along the trajectory, and the derivation of the noise-coupling interaction in Proposition~\ref{prop:interaction}.
    \item Appendix~\ref{app:abs} presents the empirical analysis of stochastic proximal consistency (SPC), including component ablations, controlled experiments that separate stochastic perturbation and proximal coupling, and their sensitivity to the corresponding hyperparameters. %\colorbox{yellow}{\color{red}check if i'm correct in the Appendix D, I added something in the summary of D and D.1, D.2.  I do not take a careful look at D.3, sorry, but pls make sure everything is consistent and explained well }
    \item Appendix~\ref{app:emprical} connects the empirical observations to the theoretical analysis through null-space reconstruction, Jacobian sensitivity, and spectral-bias experiments.
    \item Appendix~\ref{app:algo} gives implementation details for TRACE and the compared methods, together with the evaluation protocol and task-specific settings.
    \item Appendix~\ref{app:additional_ct} provides detailed settings and additional results for CT reconstruction.
    \item Appendix~\ref{app:moreresults} provides detailed settings and additional qualitative reconstruction results for natural-image restoration.
    \item Appendix~\ref{app:additional_ablation} collects additional ablations and computational analyses, including the allocation of the optimization budget between outer and inner iterations, pretrained initialization, and runtime and memory costs.
    \item Appendix~\ref{app:spc_details} provides the experimental protocol and solver-specific details for incorporating the state-consistency constraint into pretrained diffusion inverse solvers. 
\end{itemize}

{
\section{Additional Theoretical Analysis}
\label{app:theory}

\subsection{Additional Results for Stochastic Proximal Consistency }
\label{app:temporal_coupling}
In this subsection, we give additional results for the proximal model
in \eqref{eq:prox_subproblem}. We first consider the exact proximal
update and study the effect of the coupling parameter $\beta_t$.
We then relate the suboptimality gap $\varepsilon_t$ in
\eqref{eq:prox_gap} to the distance from the exact proximal update
and study how these errors propagate along the reconstruction trajectory.

Recall that
\begin{equation}
\label{eq:app_phi}
    \Phi_t(\mathbf{x};\mathbf{v})
    =
    F_t(\mathbf{x})
    +
    \frac{\beta_t}{2}
    \|\mathbf{x}-\mathbf{v}\|_2^2,
\end{equation}
where
$F_t(\mathbf{x})=f(\mathbf{x})+\lambda_t R(\mathbf{x}).$ 
Unless stated otherwise, we assume that $F_t$ is proper, lower semicontinuous, and bounded below. These are the same general conditions used in the main text.

% ============================================================
\subsubsection{Exact Proximal Update}
\label{app:exact_proximal}

We first consider the exact local problem corresponding to $\varepsilon_t=0$. For a fixed neighboring state $\mathbf{v}$ and coupling parameter $\beta>0$, define
\begin{equation}
\label{eq:app_exact_problem}
    \mathbf{x}^{\star}(\beta)
    \in
    \arg\min_{\mathbf{x}}
    \left\{
    F_t(\mathbf{x})
    +
    \frac{\beta}{2}
    \|\mathbf{x}-\mathbf{v}\|_2^2
    \right\}.
\end{equation}

Since $F_t$ is bounded below, the objective in \eqref{eq:app_exact_problem} is coercive for every $\beta>0$. Together with properness and lower semicontinuity, this guarantees the existence of at least one global minimizer.

\begin{proposition}
\label{prop:app_beta}
Let $0<\beta_1<\beta_2$,  and let
$\mathbf{x}^{\star,1}$ and $\mathbf{x}^{\star,2}$
be any global minimizers of \eqref{eq:app_exact_problem}
corresponding to $\beta_1$ and $\beta_2$, respectively.
Then
\begin{equation}
\label{eq:app_distance_beta}
    \|\mathbf{x}^{\star,2}-\mathbf{v}\|_2
    \leq
    \|\mathbf{x}^{\star,1}-\mathbf{v}\|_2,
\end{equation}
and
\begin{equation}
\label{eq:app_F_beta}
    F_t(\mathbf{x}^{\star,1})
    \leq
    F_t(\mathbf{x}^{\star,2}).
\end{equation}
\end{proposition}

\begin{proof}
Define
$d_1:=\|\mathbf{x}^{\star,1}-\mathbf{v}\|_2$, $d_2:=\|\mathbf{x}^{\star,2}-\mathbf{v}\|_2$. 
Since $\mathbf{x}^{\star,1}$ is a global minimizer for $\beta_1$,
\begin{equation}
\label{eq:app_beta1}
    F_t(\mathbf{x}^{\star,1})
    +
    \frac{\beta_1}{2}d_1^2
    \leq
    F_t(\mathbf{x}^{\star,2})
    +
    \frac{\beta_1}{2}d_2^2.
\end{equation}
Similarly, optimality of $\mathbf{x}^{\star,2}$ for $\beta_2$ gives
\begin{equation}
\label{eq:app_beta2}
    F_t(\mathbf{x}^{\star,2})
    +
    \frac{\beta_2}{2}d_2^2
    \leq
    F_t(\mathbf{x}^{\star,1})
    +
    \frac{\beta_2}{2}d_1^2.
\end{equation}
Hence \eqref{eq:app_beta1} and \eqref{eq:app_beta2} yields
\begin{equation}
    \frac{\beta_2-\beta_1}{2}
    (d_2^2-d_1^2)
    \leq 0.
\end{equation}
Since $\beta_2>\beta_1$, we obtain
\begin{equation}
    d_2\leq d_1,
\end{equation}
which proves \eqref{eq:app_distance_beta}.

Using \eqref{eq:app_beta1} and $d_2^2\leq d_1^2$,
\begin{align}
    F_t(\mathbf{x}^{\star,1})
    -
    F_t(\mathbf{x}^{\star,2})
    &\leq
    \frac{\beta_1}{2}
    (d_2^2-d_1^2)
    \leq 0,
\end{align}
which proves \eqref{eq:app_F_beta}.
\end{proof}

Proposition~\ref{prop:app_beta} gives a stronger statement than in the main text for the exact local problem. For a fixed neighboring state $\mathbf{v}$, increasing the coupling parameter cannot increase the distance of the proximal solution from $\mathbf{v}$. At the same time, the value of the uncoupled objective $F_t$ cannot decrease. Thus, $\beta_t$ controls the balance between minimizing $F_t$ and remaining close to the neighboring trajectory state.
\begin{remark}
The comparison in Proposition~\ref{prop:app_beta} is local. The state $\mathbf{v}$ is kept fixed when $\beta$ is varied. It does not imply monotonicity of the complete TRACE trajectory when all trajectory states are recomputed for different coupling schedules. 
\end{remark}

% ============================================================
\subsubsection{From Proximal Suboptimality to State Error}
\label{app:proximal_error}

The main text defines
\begin{equation}
\label{eq:app_eps}
    \varepsilon_t
    =
    \Phi_t(\mathbf{x}_t;\mathbf{x}_{t+1})
    -
    \inf_{\mathbf{x}}
    \Phi_t(\mathbf{x};\mathbf{x}_{t+1}),
\end{equation}
where $\mathbf{x}_t=D_{\theta_t^K}(\mathbf{u}_t)$ is the implemented neural state.
Without additional curvature, a bound on the objective gap $\varepsilon_t$ does not in general imply a bound on the distance between $\mathbf{x}_t$ and a minimizer. We therefore consider a stronger condition for the remainder of this section.

Recall that $F_t$ is $\rho_t$-weakly convex if $F_t(\mathbf{x})+\frac{\rho_t}{2}\|\mathbf{x}\|_2^2$ is convex.

\begin{lemma}
\label{lem:app_gap_distance}
Suppose that $F_t$ is $\rho_t$-weakly convex and $\beta_t>\rho_t$. Then, for every $\mathbf{v}$,
$\Phi_t(\cdot;\mathbf{v})$ is $(\beta_t-\rho_t)$-strongly convex and therefore has a unique
global minimizer. Denote this minimizer by
$
    P_t(\mathbf{v})
    :=
    \arg\min_{\mathbf{x}}
    \Phi_t(\mathbf{x};\mathbf{v}).
$
Then the implemented state satisfies
\begin{equation}
\label{eq:app_gap_distance_bound}
    \|
    \mathbf{x}_t
    -
    P_t(\mathbf{x}_{t+1})
    \|_2
    \leq
    \sqrt{
    \frac{2\varepsilon_t}
    {\beta_t-\rho_t}
    }.
\end{equation}
\end{lemma}

\begin{proof}
Since $F_t$ is $\rho_t$-weakly convex,
$
    \Phi_t(\mathbf{x};\mathbf{v})=
    F_t(\mathbf{x})
    +
    \frac{\beta_t}{2}
    \|\mathbf{x}-\mathbf{v}\|_2^2
$
is $(\beta_t-\rho_t)$-strongly convex whenever
$\beta_t>\rho_t$.
Hence, it has a unique global minimizer $P_t(\mathbf{v})$.

By strong convexity and optimality of $P_t(\mathbf{x}_{t+1})$,
\begin{equation}
    \Phi_t(\mathbf{x}_t;\mathbf{x}_{t+1})
    -
    \Phi_t(
        P_t(\mathbf{x}_{t+1});
        \mathbf{x}_{t+1}
    )
    \geq
    \frac{\beta_t-\rho_t}{2}
    \|
    \mathbf{x}_t
    -
    P_t(\mathbf{x}_{t+1})
    \|_2^2.
\end{equation}
The left-hand side is $\varepsilon_t$ by
\eqref{eq:app_eps}. Therefore,
\begin{equation}\label{eq:defdelta}
    \frac{\beta_t-\rho_t}{2}
    \|
    \mathbf{x}_t
    -
    P_t(\mathbf{x}_{t+1})
    \|_2^2
    \leq
    \varepsilon_t,
\end{equation}
which gives \eqref{eq:app_gap_distance_bound}.
\end{proof}

% ============================================================
\subsubsection{Propagation Along the Trajectory}
\label{app:error_propagation}

We now compare the implemented TRACE trajectory with the trajectory
obtained by exact proximal updates. Starting from the same initial state,
define
\begin{equation}
\label{eq:app_ideal_trajectory}
    \widetilde{\mathbf{x}}_T
    =
    \mathbf{x}_T,
    \qquad
    \widetilde{\mathbf{x}}_t
    =
    P_t(\widetilde{\mathbf{x}}_{t+1}),
    \quad
    t=T-1,\ldots,0.
\end{equation}

We first record the Lipschitz property of $P_t$ under weak convexity.

\begin{lemma}
\label{lem:app_prox_lipschitz}
Suppose that $F_t$ is $\rho_t$-weakly convex and
$\beta_t>\rho_t$. Then
\begin{equation}
\label{eq:app_prox_lipschitz}
    \|P_t(\mathbf{a})-P_t(\mathbf{b})\|_2
    \leq
    \frac{\beta_t}{\beta_t-\rho_t}
    \|\mathbf{a}-\mathbf{b}\|_2
\end{equation}
for all $\mathbf{a},\mathbf{b}$.
\end{lemma}

\begin{proof}
Let
$\mathbf{p}=P_t(\mathbf{a})$, $\mathbf{q}=P_t(\mathbf{b})$, and define $m_t=\beta_t-\rho_t>0$. 
The functions $\Phi_t(\cdot;\mathbf{a})$ and $\Phi_t(\cdot;\mathbf{b})$ are $m_t$-strongly convex.

Since $\mathbf{p}$ minimizes
$\Phi_t(\cdot;\mathbf{a})$,
\begin{equation}
\label{eq:app_sc1}
    \Phi_t(\mathbf{q};\mathbf{a})
    \geq
    \Phi_t(\mathbf{p};\mathbf{a})
    +
    \frac{m_t}{2}
    \|\mathbf{p}-\mathbf{q}\|_2^2.
\end{equation}
Similarly,
\begin{equation}
\label{eq:app_sc2}
    \Phi_t(\mathbf{p};\mathbf{b})
    \geq
    \Phi_t(\mathbf{q};\mathbf{b})
    +
    \frac{m_t}{2}
    \|\mathbf{p}-\mathbf{q}\|_2^2.
\end{equation}
Adding \eqref{eq:app_sc1} and \eqref{eq:app_sc2},
the $F_t$ terms cancel and we obtain
\begin{equation}
\label{eq:app_inner_product}
    m_t
    \|\mathbf{p}-\mathbf{q}\|_2^2
    \leq
    \beta_t
    \langle
    \mathbf{a}-\mathbf{b},
    \mathbf{p}-\mathbf{q}
    \rangle.
\end{equation}
Applying Cauchy--Schwarz gives
\begin{equation}
    (\beta_t-\rho_t)
    \|\mathbf{p}-\mathbf{q}\|_2^2
    \leq
    \beta_t
    \|\mathbf{a}-\mathbf{b}\|_2
    \|\mathbf{p}-\mathbf{q}\|_2.
\end{equation}
If $\mathbf{p}\neq\mathbf{q}$, division by
$\|\mathbf{p}-\mathbf{q}\|_2$ gives
\eqref{eq:app_prox_lipschitz}. If
$\mathbf{p}=\mathbf{q}$, the result is immediate.
\end{proof}

\begin{theorem}
\label{thm:app_error_propagation}
Suppose that, for every $t=T-1,\ldots,0$, $F_t$ is $\rho_t$-weakly convex and $\beta_t>\rho_t$.
Then
\begin{equation}
\label{eq:app_error_recursion}
    \|
    \mathbf{x}_t
    -
    \widetilde{\mathbf{x}}_t
    \|_2
    \leq
    \delta_t
    +
    \frac{\beta_t}{\beta_t-\rho_t}
    \|
    \mathbf{x}_{t+1}
    -
    \widetilde{\mathbf{x}}_{t+1}
    \|_2,
\end{equation}
where $
    \delta_t=
    \sqrt{
    \frac{2\varepsilon_t}
    {\beta_t-\rho_t}
    }$ is the right-hand side of \eqref{eq:app_gap_distance_bound}.

Consequently,
\begin{equation}
\label{eq:app_error_bound}
    \|
    \mathbf{x}_t
    -
    \widetilde{\mathbf{x}}_t
    \|_2
    \leq
    \sum_{s=t}^{T-1}
    \delta_s
    \prod_{i=t}^{s-1}
    \frac{\beta_i}{\beta_i-\rho_i},
\end{equation}
where the empty product is defined as $1$.
In particular,
\begin{equation}
\label{eq:app_final_error_bound}
    \|
    \mathbf{x}_0
    -
    \widetilde{\mathbf{x}}_0
    \|_2
    \leq
    \sum_{s=0}^{T-1}
    \delta_s
    \prod_{i=0}^{s-1}
    \frac{\beta_i}{\beta_i-\rho_i}.
\end{equation}
\end{theorem}

\begin{proof}
Using
$\widetilde{\mathbf{x}}_t
=
P_t(\widetilde{\mathbf{x}}_{t+1})$,
we have
\begin{align}
    \|
    \mathbf{x}_t
    -
    \widetilde{\mathbf{x}}_t
    \|_2
    &\leq
    \|
    \mathbf{x}_t
    -
    P_t(\mathbf{x}_{t+1})
    \|_2
    \nonumber\\
    &\quad+
    \|
    P_t(\mathbf{x}_{t+1})
    -
    P_t(\widetilde{\mathbf{x}}_{t+1})
    \|_2.
\end{align}
By Lemma~\ref{lem:app_gap_distance},
\begin{equation}
    \|
    \mathbf{x}_t
    -
    P_t(\mathbf{x}_{t+1})
    \|_2
    \leq
    \delta_t.
\end{equation}
By Lemma~\ref{lem:app_prox_lipschitz},
\begin{equation}
    \|
    P_t(\mathbf{x}_{t+1})
    -
    P_t(\widetilde{\mathbf{x}}_{t+1})
    \|_2
    \leq
    \frac{\beta_t}{\beta_t-\rho_t}
    \|
    \mathbf{x}_{t+1}
    -
    \widetilde{\mathbf{x}}_{t+1}
    \|_2.
\end{equation}
Combining these two inequalities gives
\eqref{eq:app_error_recursion}.

Since
$
    \mathbf{x}_T
    =
    \widetilde{\mathbf{x}}_T,
$ 
repeated application of
\eqref{eq:app_error_recursion}
gives \eqref{eq:app_error_bound}.
Setting $t=0$ gives
\eqref{eq:app_final_error_bound}.
\end{proof}

Theorem~\ref{thm:app_error_propagation} controls the deviation between
the implemented TRACE trajectory and the exact proximal trajectory.

If $F_t$ is convex, we can take $\rho_t=0$. Then
\begin{equation}
    \delta_t
    =
    \sqrt{
    \frac{2\varepsilon_t}{\beta_t}
    },
\end{equation}
and \eqref{eq:app_error_bound} becomes
\begin{equation}
\label{eq:app_convex_bound}
    \|
    \mathbf{x}_t
    -
    \widetilde{\mathbf{x}}_t
    \|_2
    \leq
    \sum_{s=t}^{T-1}
    \sqrt{
    \frac{2\varepsilon_s}{\beta_s}
    }.
\end{equation}
Thus, in the convex case, the proximal maps are nonexpansive and the
local approximation errors accumulate additively over the finite
trajectory.

The weak and strong convexity assumptions in this appendix are used only to obtain the stronger trajectory-level results above. %They are imposed on the image-domain objective $F_t$ in \eqref{eq:prox_subproblem}, not on the neural-network parameterization in \eqref{eq:training_loss}.
Under weak or strong convexity, this suboptimality can be converted into a distance from the exact proximal update, which then allows us to analyze error propagation along the trajectory.

\subsection{Derivation of the Noise--Coupling Interaction}
\label{app:interaction_derivation}

We provide the derivation of Proposition~\ref{prop:interaction}.
Consider a single TRACE transition and fix the trajectory state $x$ and
network parameters $\theta$. The perturbed network input is
\[
x+\sigma\epsilon,
\qquad
\epsilon\sim\mathcal{N}(0,I),
\]
where $\sigma$ controls the perturbation magnitude. The corresponding TRACE
objective is
\begin{equation}
\mathcal{L}
=
\frac{1}{2}
\left\|
A D_\theta(x+\sigma\epsilon)-y
\right\|_2^2
+
\frac{\beta}{2}
\left\|
D_\theta(x+\sigma\epsilon)-x
\right\|_2^2.
\label{eq:trace_objective}
\end{equation}

We study the expectation of this objective with respect to the stochastic
perturbation. Using a first-order Taylor expansion of the network around $x$,
\begin{equation}
D_\theta(x+\sigma\epsilon)
\approx
D_\theta(x)+\sigma J_D(x)\epsilon,
\label{eq:taylor_D}
\end{equation}
where
\[
J_D(x)
=
\left.
\frac{\partial D_\theta(u)}{\partial u}
\right|_{u=x}
\]
denotes the Jacobian of the network output with respect to its input.
The following derivation is therefore with respect to this locally linearized
network mapping.

Substituting \eqref{eq:taylor_D} into \eqref{eq:trace_objective} gives
\begin{equation}
\begin{aligned}
\mathcal{L}
\approx{}&
\frac{1}{2}
\left\|
A\left(
D_\theta(x)+\sigma J_D(x)\epsilon
\right)-y
\right\|_2^2
\\
&+
\frac{\beta}{2}
\left\|
D_\theta(x)+\sigma J_D(x)\epsilon-x
\right\|_2^2.
\end{aligned}
\label{eq:linearized_trace_objective}
\end{equation}

Expanding the data-fidelity term,
\begin{equation}
\begin{aligned}
&\frac{1}{2}
\left\|
A\left(
D_\theta(x)+\sigma J_D(x)\epsilon
\right)-y
\right\|_2^2
\\
={}&
\frac{1}{2}
\left\|
A D_\theta(x)-y
\right\|_2^2
+
\sigma
\left(
A D_\theta(x)-y
\right)^\top
A J_D(x)\epsilon
\\
&+
\frac{\sigma^2}{2}
\epsilon^\top
J_D(x)^\top A^\top A J_D(x)
\epsilon.
\end{aligned}
\label{eq:data_expansion}
\end{equation}

Similarly, expanding the temporal-coupling term gives
\begin{equation}
\begin{aligned}
&\frac{\beta}{2}
\left\|
D_\theta(x)+\sigma J_D(x)\epsilon-x
\right\|_2^2
\\
={}&
\frac{\beta}{2}
\left\|
D_\theta(x)-x
\right\|_2^2
+
\beta\sigma
\left(
D_\theta(x)-x
\right)^\top
J_D(x)\epsilon
\\
&+
\frac{\beta\sigma^2}{2}
\epsilon^\top
J_D(x)^\top J_D(x)
\epsilon.
\end{aligned}
\label{eq:temporal_expansion}
\end{equation}

Combining \eqref{eq:data_expansion} and
\eqref{eq:temporal_expansion}, we obtain
\begin{equation}
\begin{aligned}
\mathcal{L}
\approx{}&
\frac{1}{2}
\left\|
A D_\theta(x)-y
\right\|_2^2
+
\frac{\beta}{2}
\left\|
D_\theta(x)-x
\right\|_2^2
\\
&+
\sigma
\left(
A D_\theta(x)-y
\right)^\top
A J_D(x)\epsilon
+
\beta\sigma
\left(
D_\theta(x)-x
\right)^\top
J_D(x)\epsilon
\\
&+
\frac{\sigma^2}{2}
\epsilon^\top
J_D(x)^\top A^\top A J_D(x)
\epsilon
+
\frac{\beta\sigma^2}{2}
\epsilon^\top
J_D(x)^\top J_D(x)
\epsilon.
\end{aligned}
\label{eq:combined_expansion}
\end{equation}

We now take the expectation with respect to
$\epsilon\sim\mathcal{N}(0,I)$. Since
\[
\mathbb{E}_{\epsilon}[\epsilon]=0,
\]
the terms linear in $\epsilon$ vanish. Moreover, for any fixed matrix $Q$,
\[
\mathbb{E}_{\epsilon}
\left[
\epsilon^\top Q\epsilon
\right]
=
\operatorname{Tr}
\left(
Q\,\mathbb{E}_{\epsilon}[\epsilon\epsilon^\top]
\right)
=
\operatorname{Tr}(Q),
\]
since $\mathbb{E}_{\epsilon}[\epsilon\epsilon^\top]=I$. Therefore,
\begin{equation}
\begin{aligned}
\mathbb{E}_{\epsilon}
\left[
\mathcal{L}
\right]
\approx{}&
\frac{1}{2}
\left\|
A D_\theta(x)-y
\right\|_2^2
+
\frac{\beta}{2}
\left\|
D_\theta(x)-x
\right\|_2^2
\\
&+
\frac{\sigma^2}{2}
\operatorname{Tr}
\left(
J_D(x)^\top A^\top A J_D(x)
\right)
\\
&+
\frac{\beta\sigma^2}{2}
\operatorname{Tr}
\left(
J_D(x)^\top J_D(x)
\right).
\end{aligned}
\label{eq:expected_trace_objective_trace}
\end{equation}

Using
\[
\operatorname{Tr}
\left(
J_D(x)^\top A^\top A J_D(x)
\right)
=
\left\|
A J_D(x)
\right\|_F^2
\]
and
\[
\operatorname{Tr}
\left(
J_D(x)^\top J_D(x)
\right)
=
\left\|
J_D(x)
\right\|_F^2,
\]
we finally obtain
\begin{equation}
\begin{aligned}
\mathbb{E}_{\epsilon}
\left[
\mathcal{L}
\right]
\approx{}&
\underbrace{
\frac{1}{2}
\left\|
A D_\theta(x)-y
\right\|_2^2
+
\frac{\beta}{2}
\left\|
D_\theta(x)-x
\right\|_2^2
}_{\text{unperturbed TRACE objective}}
\\
&+
\underbrace{
\frac{\sigma^2}{2}
\left\|
A J_D(x)
\right\|_F^2
}_{\text{measurement-space sensitivity}}
+
\underbrace{
\frac{\beta\sigma^2}{2}
\left\|
J_D(x)
\right\|_F^2
}_{\text{noise--coupling interaction}}.
\end{aligned}
\label{eq:interpreted_objective}
\end{equation}

The final term represents the additional regularization induced
by the stochastic extension of the proximal relation, which
vanishes in the absence of proximal coupling, as stated in
Proposition~\ref{prop:interaction}.

}
\section{Empirical Analysis of Stochastic Proximal Consistency}
\label{app:abs}

This section examines the empirical behavior of stochastic proximal
consistency (SPC).
We first evaluate the contribution of the key design choices of TRACE
to the final reconstruction quality.
We then investigate how extending the proximal relation to a stochastic
neighborhood affects the reconstruction trajectory, using ablations
that isolate the effects of stochastic perturbation and proximal coupling.
To separate these effects from the neural parameterization, we first
consider network-free inverse problems with explicit regularizers and
then repeat the analysis under the network-parameterized TRACE formulation.
Finally, we examine the sensitivity to the perturbation and proximal
coupling strengths.

The purpose of these experiments is complementary to the theoretical
analysis in Sec.~\ref{sec:theory}.
Proposition~\ref{prop:trajectory_control} characterizes the local
trajectory control provided by the proximal relation, while
Proposition~\ref{prop:interaction} characterizes the additional
local regularization induced by its stochastic extension, which
vanishes in the absence of proximal coupling.

\subsection{Component Analysis}
We first examine the contribution of the main components of TRACE to the final reconstruction quality.
We consider three variants obtained by removing
(i) proximal coupling,
(ii) stochastic perturbation, and
(iii) weight inheritance across trajectory states.
We evaluate three Set14 images under $50\%$ missing-pixel inpainting.

Table~\ref{tab:ablation} shows that removing any of the three components degrades reconstruction quality.
Without proximal coupling, PSNR decreases from $32.77$ to $32.31$ dB.
Removing stochastic perturbation produces a larger drop to $30.59$ dB, while removing weight inheritance gives the largest degradation, reaching $26.01$ dB.
The corresponding SSIM and LPIPS values show the same overall trend.
These results indicate that all three design components contribute to the final reconstruction performance.

\begin{table}[t]
\centering
% \scriptsize
\setlength{\tabcolsep}{7.0pt}
\renewcommand{\arraystretch}{1.10}

\caption{
Component ablation of TRACE on $50\%$ missing-pixel inpainting
over three Set14 images.
Best and second-best values are shown in \textbf{bold} and
\underline{underlined}, respectively.
}
\label{tab:ablation}

\begin{tabular}{@{}lccc@{}}
\toprule
\rowcolor{headergray}
\textbf{Variant}
&
\textbf{PSNR $\uparrow$}
&
\textbf{SSIM $\uparrow$}
&
\textbf{LPIPS $\downarrow$}
\\
\midrule

TRACE w/o proximal coupling
&
\underline{32.31}
&
\underline{0.950}
&
\underline{0.0279}
\\

TRACE w/o stochastic perturbation
&
30.59
&
0.939
&
0.0395
\\

TRACE w/o weight inheritance
&
26.01
&
0.784
&
0.2166
\\

\midrule
\rowcolor{tracebg}
\textbf{\textcolor{traceblue}{TRACE}}
&
\textbf{32.77}
&
\textbf{0.953}
&
\textbf{0.0243}
\\

\bottomrule
\end{tabular}
\end{table}

\subsection{Trajectory Behavior and Proximal Coupling}
Before separating the individual components of SPC, we first show a representative TRACE trajectory in Fig.~\ref{fig:trace_trajectory}.

\begin{figure}[t]
\centering
\setlength{\tabcolsep}{2pt}
\captionsetup[subfigure]{font=scriptsize,skip=1pt}

\begin{subfigure}[t]{0.41\textwidth}
    \centering
    \includegraphics[width=\linewidth, trim=6 4 8 6, clip]{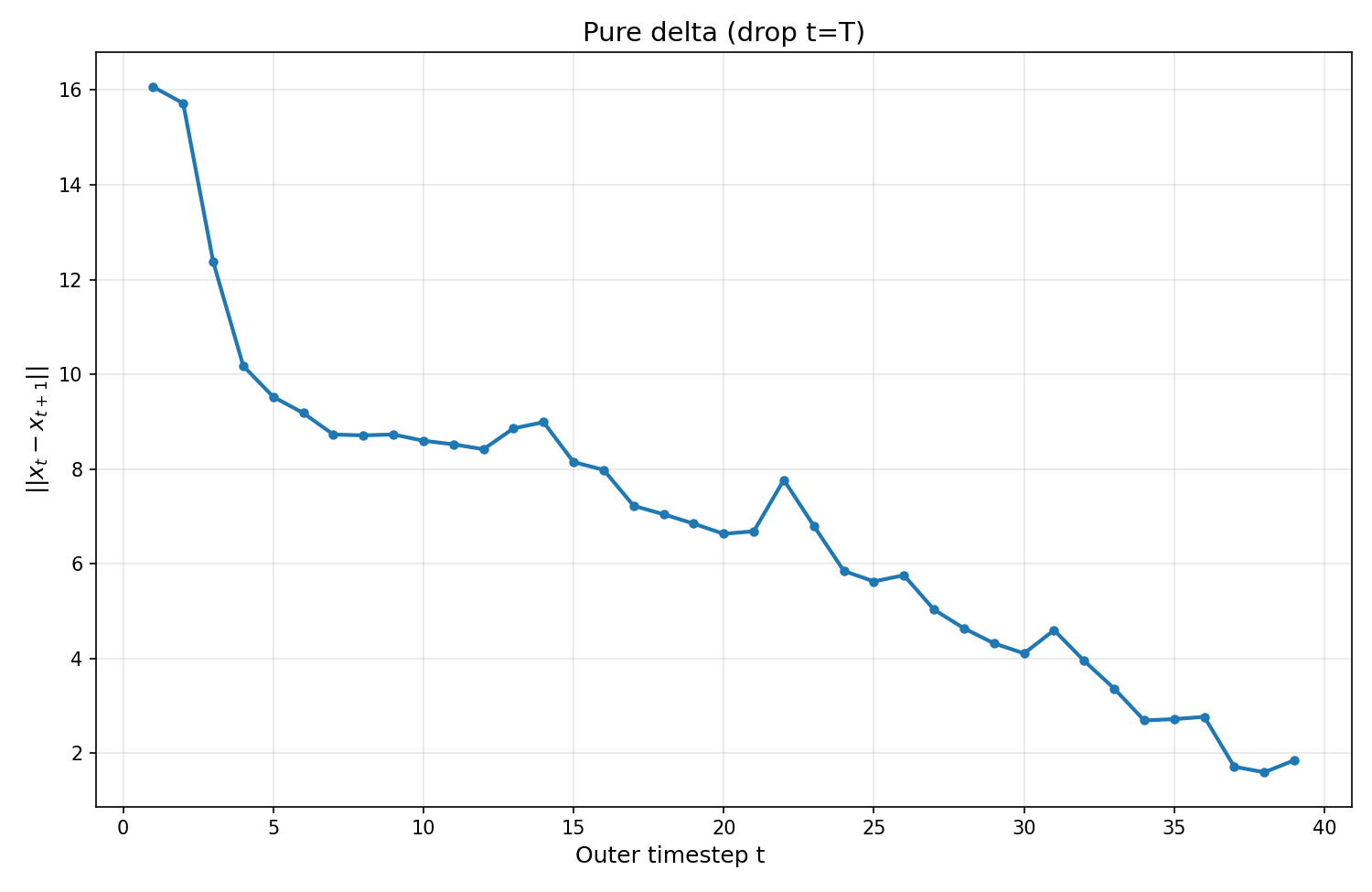}
    \caption{$\Delta_t=\|\mathbf{x}_t-\mathbf{x}_{t+1}\|_2$}
    \label{fig:traj_delta}
\end{subfigure}
\hspace{0.035\textwidth}
\begin{subfigure}[t]{0.41\textwidth}
    \centering
    \includegraphics[width=\linewidth, trim=6 4 8 6, clip]{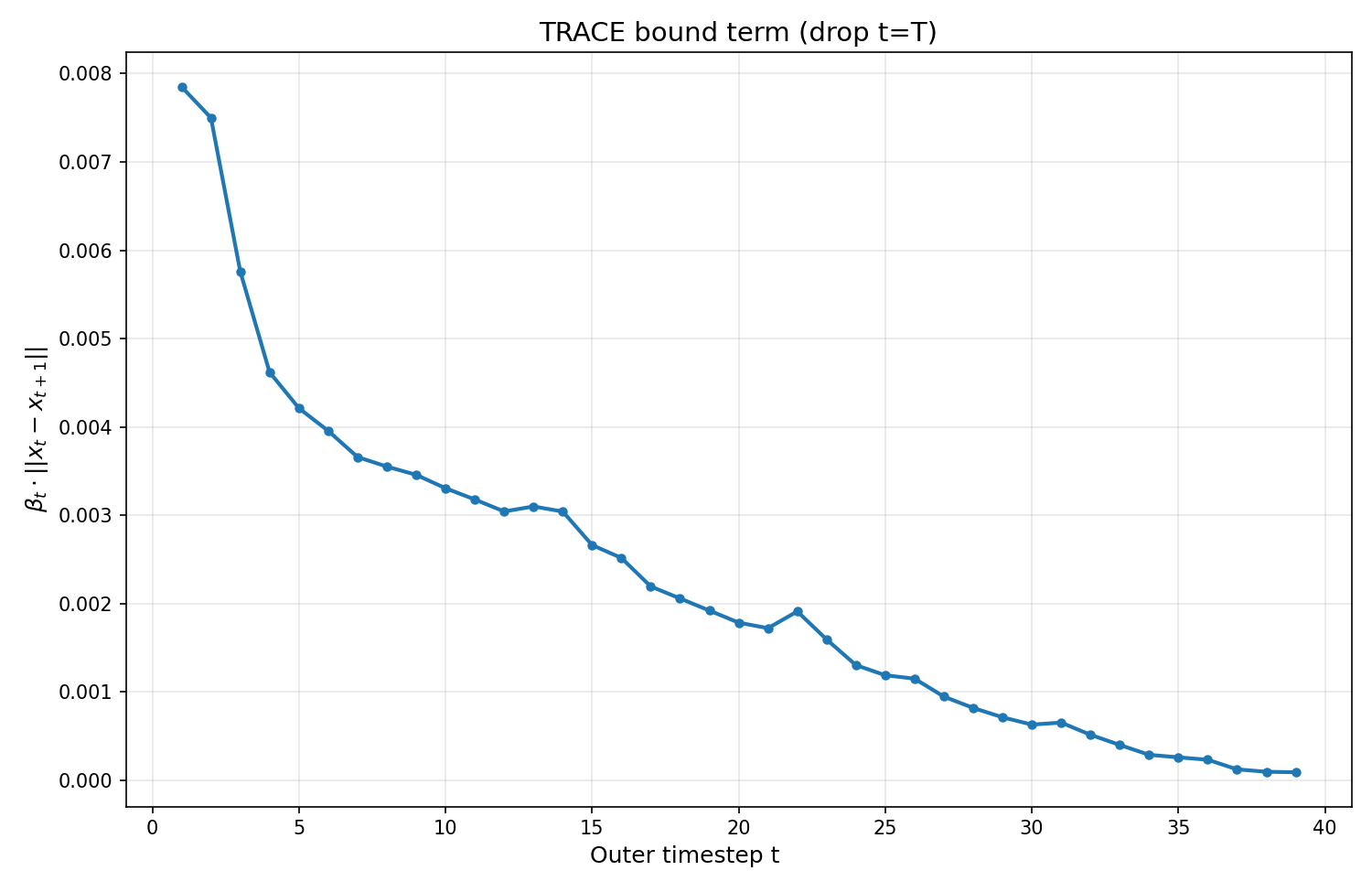}
    \caption{$\beta_t\Delta_t$}
    \label{fig:traj_beta_delta}
\end{subfigure}

\vspace{0.35em}

\begin{subfigure}[t]{0.41\textwidth}
    \centering
    \includegraphics[width=\linewidth, trim=6 4 8 6, clip]{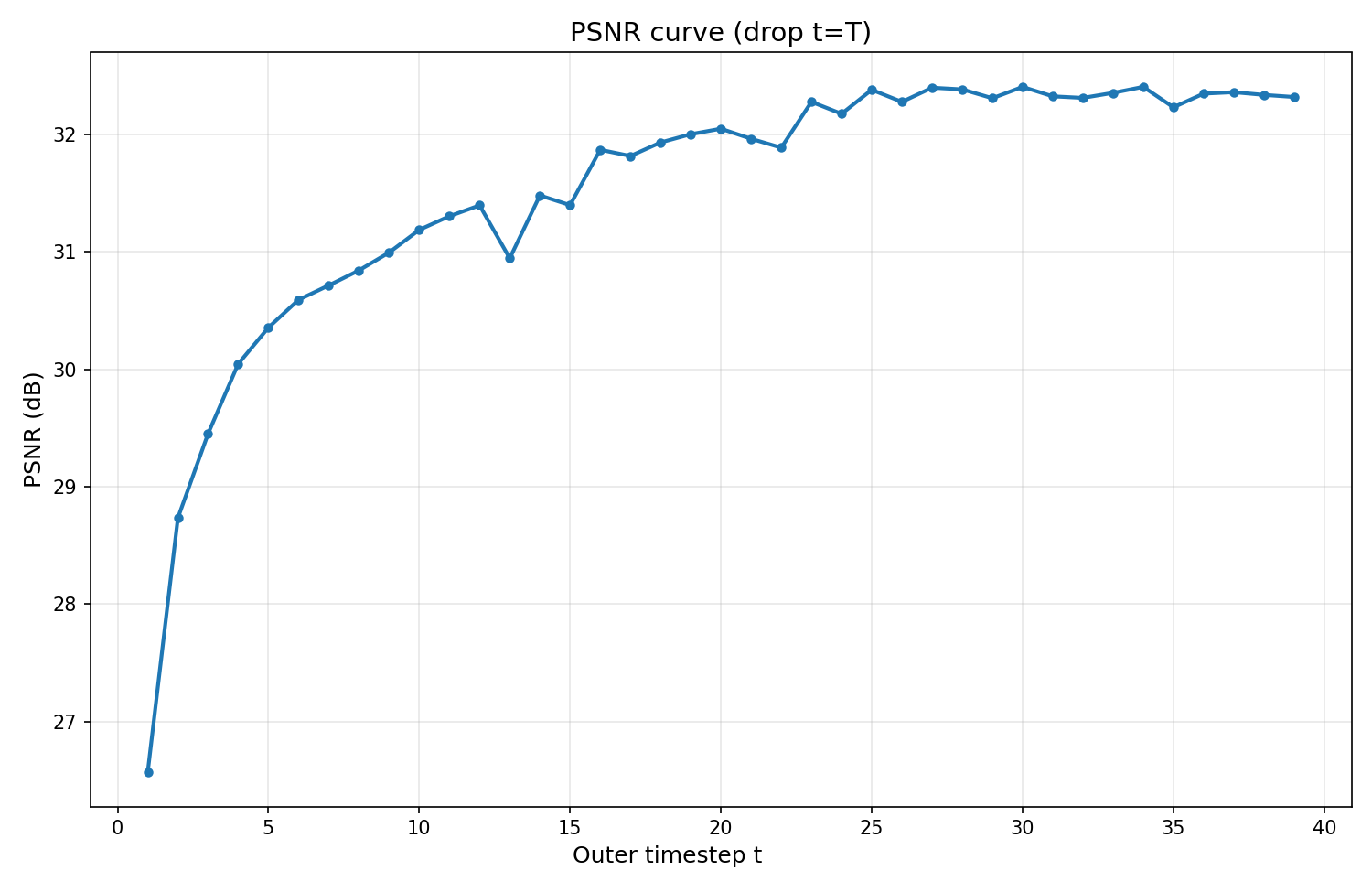}
    \caption{PSNR}
    \label{fig:traj_psnr}
\end{subfigure}
\hspace{0.035\textwidth}
\begin{subfigure}[t]{0.41\textwidth}
    \centering
    \includegraphics[width=\linewidth, trim=6 4 8 6, clip]{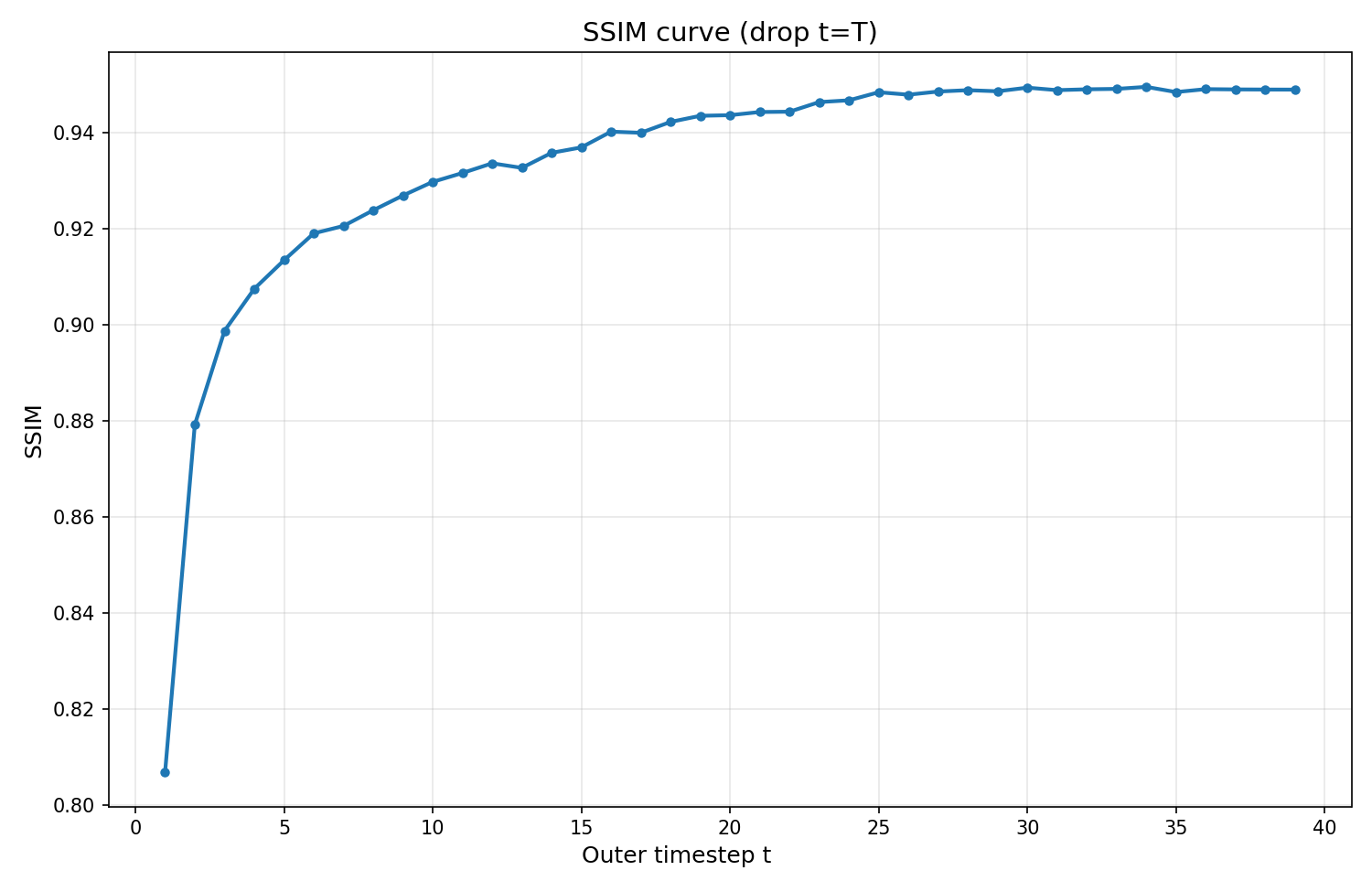}
    \caption{SSIM}
    \label{fig:traj_ssim}
\end{subfigure}

\vspace{0.25em}
\caption{
Representative TRACE trajectory.
Both the state variation and its weighted counterpart generally decrease as reconstruction proceeds, while PSNR and SSIM improve and eventually level off.
This figure illustrates the behavior of the complete TRACE trajectory; controlled comparisons that isolate the effects of stochastic perturbation and proximal coupling are presented in the next subsection.
}
\label{fig:trace_trajectory}
\end{figure}
\begin{table}[t]
\centering
% \scriptsize
\setlength{\tabcolsep}{5.2pt}
\renewcommand{\arraystretch}{1.10}

\caption{
Sensitivity to the proximal coupling schedule for super-resolution and
deblurring.
The first column reports the endpoints of the annealed $\beta_t$ schedule
after scaling by $10^3$.
Best and second-best values are shown in \textbf{bold} and
\underline{underlined}, respectively.
}
\label{tab:beta_ablation}

\begin{tabular}{@{}lccc@{\hspace{10pt}}ccc@{}}
\toprule
\rowcolor{headergray}
\multirow{2}{*}{\textbf{$10^3\beta_t$}}
& \multicolumn{3}{c}{\textbf{Super-Resolution}}
& \multicolumn{3}{c}{\textbf{Deblurring}}
\\
\cmidrule(lr){2-4}
\cmidrule(lr){5-7}

\rowcolor{headergray}
&
\textbf{PSNR $\uparrow$}
&
\textbf{SSIM $\uparrow$}
&
\textbf{LPIPS $\downarrow$}
&
\textbf{PSNR $\uparrow$}
&
\textbf{SSIM $\uparrow$}
&
\textbf{LPIPS $\downarrow$}
\\
\midrule

$0$
& 27.67 & 0.891 & 0.0889
& 27.26 & 0.836 & 0.1757
\\

$[1,\,0.1]$
& 27.85 & 0.892 & \underline{0.0803}
& 27.15 & 0.832 & 0.1828
\\

$[2,\,0.2]$
& 27.68 & 0.890 & 0.0826
& 27.30 & 0.837 & 0.1810
\\

$[3,\,0.3]$
& 27.92 & 0.896 & 0.0852
& 27.19 & 0.833 & 0.1755
\\

$[4,\,0.4]$
& 28.18 & \textbf{0.898} & \textbf{0.0768}
& \underline{27.31} & 0.838 & \textbf{0.1653}
\\

$[5,\,0.5]$
& \textbf{28.23} & \textbf{0.898} & 0.0834
& \textbf{27.49} & \textbf{0.846} & \underline{0.1686}
\\

$[6,\,0.6]$
& 27.69 & 0.891 & 0.0898
& 27.05 & 0.828 & 0.1843
\\

$[7,\,0.7]$
& 27.99 & \underline{0.897} & 0.0832
& 27.28 & \underline{0.839} & 0.1765
\\

\bottomrule
\end{tabular}
\end{table}

The consecutive-state difference $\Delta_t$ decreases substantially over the trajectory, although small non-monotone variations remain. The weighted quantity $\beta_t\Delta_t$ shows a similar overall decay. Meanwhile, both PSNR and SSIM improve and gradually saturate as the trajectory progresses.

These observations are consistent with a reconstruction process in which successive states become progressively closer while reconstruction quality continues to improve. However, this single trajectory is descriptive and does not by itself isolate the effect of proximal coupling, since both stochastic perturbation and the coupling schedule vary along the full TRACE trajectory.

This observation raises a related question: how does the strength of the proximal coupling affect the final reconstruction?
We therefore vary the $\beta_t$ schedule while keeping the remaining TRACE configuration fixed.

The parameter $\beta_t$ determines the strength of the proximal constraint between consecutive reconstruction states. To examine its effect independently of the other TRACE components, we keep the remaining configuration fixed and vary the magnitude of the annealed $\beta_t$ schedule.

We evaluate three images from ImageNet Large Scale Visual Recognition Challenge (ILSVRC) on super-resolution and deblurring.
Table~\ref{tab:beta_ablation} shows that the uncoupled setting ($\beta_t=0$) is competitive, but several nonzero coupling schedules improve the final reconstruction quality. The schedule $[5,0.5]\times10^{-3}$ gives the highest PSNR for both tasks, improving over the uncoupled case by $0.56$ dB for super-resolution and $0.23$ dB for deblurring.
The dependence on the coupling strength is not monotone: increasing $\beta_t$ beyond the intermediate range does not consistently improve PSNR, SSIM, or LPIPS.
Thus, the proximal constraint is beneficial over a range of coupling strengths, while its effect depends on the choice of $\beta_t$.

\subsection{Trajectory Analysis}

TRACE enforces SPC along the reconstruction
trajectory through a proximal constraint that keeps the reconstruction
generated from each stochastically perturbed intermediate trajectory estimate
close to its preceding trajectory state. To understand the mechanism underlying SPC, we
disentangle its stochastic and proximal constituents and study their individual
and joint effects.

To address this question, we study stochasticity and proximal consistency
across a progression of optimization regimes with different smoothness,
convexity, and parameterization properties. Specifically, we first consider
network-free formulations, where the image $\mathbf{x}$ is optimized directly
under explicit regularization, followed by the full network-parameterized TRACE
formulation, where each reconstruction state is represented by an untrained
neural network. This progression follows a strategy used in prior work, where a
mechanism is first characterized in a controlled and interpretable setting
before examining how its behavior changes under higher-dimensional or neural
inverse-problem formulations
\citep{NEURIPS2023_b3411e30, ICLR2024_f4757db8}. This allows us to ask the central
question of our ablation:
\emph{which roles of stochasticity and proximal consistency persist across
different optimization regimes and which behaviors emerge only after
introducing the neural parameterization?}

For each optimization regime, we evaluate four ablation configurations obtained
by independently enabling the stochastic and proximal constituents:
(i) neither constituent, (ii) stochastic perturbation only,
(iii) the proximal constraint only, and (iv) both stochastic perturbation and
the proximal constraint. This $2\times2$ design isolates the individual effects
of the two constituents and their interaction.

We characterize these configurations through two complementary aspects of the
PSNR trajectory. First, the peak PSNR over the outer iterations measures the
best reconstruction quality attained along the trajectory. Second, we examine
the peak-aligned evolution of PSNR after the maximum is reached to assess
whether the attained reconstruction quality is maintained over subsequent
iterations, independently of when the peak occurs. In both classical
regularization methods and network-parameterized reconstruction, prior work has
shown that reconstruction quality can deteriorate under prolonged optimization,
which motivates this second analysis
\citep{chang2024relaxed,bachmayr2009iterative,wang2023earlystopping}.

For the post-peak analysis, each trajectory is aligned with respect to its own
peak. Specifically, letting
\begin{equation}
    t_{\mathrm{peak}}
    =
    \arg\max_t \operatorname{PSNR}(\mathbf{x}_t),
\end{equation}
we re-index the subsequent outer iterations by
$\tau=t-t_{\mathrm{peak}}$ and measure the degradation relative to the peak as
\begin{equation}
    \Delta_{\mathrm{post}}(\tau)
    =
    \operatorname{PSNR}(\mathbf{x}_{t_{\mathrm{peak}}})
    -
    \operatorname{PSNR}(\mathbf{x}_{t_{\mathrm{peak}}+\tau}),
    \qquad \tau\geq0.
\end{equation}
Thus, every trajectory starts from
$\Delta_{\mathrm{post}}(0)=0$, allowing post-peak stability to be compared
independently of the outer iteration at which the peak is attained.

\subsubsection{Network-free setting.}

We first remove the neural parameterization and optimize the image
$\mathbf{x}$ directly under an explicit regularizer $R(\mathbf{x})$.To represent the four ablation configurations within a common formulation, let $c_{\sigma},c_{\beta}\in\{0,1\}$ indicate whether stochastic perturbation and the proximal constraint are active, respectively. At outer iteration $t$, the
inner optimization is initialized as
\begin{equation}
    \mathbf{x}_{t}^{0}
    =
    \mathbf{x}_{t+1}
    +
    c_{\sigma}\sigma_t\boldsymbol{\epsilon}_t,
    \qquad
    \boldsymbol{\epsilon}_t\sim\mathcal{N}(\mathbf{0},\mathbf{I}),
    \label{eq:network_free_initialization}
\end{equation}
and we perform $K$ optimization steps on
\begin{equation}
    F_t(\mathbf{x})
    =
    \frac{1}{2}\|A\mathbf{x}-\mathbf{y}\|_2^2
    +\lambda R(\mathbf{x})
    +c_{\beta}\frac{\beta_t}{2}
    \|\mathbf{x}-\mathbf{x}_{t+1}\|_2^2.
    \label{eq:network_free_ablation}
\end{equation}
The four configurations therefore correspond to
$(c_{\sigma},c_{\beta})=(0,0),(1,0),(0,1)$, and $(1,1)$, representing
plain optimization, perturbation only, coupling only, and the combined
setting, respectively.

We evaluate this controlled formulation under three explicit regularizers chosen
to vary the structural properties of the resulting optimization problem.
Tikhonov regularization yields a smooth convex objective, Total Variation (TV)
introduces convex nonsmooth regularization, and the $L_p$ penalty with $p=0.8$
introduces nonconvexity:
\begin{align}
R_{\mathrm{Tik}}(\mathbf{x})
&=
\frac{1}{2}\|D_x\mathbf{x}\|_2^2
+
\frac{1}{2}\|D_y\mathbf{x}\|_2^2, \\
R_{\mathrm{TV}}(\mathbf{x})
&=
\|D_x\mathbf{x}\|_1
+
\|D_y\mathbf{x}\|_1, \\
R_p(\mathbf{x})
&=
\frac{1}{p}
\sum_{d\in\{D_x\mathbf{x},D_y\mathbf{x}\}}
\sum_i |d_i|^p,
\qquad p=0.8,
\end{align}
where $D_x$ and $D_y$ denote the horizontal and vertical finite-difference
operators. This progression allows us to examine
whether the effects of stochasticity and proximal consistency change with the
smoothness and convexity of the optimization problem before introducing the
nonconvex parameter optimization associated with the neural parameterization.
Such distinctions are particularly relevant for perturbation-based and
proximal-type optimization mechanisms, whose behavior can depend on the
geometry and regularity of the objective
\citep{pmlr-v70-jin17a,davis2021escapingstrictsaddlepoints}.

\subsubsection{Network-parameterized setting.}

We next repeat the ablation under the full network parameterization of TRACE.
Using the same indicators $c_{\sigma},c_{\beta}\in\{0,1\}$, we define the
network input at outer iteration $t$ as
\begin{equation}
    \mathbf{u}_t
    =
    \mathbf{x}_{t+1}
    +
    c_{\sigma}\sigma_t\boldsymbol{\epsilon}_t,
    \qquad
    \boldsymbol{\epsilon}_t
    \sim \mathcal{N}(\mathbf{0},\mathbf{I}),
    \label{eq:network_ablation_input}
\end{equation}
and the network parameters are optimized for $K$ iterations using
\begin{equation}
    \mathcal{L}_t(\boldsymbol{\theta}_t)
    =
    \frac{1}{2}
    \left\|
    A D_{\boldsymbol{\theta}_t}(\mathbf{u}_t)
    -\mathbf{y}
    \right\|_2^2
    +
    c_{\beta}\frac{\beta_t}{2}
    \left\|
    D_{\boldsymbol{\theta}_t}(\mathbf{u}_t)
    -\mathbf{x}_{t+1}
    \right\|_2^2.
    \label{eq:network_ablation}
\end{equation}
As in the network-free setting,
$(c_{\sigma},c_{\beta})=(0,0),(1,0),(0,1)$, and $(1,1)$ correspond to the
plain, stochastic-only, proximal-only, and stochastic-proximal configurations,
respectively. The case $(c_{\sigma},c_{\beta})=(1,1)$ recovers the full TRACE
formulation in Eqs.~\eqref{eq:perturbed_state} and \eqref{eq:training_loss},
while all remaining optimization and architectural choices are held fixed
across the ablation.

\begin{figure*}[t]
    \centering
    \includegraphics[width=\textwidth]{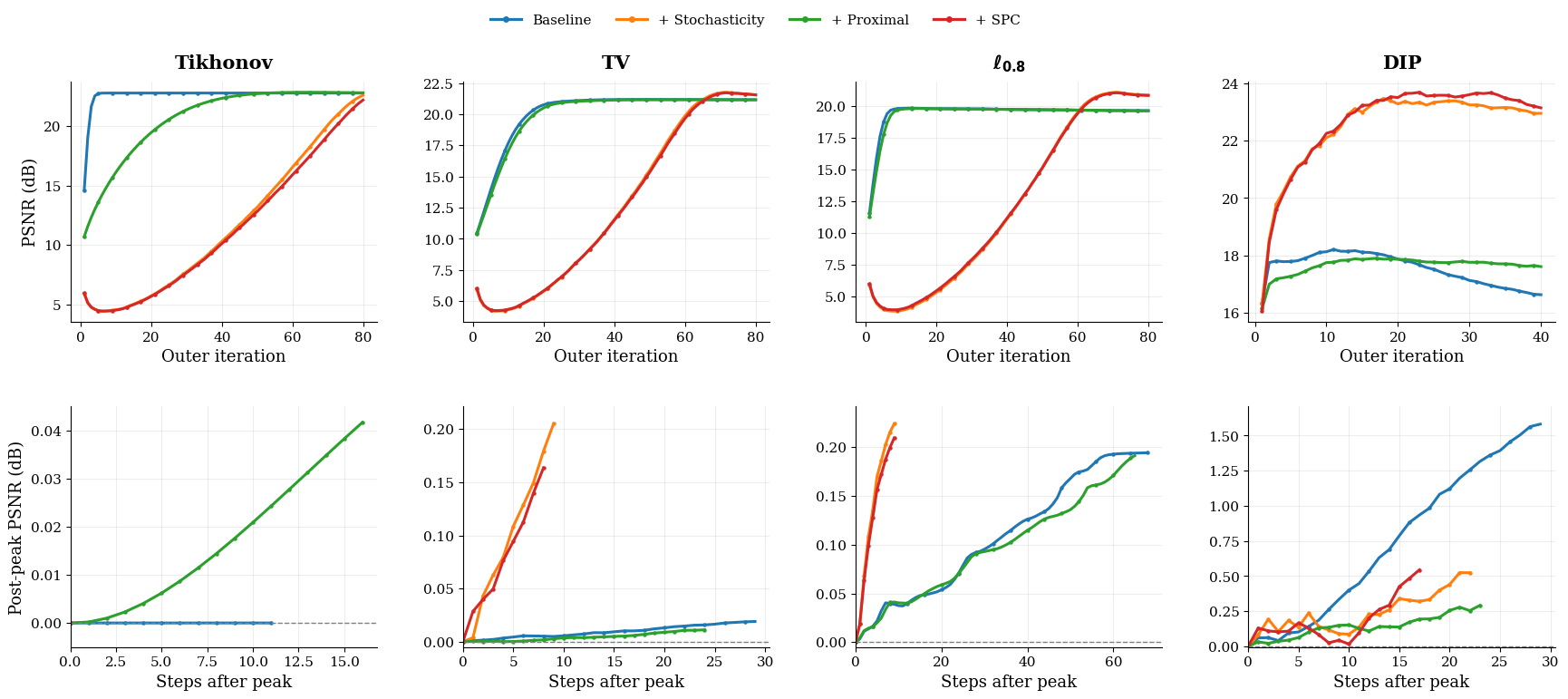}
    \caption{
    Component ablation of stochasticity and the proximal term for Tikhonov, TV,
    $\ell_{0.8}$, and DIP-based reconstruction. Top: PSNR across outer iterations.
    Bottom: peak-aligned post-peak PSNR degradation. For the network-free
    experiments, the stochasticity and proximal schedules are
    $10^{-4}\!\rightarrow\!10^{-2}$ and
    $\beta_t:10^{-2}\!\rightarrow\!10^{-3}$, respectively. For DIP, they are
    $10^{-4}\!\rightarrow\!10^{-2}$ and
    $\beta_t:10^{-3}\!\rightarrow\!10^{-4}$.
    }
    \label{fig:component_ablation}
\end{figure*}

\begin{figure*}[t]
    \centering
    \includegraphics[width=\textwidth]{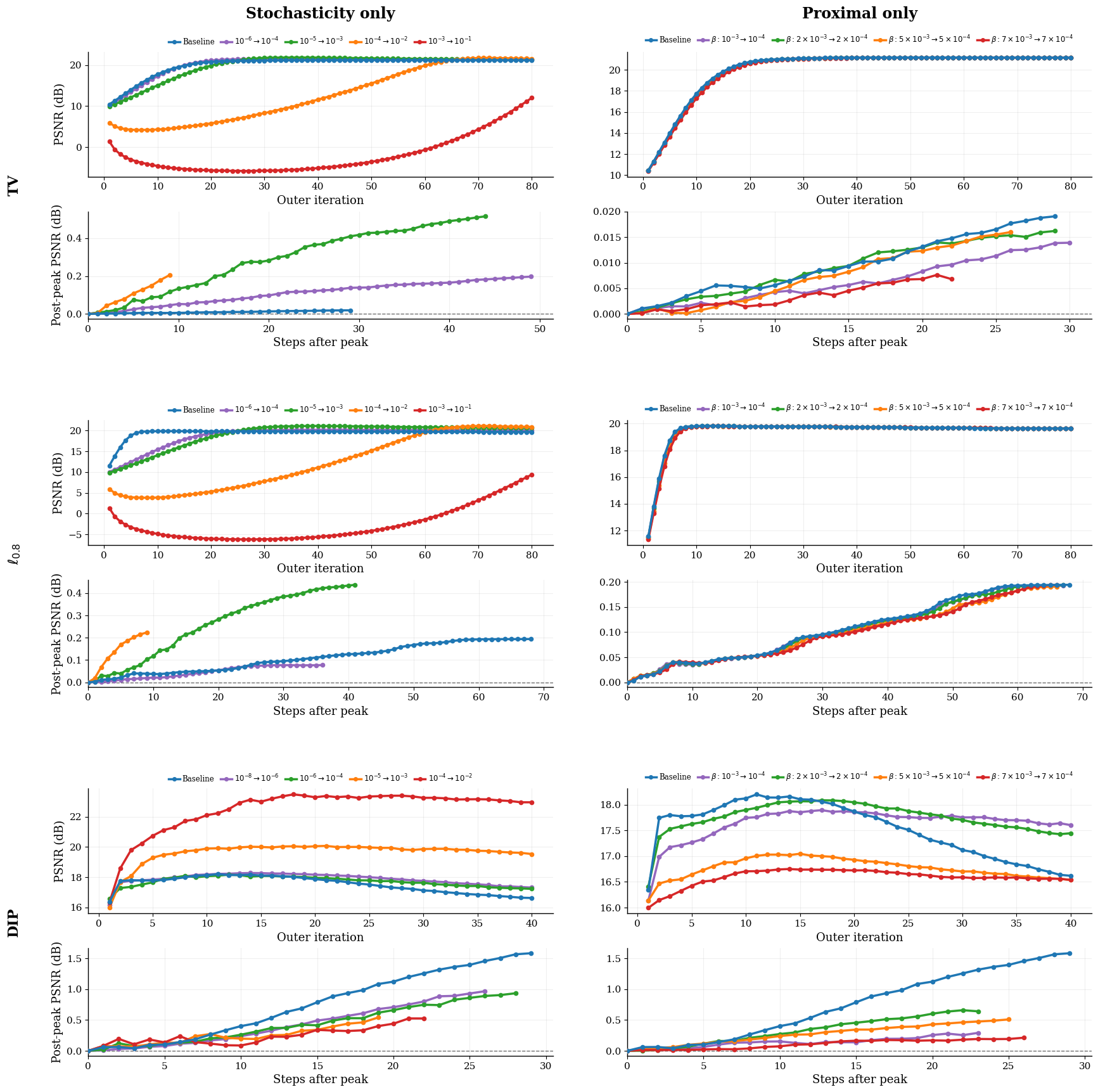}
    \caption{
    Sensitivity to stochasticity and the proximal term in isolation for TV,
    $\ell_{0.8}$, and DIP reconstruction. The left column varies $c_{\alpha}$
    with $c_{\beta}=0$, while the right column varies $c_{\beta}$ with
    $c_{\alpha}=0$. In each setting, the top panel shows PSNR trajectories and
    the bottom panel shows peak-aligned post-peak degradation.
    }
    \label{fig:perturbation_coupling_sweeps}
\end{figure*}

\begin{figure*}[t]
    \centering
    \includegraphics[width=\textwidth]{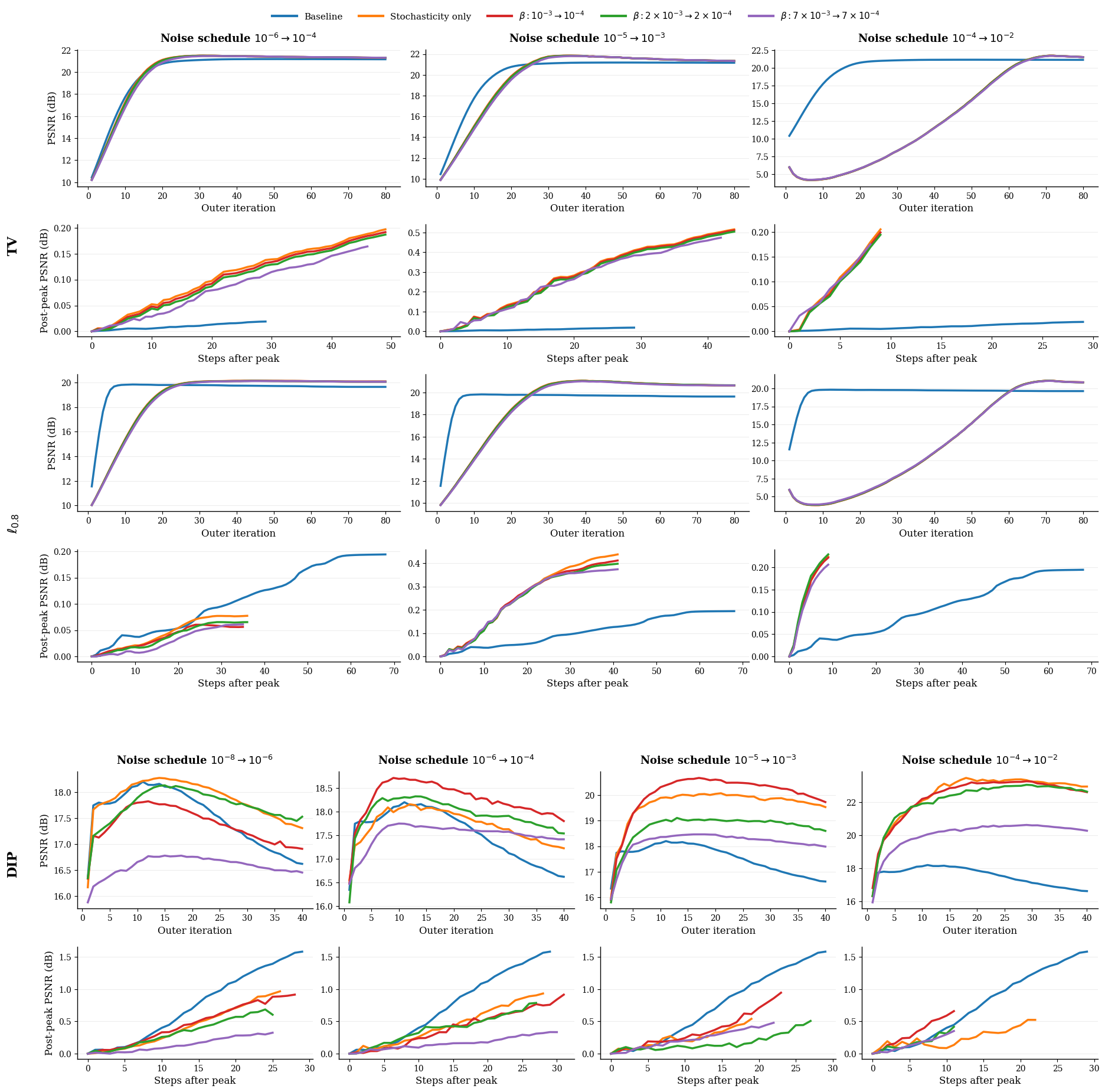}
    \caption{ Joint sensitivity to stochasticity and the proximal term for TV, $\ell_{0.8}$, and DIP reconstruction. Each column fixes $c_{\alpha}$ and compares stochasticity alone ($c_{\beta}=0$) with several choices of $c_{\beta}>0$, together with the baseline. The top panels show PSNR trajectories and the bottom panels show peak-aligned post-peak degradation. The DIP results are shown separately because they use a different range of $c_{\alpha}$. }
    \label{fig:combined_sweeps}
\end{figure*}

\subsubsection{Observations}
\label{sec:observations}

The ablations in Figs.~\ref{fig:component_ablation}--\ref{fig:combined_sweeps}
reveal distinct roles for stochasticity and the proximal term in shaping the
reconstruction trajectory. Across the network-free
experiments,
stochastic perturbation and temporal coupling exhibit distinct but
complementary effects on the reconstruction trajectory. In the smooth convex
Tikhonov setting, the different configurations eventually approach similar
reconstruction quality, with perturbation and coupling mainly affecting the
transient trajectory. Their influence becomes more pronounced under TV and
$\ell_{0.8}$ regularization
(Fig.~\ref{fig:perturbation_coupling_sweeps}, TV and $\ell_{0.8}$ rows).
Perturbation can lead to higher peak PSNR, but the improvement is generally
accompanied by greater post-peak degradation, whereas temporal coupling has a
comparatively small effect on the attainable peak while reducing post-peak
variation.

The network-parameterized DIP experiments reveal a stronger and more nonlinear
interaction between the two mechanisms. Increasing the perturbation strength
produces substantially higher peak PSNR
(Fig.~\ref{fig:perturbation_coupling_sweeps}, DIP row, stochasticity-only
column), while, unlike the network-free experiments, the stronger perturbation
levels also exhibit less post-peak degradation than the baseline trajectory.
Thus, stochastic perturbation in the neural setting is associated not only
with reaching higher-quality reconstructions but also with improved trajectory
stability. This contrast is consistent with our local analysis, where input
perturbation induces a sensitivity penalty proportional to
$\sigma_t^2\|A J_{D_{\boldsymbol{\theta}_t}}\|_F^2$, and suggests that its
interaction with the neural parameterization may contribute to the observed
stability.

The joint sweeps in Fig.~\ref{fig:combined_sweeps} further show that temporal
coupling can modify this behavior. For several perturbation schedules,
appropriate coupling strengths further reduce post-peak degradation relative
to perturbation alone and can also attain higher peak PSNR than both the
baseline and perturbation-only trajectories
(see the DIP row of Fig.~\ref{fig:combined_sweeps} across the intermediate
perturbation schedules). This ordering is not uniform at the strongest
perturbation level: for the $10^{-4}\!\rightarrow10^{-2}$ noise schedule
(rightmost column of the DIP block in Fig.~\ref{fig:combined_sweeps}),
perturbation alone can be more stable than some of the coupled trajectories.
The strong dependence on the relative perturbation and coupling strengths
highlights the role of stochastic proximal consistency in shaping the
reconstruction trajectory under the neural parameterization, while a rigorous
characterization of this interaction remains an interesting direction for
future work.

\FloatBarrier
\section{Empirical Analysis of Stochastic Proximal Consistency}
\label{app:emprical}

Proposition~\ref{prop:interaction} shows that SPC introduces an additional
Jacobian-dependent regularization whose strength is jointly controlled by
$\beta$ and $\sigma^2$. We next examine how this regularization is reflected
in the reconstruction dynamics through two questions. First, does SPC
regularize reconstruction components that lie in the null space of $A$?
Second, does SPC alter the spectral bias of the untrained network?

\subsection{Null-Space Regularization}
\label{sec:nullspace_exp}

We empirically examine whether the predicted null-space regularization is
reflected in the reconstruction trajectory. We decompose each reconstruction
into its components in $\operatorname{range}(A^\top)$ and
$\operatorname{null}(A)$, and track the reconstruction accuracy in the full
image space and in each subspace separately. This allows us to examine how
SPC affects reconstruction errors that are invisible to the measurement
operator.

\begin{figure}[t]
    \centering
    \includegraphics[width=\linewidth]{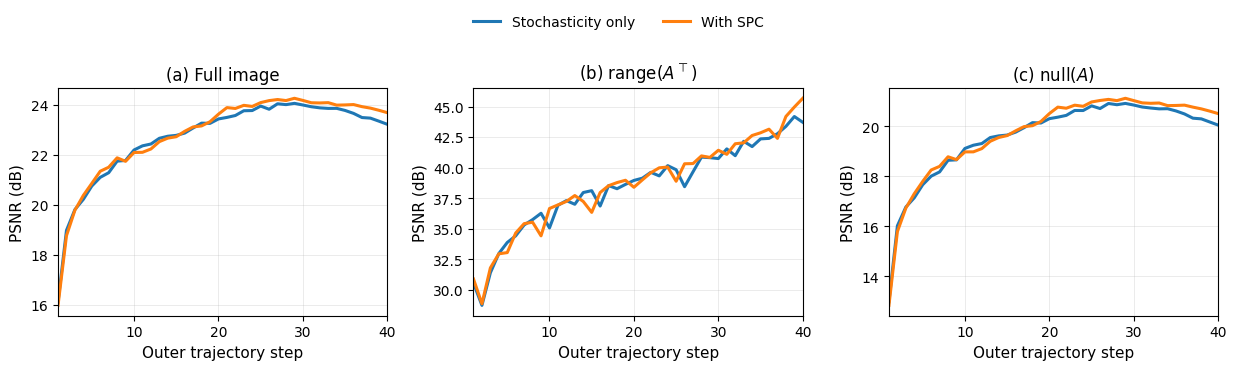}
    \caption{
    Reconstruction accuracy across the TRACE trajectory for image inpainting.
    Stochasticity alone is compared with stochasticity combined with the proximal
    term. PSNR is shown for the full reconstruction and for its components in
    $\operatorname{range}(A^\top)$ and $\operatorname{null}(A)$.    }
    \label{fig:nullspace_analysis}
\end{figure}

\begin{figure}[t]
    \centering
    \includegraphics[width=0.82\linewidth]{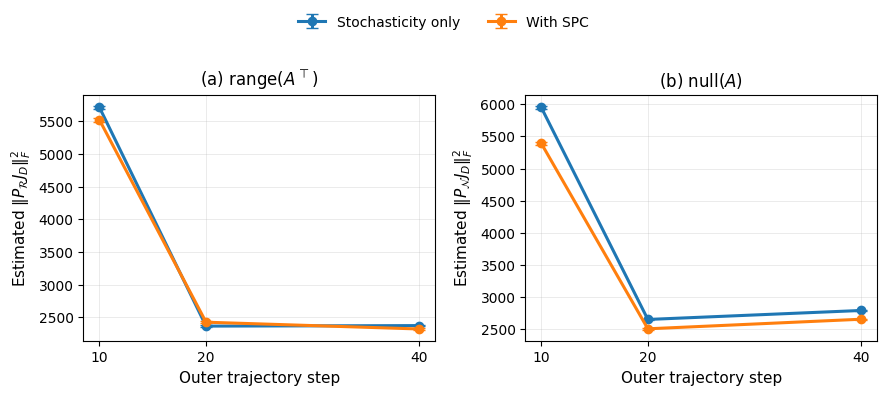}
    \caption{
    Estimated Jacobian sensitivity along the TRACE trajectory for image
    inpainting. We compare stochastic perturbation alone with stochastic
    perturbation combined with temporal coupling. The range-space and
    null-space Jacobian energies are estimated using the Monte Carlo
    finite-difference procedure described in Sec.~\ref{sec:nullspace_exp}.
    Temporal coupling reduces the estimated null-space Jacobian energy at
    all sampled trajectory steps.
    }
    \label{fig:nullspace_jacobian}
\end{figure}

We consider image inpainting, for which the range--null decomposition can be
computed exactly. Let $A(x)=M\odot x$, where $M$ is the binary sampling mask.
Since $A$ is an orthogonal projection in this setting,
$P_{\mathcal R}=M$ and $P_{\mathcal N}=I-M$ project onto the observed and
missing image components, respectively. At each outer trajectory step, we
compute the PSNR of the complete reconstruction, the observed component
$P_{\mathcal R}\hat{x}$, and the missing component
$P_{\mathcal N}\hat{x}$ with respect to the corresponding components of the
ground-truth image. We compare stochastic perturbation alone with the full SPC
formulation to isolate the effect of proximal consistency on
reconstruction accuracy in the observed and missing image components.

As shown in Fig.~\ref{fig:nullspace_analysis}, SPC provides little change in
the observed component but consistently improves null-space PSNR. The gain in
full-image PSNR is therefore primarily driven by more accurate reconstruction
of the measurement-invisible component.

The null-space reconstruction error, however, does not directly measure the
Jacobian sensitivity appearing in Proposition~\ref{prop:interaction}. We
therefore approximately estimate
$\|P_{\mathcal N}J_D(x)\|_F^2$ along the reconstruction trajectory. Following
Monte Carlo finite-difference Jacobian estimation
\citep{ramani2008montecarlo}, we approximate Jacobian-vector products along
Gaussian random directions using central finite differences. For
$v\sim\mathcal{N}(0,I)$,
\[
P_{\mathcal N}J_D(x)v
\approx
P_{\mathcal N}
\frac{
D_\theta(x+h v)-D_\theta(x-h v)
}{
2h
}.
\]

Since $\mathbb{E}[vv^\top]=I$,
\[
\mathbb{E}_{v}
\left[
\|P_{\mathcal N}J_D(x)v\|_2^2
\right]
=
\|P_{\mathcal N}J_D(x)\|_F^2.
\]
We therefore use the Monte Carlo estimator
\[
\widehat{\mathcal{S}}_{\mathcal N}(x)
=
\frac{1}{K}
\sum_{k=1}^{K}
\left\|
P_{\mathcal N}
\frac{
D_\theta(x+h v_k)-D_\theta(x-h v_k)
}{
2h
}
\right\|_2^2,
\]
where $v_k\overset{\mathrm{i.i.d.}}{\sim}\mathcal{N}(0,I)$. In our
experiments, we use $K=32$ Gaussian random vectors and set $h=10^{-3}$.
At each selected outer trajectory step, we fix the network parameters after
the inner optimization and evaluate the finite-difference approximation
around the corresponding unperturbed trajectory state $x_{\mathrm{t+1}}$.
The same Gaussian random vectors are used for the perturbation-only and SPC
variants to enable a paired comparison. Due to the computational cost of the
finite-difference estimation, we evaluate the Jacobian sensitivity at
$t\in\{10,20,40\}$.

As shown in Fig.~\ref{fig:nullspace_jacobian}, the estimated null-space
Jacobian sensitivity is consistently lower with SPC than with stochastic
perturbation alone, while the range-space sensitivity remains comparable.
This is consistent with Proposition~\ref{prop:interaction}, as SPC reduces
Jacobian sensitivity primarily in the measurement-invisible subspace.
Together with the improved null-space PSNR, these results indicate that the
additional regularization acts primarily on components that are not
constrained by the measurements.

\subsection{Spectral Bias Modulation}
\label{sec:spectralbias_exp}

Deep image priors are known to exhibit a spectral bias, fitting
low-frequency image components earlier than high-frequency components.
Following \citet{shi2021measuringcontrollingspectralbias}, we use
frequency-band correspondence (FBC) to examine whether SPC modulates this
spectral progression. For a reconstruction $\hat{x}_t$ and ground-truth image
$x^\star$, the Fourier-domain correspondence is written as
\[
H_t
=
\frac{\mathcal{F}\{\hat{x}_t\}}
     {\mathcal{F}\{x^\star\}},
\]
where $\mathcal{F}$ denotes the Fourier transform. The correspondence map is
divided into non-overlapping frequency bands, and the mean correspondence
within each band defines its FBC value $C_n(t)$. We use five bands, with
$C_1$ and $C_5$ denoting the lowest- and highest-frequency bands,
respectively.

\begin{figure}[t]
    \centering
    \includegraphics[width=\linewidth]{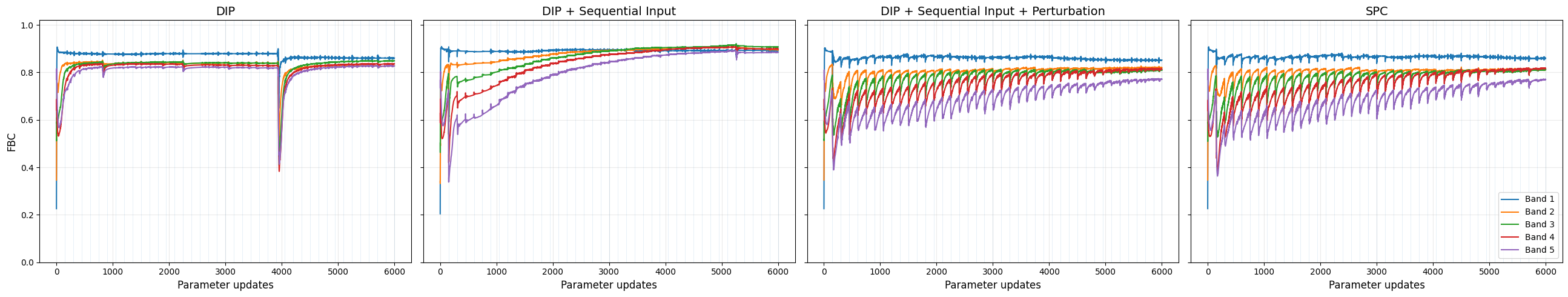}
    \caption{
    Frequency-band correspondence (FBC) over 6000 optimization updates for
    four reconstruction variants. Standard DIP uses a fixed network input,
    while the second variant sequentially feeds the previous reconstruction
    as the input to the next outer step. Adding stochastic perturbation
    produces repeated frequency-dependent drops and recovery across outer
    steps, with larger changes at higher frequencies. SPC exhibits the same
    overall stochastic spectral dynamics while introducing a more subtle
    modulation of the frequency progression.
    }
    \label{fig:spectral_bias}
\end{figure}

\begin{figure}[t]
    \centering
    \includegraphics[width=\linewidth]{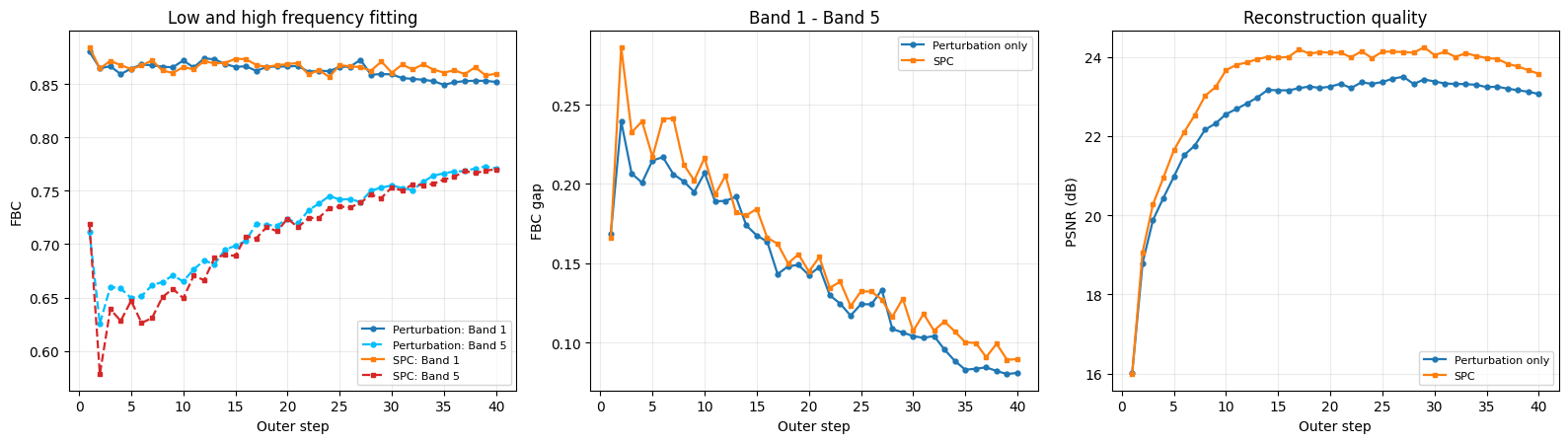}
    \caption{
    Comparison of spectral progression and reconstruction quality for
    stochastic perturbation alone and SPC. The ``Low and high frequency
    fitting'' plot compares the lowest- and highest-frequency FBC
    trajectories. The ``Band 1 - Band 5'' plot reports the low--high FBC gap,
    $C_1-C_5$. The ``Reconstruction quality'' plot compares PSNR across the
    outer trajectory.
    }
    \label{fig:spectral_bias_gap}
\end{figure}

Figure~\ref{fig:spectral_bias} compares the FBC trajectories of four
variants trained for the same number of optimization steps. Standard DIP
shows the expected low-to-high frequency ordering, but the bands quickly
approach similar correspondence levels. Sequentially feeding the previous
reconstruction as the network input prolongs this separation, although the
frequency bands again converge relatively early. At each outer step,
stochastic perturbation lowers the frequency correspondence. The inner
optimization then progressively restores it, resulting in the repeated
sawtooth structure in the FBC trajectories. The perturbation-only and SPC
trajectories exhibit the same overall pattern, indicating that stochastic
perturbation is primarily responsible for the repeated frequency-dependent
reset and recovery.

The perturbation-only and SPC trajectories are qualitatively similar in
Fig.~\ref{fig:spectral_bias}, making the effect of SPC difficult to
distinguish from the full FBC trajectories. The ``Low and high frequency
fitting'' plot in Fig.~\ref{fig:spectral_bias_gap} therefore focuses on
$C_1$ and $C_5$. To quantify their separation, the ``Band 1 - Band 5'' plot
reports
\[
G(t)=C_1(t)-C_5(t).
\]
SPC maintains a slightly larger FBC gap than stochastic perturbation alone
over much of the trajectory, indicating a stronger separation between
low- and high-frequency fitting.

Importantly, a larger FBC gap is not itself a measure of reconstruction
quality. \citet{shi2021measuringcontrollingspectralbias} observed that DIP
can attain its best reconstruction quality while low-frequency components
are already well matched and high-frequency components remain less fully
incorporated, and showed that controlling this spectral progression can
mitigate subsequent performance degradation. In our setting, the increased
low-high separation with SPC coincides with improved reconstruction PSNR,
as shown in the ``Reconstruction quality'' plot of
Fig.~\ref{fig:spectral_bias_gap}.

Together, these observations suggest that the additional Jacobian
regularization induced by SPC modulates the spectral bias of the untrained
network, rather than uniformly changing the fitting rate across frequencies.
% Revised "Algorithmic Details" appendix. Every value below was checked against the saved run configs / runner
% defaults of the runs that produce the current tables (see the chat summary for sources).
% Lines marked TODO could not be verified from saved configs and must be confirmed by the authors.
\section{Algorithmic Details}
\label{app:algo}

This section summarizes the evaluation protocol and the implementation settings of TRACE and all compared methods.
Details specific to CT reconstruction are given in Appendix~\ref{app:additional_ct}, and details of the SPC
experiments on diffusion solvers in Appendix~\ref{app:spc_details}.

\paragraph{Evaluation protocol.}
All methods are run on the same measurements, generated by the same forward operators. The natural-image
benchmark uses $14$ images ($10$ from BSDS68 and $4$ from Set14), resized to $256\times256$, and six tasks:
random inpainting with $50\%$ and $70\%$ of the pixels missing (a random pixel-wise mask shared by the three color
channels); super-resolution by $\times2$ and $\times4$ average-pooling downsampling; motion deblurring with a $21\times21$ linear motion kernel at $45^\circ$; and for nonlinear deblurring, we follow DPS~\citep{chung2023dps} and use the BKSE forward
model~\citep{tran2021bkse}: the measurement is generated as $y=\mathcal{F}_\phi(x,k)+n$, where
$\mathcal{F}_\phi$ is a pretrained blur network and $k$ is a randomly sampled latent blur kernel,
with Gaussian noise of standard deviation $0.05$. The CT benchmark ($24$ slices, sparse-view and limited-angle)
is described in Appendix~\ref{app:ct_data}. We report PSNR, SSIM and LPIPS for all tasks, computed on
images in $[0,1]$ at $256\times256$ resolution; for CT, LPIPS is computed on the grayscale slice replicated to
three channels.

\paragraph{Reported output.}
TRACE, SDI, DIP and all diffusion-based and CNN-based methods report their final output; DIP reports the iterate with the
lowest data-fidelity loss, which does not use the ground truth. For ASeqDIP, Self-Guided DIP, UGoDIT and uDiG-DIP we
use the evaluation mode of their implementations, which returns the iterate with the highest PSNR with respect to
the ground truth.

\paragraph{Hardware and software.}
All experiments are implemented in PyTorch and run on NVIDIA RTX PRO 6000 Blackwell Server Edition GPUs. Random
seeds are fixed for all methods.

% \subsection{TRACE}
% \label{app:algo_trace}
% For all natural-image tasks, TRACE uses $T=40$ trajectory states and $K=150$ inner Adam iterations per state with
% learning rate $10^{-3}$, and the network weights are inherited from one state to the next. The network is a
% randomly initialized U-Net with five downsampling levels ($64$ to $1024$ channels) and bilinear upsampling, with
% about $88$M parameters. The stochastic perturbation level $\sigma_t$ follows the DDPM schedule,
% $\sigma_t=\sqrt{1-\bar\alpha_t}$ with $\bar\alpha_t=\prod_{i\le t}(1-\beta^{\mathrm{ddpm}}_i)$ and
% $\beta^{\mathrm{ddpm}}$ increasing linearly from $10^{-4}$ to $10^{-2}$, traversed from the noisiest to the cleanest
% state; for $T=40$, $\sigma_t$ decreases from about $0.43$ to $0.01$. The consistency weight $\beta_t$ decreases
% linearly over the states, from $5\times10^{-4}$ to $5\times10^{-5}$ for inpainting and super-resolution and from
% $5\times10^{-2}$ to $5\times10^{-3}$ for the two deblurring tasks. We use random seed $42$ and report the output
% of the final state. The CT settings are given in Appendix~\ref{app:ct_trace}. All ablations in the paper use this
% same implementation.
\subsection{TRACE}
\label{app:algo_trace}

For all natural-image tasks, TRACE uses $T=40$ outer optimization
steps and $K=150$ inner Adam iterations per step, with a learning rate
of $10^{-3}$. The first outer step, indexed by $t=T$, uses the
uncoupled DIP objective for $K$ iterations to obtain the initialization
$x_T$. The remaining $T-1$ outer steps use the SPC objective, with the
first SPC step starting at $t=T-1$. Thus, the complete reconstruction
performs
\[
    K+(T-1)K=TK=6000.
\]
Adam updates. The network weights are inherited from one outer step to
the next.

The network is a randomly initialized U-Net with five downsampling
levels ($64$ to $1024$ channels), bilinear upsampling, and
approximately $88$M parameters.

The stochastic perturbation level $\sigma_t$ follows a DDPM-style noise
schedule. Specifically,
\[
    \sigma_t=\sqrt{1-\bar{\alpha}_t},
    \qquad
    \bar{\alpha}_t=\prod_{i=0}^{t}(1-v_i),
\]
where the noise-variance parameter $v_i$ increases linearly from
$10^{-4}$ at $i=0$ to $10^{-2}$ at $i=T-1$. After the DIP
initialization at $t=T$, the SPC trajectory begins with
$\sigma_{T-1}$ and follows the reverse noise schedule toward the
cleanest noise level. For $T=40$, the noise level decreases from
approximately $0.43$ to $0.01$. We use $v_i$ exclusively for the
DDPM-style noise-variance schedule and reserve $\beta_t$ for the
consistency weight.

The consistency weight is defined as
\[
    \beta_t
    =
    \beta_{\min}
    +
    \frac{t}{T-1}
    \left(\beta_{\max}-\beta_{\min}\right).
\]
Consequently, the consistency weight decreases along the SPC
trajectory from $\beta_{T-1}=\beta_{\max}$ at the first SPC step to
$\beta_0=\beta_{\min}$ at the final SPC step.

For the main natural-image experiments, we use
\[
    (\beta_{\max},\beta_{\min})
    =
    (5\times10^{-4},5\times10^{-5})
\]
for inpainting and super-resolution, and
\[
    (\beta_{\max},\beta_{\min})
    =
    (5\times10^{-2},5\times10^{-3})
\]
for the two deblurring tasks.

We use random seed $42$ and report the output of the final outer step.
The CT settings are provided in Appendix~\ref{app:ct_trace}. Unless
explicitly varied, all ablations use the same implementation and
hyperparameters.

\subsection{Compared methods}
\label{app:algo_baselines}
All baselines are run with the authors' released code, or with the reference implementation of a public library
where stated, using their default hyperparameters unless noted below. For the neural-prior methods, $T$ denotes the
number of outer stages and $K$ the number of inner optimization steps per stage.

\paragraph{DIP~\citep{ulyanov2018dip}.}
We use the standard skip-connection encoder--decoder of DIP ($5$ scales, $128$ channels), a $32$-channel random
input with input-noise regularization of standard deviation $1/30$, and $3000$ Adam iterations with learning rate
$10^{-2}$.

\paragraph{ASeqDIP~\citep{alkhouri2024image}.}
For natural images we use $T=600$ stages with $K=2$ inner steps, learning rate $3\times10^{-4}$ and
autoencoding-regularization weight $\lambda=1$. For CT we use $T=600$ stages, learning rate $2\times10^{-4}$ and
$\lambda=1$.
% TODO: the number of inner steps K for the CT runs (AseqDIP_new-ct, tag ct_recon) is not recorded; the runner
% default is K=5 but the default T=1200 was also overridden (600 stages were logged). Confirm K before submission.

\paragraph{Self-Guided DIP~\citep{liang2025analysis}.}
We run $1200$ iterations with learning rate $3\times10^{-4}$ for the network and $10^{-1}$ for the reference
input, $3$ smoothing samples and an exponential moving average with weight $0.99$. The guidance weight is
$\alpha=2$ for natural images.
% TODO: the CT runs carry the tag "sgdip_ct_a0", which suggests alpha = 0 for CT; confirm and state it.

% \paragraph{SDI~\citep{luo2025selfdiffusion}.}
% SDI uses the same U-Net, $T=40$ steps, $K=150$ inner Adam iterations with learning rate $10^{-3}$, and the same
% DDPM noise schedule ($\beta^{\mathrm{ddpm}}$ from $10^{-4}$ to $10^{-2}$) as TRACE.
\paragraph{SDI~\citep{luo2025selfdiffusion}.}
SDI uses the same U-Net, $T=40$ trajectory steps, $K=150$ inner Adam
iterations per step with learning rate $10^{-3}$, and the same
DDPM-style noise-variance schedule $\{v_i\}_{i=0}^{T-1}$ as TRACE.

\paragraph{UGoDIT~\citep{liangugodit}.}
UGoDIT uses $T=300$ stages with $K=2$ inner steps and learning rate $10^{-2}$. Its encoder is initialized with the
transferable weights learned by unsupervised group training and kept frozen, and only the decoder is optimized.

\paragraph{uDiG-DIP~\citep{liang2025sequential}.}
uDiG-DIP is evaluated on CT only, with its released CT diffusion prior (the same prior as DDS), $T=300$ stages,
learning rate $2\times10^{-4}$ and $\lambda=0$.
% TODO: K for these runs (uDIGDIP_new, tag run1) is not recorded; runner default is K=5. Confirm.

\paragraph{SITCOM~\citep{alkhouri2025sitcom}.}
We use the released implementation with $20$ sampling steps, $30$ inner optimization steps per sampling step and
learning rate $0.02$. The diffusion prior is the unconditional natural-image pre-trained model released with SITCOM, for both natural
images and CT.

\paragraph{DPS, DAPS, RED-diff and PnP-DM.}
These four solvers are run with the implementations and default configurations of InverseBench~\citep{zheng2025inversebench},
all with the same unconditional natural-image diffusion prior: DPS with $1000$ steps and guidance scale $10$; DAPS with
$200$ annealing steps, $5$ diffusion steps per annealing step and $100$ Langevin steps (learning rate $10^{-4}$,
$\tau=0.01$); RED-diff with $1000$ steps, learning rate $0.02$ and regularization weight $0.25$; and PnP-DM with
$200$ annealing steps and $100$ Langevin steps (learning rate $10^{-4}$, $\tau=0.05$). The forward operators and
measurements are the same as for all other methods, and PSNR and SSIM are those computed by the InverseBench image
evaluator.

\paragraph{PnP-DnCNN and PnP-FFDNet.}
Both use the plug-and-play proximal-gradient solver of the DeepInverse library with the released color DnCNN and
FFDNet denoisers, denoiser noise level $0.01$, step size $1$ (with the forward operator normalized to unit spectral
norm), and at most $200$ iterations with early stopping.

\paragraph{IRCNN, RED-CNN, EDCNN and DDS.}
These methods are evaluated on CT only; their settings are given in Appendix~\ref{app:ct_baselines}.

\section{Additional Details and Results for CT Reconstruction}
\label{app:additional_ct}

% \subsection{Data and forward operators}
% \label{app:ct_data}
% We use full-dose abdominal CT scans from the AAPM Low-Dose CT Grand Challenge dataset~\cite{mccollough2017grandchallenge} (1\,mm slice thickness, B30 reconstruction kernel). The test set consists of $24$ slices: twelve consecutive slices
% from each of two patients (L067 and L109). Each slice is normalized to $[0,1]$ by per-slice min--max scaling of the
% stored DICOM values and resized to $256\times256$ by bilinear interpolation. Both tasks use a 2D parallel-beam
% Radon transform applied to the image masked by its inscribed circle, and the sinogram is divided by the maximum of
% the ground-truth sinogram. Sparse-view CT uses $60$ projection angles uniformly spaced over $[0^\circ,180^\circ)$;
% limited-angle CT uses $120$ angles from $0^\circ$ to $119^\circ$ in $1^\circ$ steps, so that the remaining
% $60^\circ$ of the angular range is never observed. Measurements are noiseless.
% PSNR and SSIM are computed on the $[0,1]$ single-channel reconstruction with respect to the normalized ground
% truth. LPIPS uses the AlexNet backbone, with the grayscale image replicated to three channels and scaled to
% $[-1,1]$. All metrics in Table~\ref{tab:ct} are averaged over the $24$ test slices.

\subsection{Data and forward operators}
\label{app:ct_data}
We evaluate CT reconstruction in two challenging tomographic settings, sparse-view and limited-angle,
using full-dose abdominal CT scans from the AAPM Low-Dose CT Grand Challenge dataset~\cite{mccollough2017grandchallenge}
(1\,mm slice thickness, B30 reconstruction kernel). The test set consists of $24$ slices: twelve consecutive slices
from each of two patients (L067 and L109). Each slice is normalized to $[0,1]$ by per-slice min--max scaling of the
stored DICOM values and resized to $256\times256$ by bilinear interpolation. Both tasks use a 2D parallel-beam
Radon transform applied to the image masked by its inscribed circle, and the sinogram is divided by the maximum of
the ground-truth sinogram. Sparse-view CT uses $60$ projection angles uniformly spaced over $[0^\circ,180^\circ)$,
which preserves the full angular coverage but severely reduces the number of measurements. Limited-angle CT uses
$120$ contiguous angles from $0^\circ$ to $119^\circ$ in $1^\circ$ steps, so that the remaining $60^\circ$ of the
angular range is never observed, which leads to missing-angle artifacts and strong directional ambiguity.
Measurements are noiseless.
PSNR and SSIM are computed on the $[0,1]$ single-channel reconstruction with respect to the normalized ground
truth. LPIPS uses the AlexNet backbone, with the grayscale image replicated to three channels and scaled to
$[-1,1]$. All metrics in Table~\ref{tab:ct} are averaged over the $24$ test slices.

\subsection{Implementation of TRACE}
\label{app:ct_trace}
% TRACE uses the same U-Net for both tasks, with five downsampling levels (from $64$ to $1024$ channels), bilinear
% upsampling, a single input and output channel and about $88$M parameters, initialized randomly. The network is
% optimized with Adam for $150$ inner iterations per state. The injected noise follows
% a DDPM schedule with $\beta$ increasing linearly from $10^{-4}$ to $10^{-2}$, and the weight of the proximal
% consistency term decreases linearly over the states. For sparse-view CT we use $300$ states, learning rate
% $10^{-3}$ and a consistency weight decreasing from $2\times10^{-3}$ to $2\times10^{-4}$; for limited-angle CT we use
% $80$ states, learning rate $5\times10^{-4}$ and a consistency weight decreasing from $5\times10^{-4}$ to
% $5\times10^{-5}$. No training data or pretrained weights are used, and one reconstruction is produced per slice.
TRACE uses the same U-Net for both tasks, with five downsampling levels
(from $64$ to $1024$ channels), bilinear upsampling, a single input and
output channel, and approximately $88$M parameters. The network is
randomly initialized and optimized with Adam for $K=150$ inner
iterations per outer step. As in Appendix~\ref{app:algo_trace}, the
first outer step, indexed by $t=T$, uses the uncoupled DIP objective to
obtain the initialization $x_T$. The subsequent trajectory starts at
$t=T-1$ and uses the SPC objective. The network weights are inherited
from one outer step to the next, and the complete reconstruction uses
$TK$ Adam updates.

The injected noise follows the same DDPM-style noise-variance schedule
as in Appendix~\ref{app:algo_trace}, with $v_i$ increasing linearly from
$10^{-4}$ at $i=0$ to $10^{-2}$ at $i=T-1$. After the DIP
initialization, the SPC trajectory begins with $\sigma_{T-1}$ and
proceeds toward the cleanest noise level.

The SPC weight $\beta_t$ decreases along the SPC trajectory, with
$\beta_{T-1}=\beta_{\max}$ at the first SPC step and
$\beta_0=\beta_{\min}$ at the end of the trajectory. For sparse-view
CT, we use $T=300$, a learning rate of $10^{-3}$, and
\[
(\beta_{\max},\beta_{\min})
=
(2\times10^{-3},2\times10^{-4}).
\]
For limited-angle CT, we use $T=80$, a learning rate of
$5\times10^{-4}$, and
\[
(\beta_{\max},\beta_{\min})
=
(5\times10^{-4},5\times10^{-5}).
\]
No training data or pretrained weights are used, and one reconstruction
is produced per slice.

\subsection{Implementation of the compared methods}
\label{app:ct_baselines}
All methods are evaluated on the same $24$ slices with the same projection angles, and all reconstructions are
mapped to the same $[0,1]$ normalization before computing metrics.
\begin{itemize}
    \item \textbf{Neural-prior methods} (DIP, ASeqDIP, Self-Guided DIP, SDI) are run with their released code, using
    the same forward operators and measurements as TRACE.
    \item \textbf{IRCNN} is a plug-and-play method whose CNN denoiser is pretrained on natural images; it is used
    without retraining on CT data.
    \item \textbf{RED-CNN} and \textbf{EDCNN} are supervised networks that map the filtered back-projection (FBP) of
    the measurements to the full-dose image. A separate model is trained for each task on slices from four AAPM
    patients (L143, L192, L286 and L310) that are disjoint from the test patients, with the measurements simulated
    by the same forward operators as above.
    \item \textbf{DDS} is run with its released code and CT diffusion prior, pretrained on AAPM CT images, using the
    sampling settings of its released sparse-view evaluation script and changing only the number and range of
    projection angles to match our tasks. Since this prior operates on a fixed intensity window, each slice is
    mapped into that window before reconstruction and the result is mapped back to the per-slice $[0,1]$
    normalization for evaluation.
    \item \textbf{uDiG-DIP} is run with its released code and the same pretrained CT diffusion prior as DDS.
    \item \textbf{SITCOM} uses a diffusion model pretrained on natural images. The grayscale slice is
    replicated to three channels as input, and the first output channel is used for evaluation.
\end{itemize}

\subsection{Additional qualitative results}
\label{app:ct_qual}
Figures~\ref{fig:ct_qual_sparse} and~\ref{fig:ct_limitedangle_qual} show one additional slice for sparse-view
and limited-angle CT, respectively, complementing the example in Figure~\ref{fig:ct_qual}. As in the main text, the
out-of-domain SITCOM prior leads to background texture artifacts in both tasks and to structural distortions in
the limited-angle case. On the sparse-view slice in Figure~\ref{fig:ct_qual_sparse}, TRACE attains the best PSNR,
SSIM and LPIPS, with IRCNN second in PSNR and SSIM and EDCNN second in LPIPS. On the limited-angle slice in
Figure~\ref{fig:ct_limitedangle_qual}, TRACE attains the best PSNR and SSIM and the second-best LPIPS, while
Self-Guided DIP, ASeqDIP and DDS attain the second-best PSNR, the second-best SSIM and the best LPIPS, respectively.

\providecommand{\ctcelltrace}[3]{}
\renewcommand{\ctcelltrace}[3]{%
\begin{minipage}[t]{0.155\textwidth}
    \centering
    \includegraphics[width=\linewidth]{#1}\par
    \vspace{-0.04in}
    {\scriptsize\textcolor{traceblue}{\textbf{#2}}}\par
    \vspace{-0.04in}
    {\scriptsize\textbf{#3}}
\end{minipage}%
}

\begin{figure}[t]
\centering
\setlength{\tabcolsep}{1.5pt}
\renewcommand{\arraystretch}{1.0}

\textbf{CT Reconstruction: Sparse-view}
\par\vspace{0.45em}

\begin{tabular}{@{}cccccc@{}}

\ctcell{
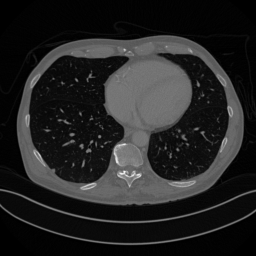}
{Ground Truth}
{}
&
\ctcell{
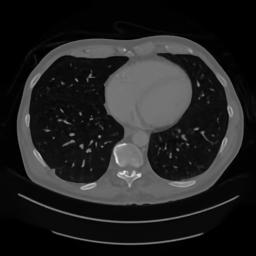}
{IRCNN}
{\fontsize{6.4}{7.5}\selectfont \underline{31.29}/\underline{0.889}/0.053}
&
\ctcell{
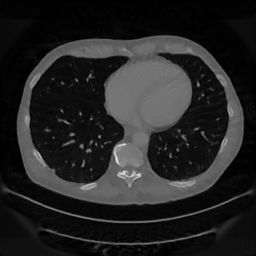}
{RED-CNN}
{\fontsize{6.4}{7.5}\selectfont 30.85/0.844/0.058}
&
\ctcell{
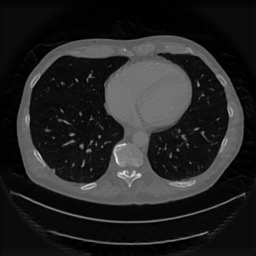}
{EDCNN}
{\fontsize{6.4}{7.5}\selectfont 30.50/0.823/\underline{0.047}}
&
\ctcell{
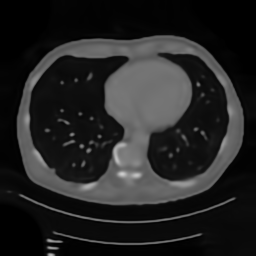}
{uDiG-DIP}
{\fontsize{6.4}{7.5}\selectfont 28.35/0.820/0.146}
&
\ctcell{
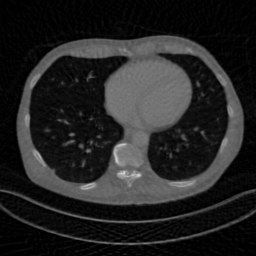}
{DDS}
{\fontsize{6.4}{7.5}\selectfont 30.77/0.834/0.129}
\\[0.23in]

\ctcell{
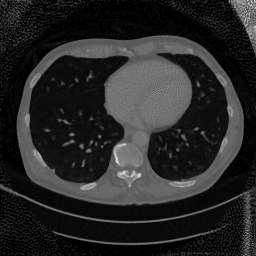}
{SITCOM}
{\fontsize{6.4}{7.5}\selectfont 23.27/0.733/0.199}
&
\ctcell{
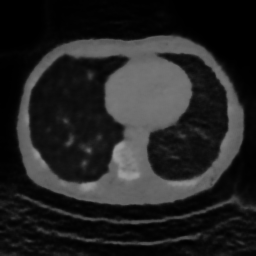}
{DIP}
{\fontsize{6.4}{7.5}\selectfont 25.89/0.671/0.251}
&
\ctcell{
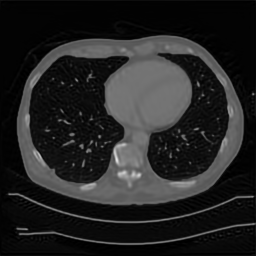}
{ASeqDIP}
{\fontsize{6.4}{7.5}\selectfont 28.92/0.861/0.068}
&
\ctcell{
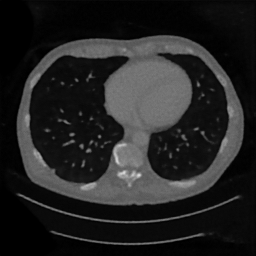}
{Self-Guided DIP}
{\fontsize{6.4}{7.5}\selectfont 30.11/0.848/0.090}
&
\ctcell{
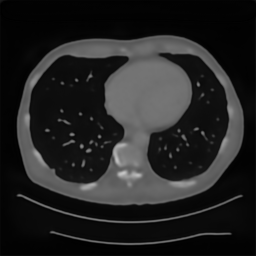}
{SDI}
{\fontsize{6.4}{7.5}\selectfont 29.71/0.862/0.113}
&
\ctcelltrace{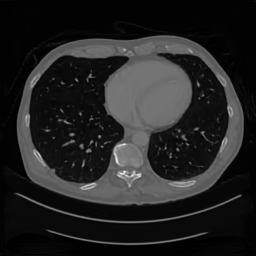}
{TRACE}
{\fontsize{6.4}{7.5}\selectfont 32.01/0.922/0.044}

\end{tabular}
\caption{
Representative sparse-view CT reconstructions (PSNR$\uparrow$ / SSIM$\uparrow$ / LPIPS$\downarrow$);
the best and second-best values are shown in \textbf{bold} and \underline{underlined}, respectively.
TRACE preserves anatomical structures more clearly and achieves the best PSNR, SSIM, and LPIPS in this example.
}
\label{fig:ct_qual_sparse}

\end{figure}

\providecommand{\ctcell}[3]{}
\renewcommand{\ctcell}[3]{%
\begin{minipage}[t]{0.155\textwidth}
    \centering
    \includegraphics[width=\linewidth]{#1}\par
    \vspace{-0.04in}
    {\scriptsize #2}\par
    \vspace{-0.04in}
    {\scriptsize #3}
\end{minipage}%
}

\providecommand{\ctcelltrace}[3]{}
\renewcommand{\ctcelltrace}[3]{%
\begin{minipage}[t]{0.155\textwidth}
    \centering
    \includegraphics[width=\linewidth]{#1}\par
    \vspace{-0.04in}
    {\scriptsize\textcolor{traceblue}{\textbf{#2}}}\par
    \vspace{-0.04in}
    {\scriptsize\textbf{#3}}
\end{minipage}%
}

\begin{figure}[t]
\centering
\setlength{\tabcolsep}{1.5pt}
\renewcommand{\arraystretch}{1.0}

\textbf{CT Reconstruction: Limited-angle}
\par\vspace{0.45em}

\begin{tabular}{@{}cccccc@{}}

\ctcell{
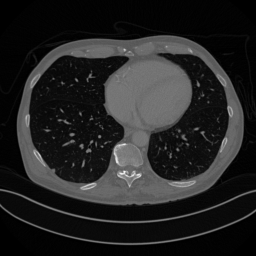}
{Ground Truth}
{}
&
\ctcell{
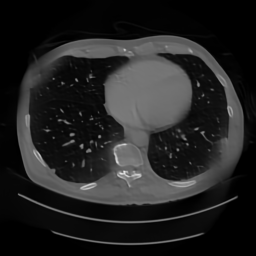}
{IRCNN}
{\fontsize{6.4}{7.5}\selectfont 27.62/0.828/0.094}
&
\ctcell{
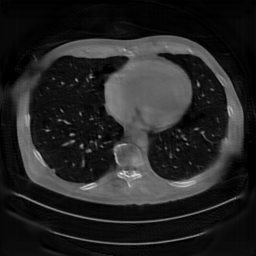}
{RED-CNN}
{\fontsize{6.4}{7.5}\selectfont 26.13/0.770/0.102}
&
\ctcell{
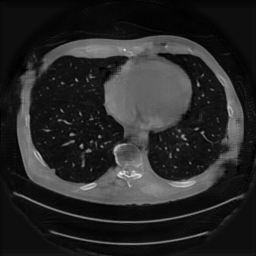}
{EDCNN}
{\fontsize{6.4}{7.5}\selectfont 24.87/0.709/0.091}
&
\ctcell{
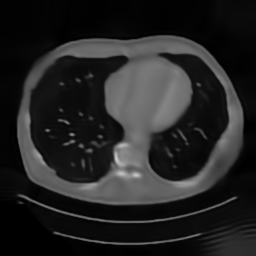}
{uDiG-DIP}
{\fontsize{6.4}{7.5}\selectfont 25.93/0.760/0.183}
&
\ctcell{
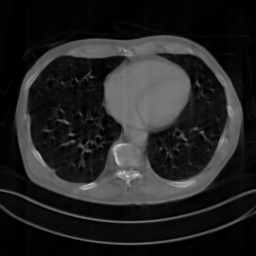}
{DDS}
{\fontsize{6.4}{7.5}\selectfont 26.17/0.807/\textbf{0.070}}
\\[0.23in]

\ctcell{
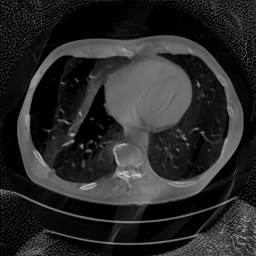}
{SITCOM}
{\fontsize{6.4}{7.5}\selectfont 20.51/0.567/0.241}
&
\ctcell{
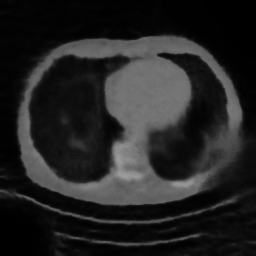}
{DIP}
{\fontsize{6.4}{7.5}\selectfont 24.66/0.664/0.292}
&
\ctcell{
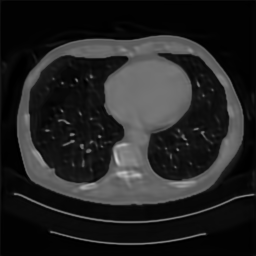}
{ASeqDIP}
{\fontsize{6.4}{7.5}\selectfont 28.66/\underline{0.852}/0.099}
&
\ctcell{
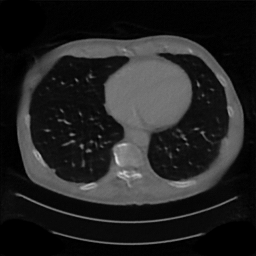}
{Self-Guided DIP}
{\fontsize{6.4}{7.5}\selectfont \underline{28.81}/0.850/0.092}
&
\ctcell{
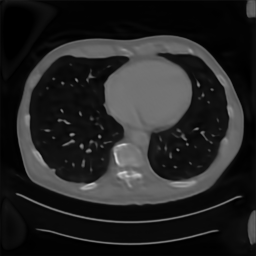}
{SDI}
{\fontsize{6.4}{7.5}\selectfont 27.99/0.848/0.094}
&
\ctcelltrace{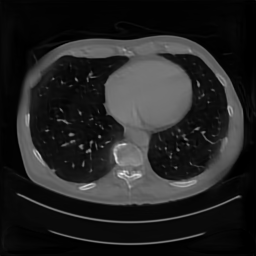}
{TRACE}
{\fontsize{6.4}{7.5}\selectfont 29.06/0.862{\mdseries /\underline{0.086}}}

\end{tabular}
\caption{
Additional qualitative results for limited-angle CT reconstruction (PSNR$\uparrow$ / SSIM$\uparrow$ / LPIPS$\downarrow$);
the best and second-best values are shown in \textbf{bold} and \underline{underlined}, respectively.
TRACE achieves the highest PSNR and SSIM and the second-best LPIPS in this example, showing improved
robustness to missing angular measurements and reduced limited-angle artifacts.
}
\label{fig:ct_limitedangle_qual}

\end{figure}

\newcommand{\qualcell}[3]{
\begin{minipage}[t]{0.235\linewidth}
    \centering
    \includegraphics[width=\linewidth]{#1}\\[-0.15em]
    {\scriptsize #2}\\[-0.15em]
    {\scriptsize #3}
\end{minipage}
}

\providecommand{\imgcell}[3]{}
\renewcommand{\imgcell}[3]{%
\begin{minipage}[t]{0.185\linewidth}
    \centering
    \includegraphics[width=\linewidth]{#1}\par
    \vspace{-0.04in}
    {\scriptsize #2}\par
    \vspace{-0.04in}
    {\scriptsize #3}
\end{minipage}%
}

\begin{figure}[t]
\centering
\setlength{\tabcolsep}{2pt}
\renewcommand{\arraystretch}{1.0}

\textbf{Additional Result: Inpainting with $50\%$ Missing Pixels}
\par\vspace{0.45em}

\begin{tabular}{@{}ccccc@{}}
\imgcell{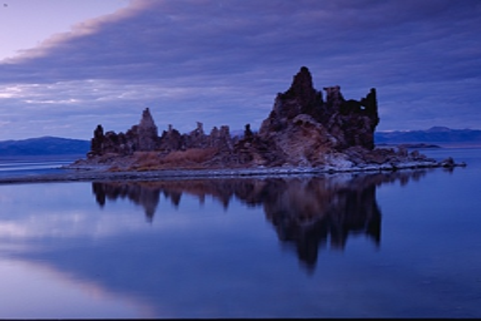}
{(a) GT}
{}
&
\imgcell{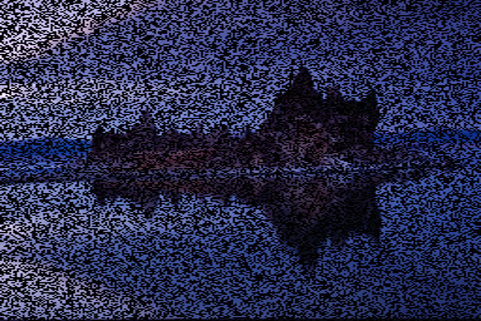}
{(b) Observation}
{11.00 / 0.118 / 0.9134}
&
\imgcell{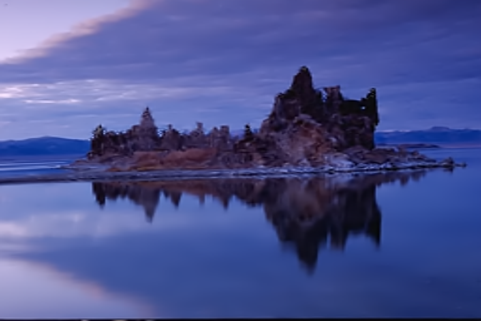}
{(c) PnP-DnCNN}
{35.71 / 0.962 / 0.0335}
&
\imgcell{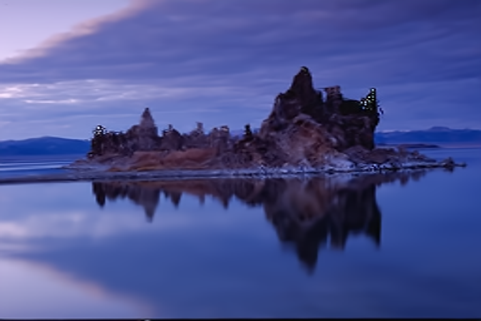}
{(d) PnP-FFDNet}
{33.90 / 0.958 / 0.0475}
&
\imgcell{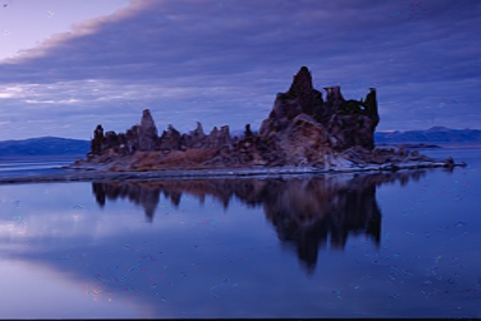}
{(e) SITCOM}
{35.14 / 0.941 / 0.0255}
\\[2.0em]

\imgcell{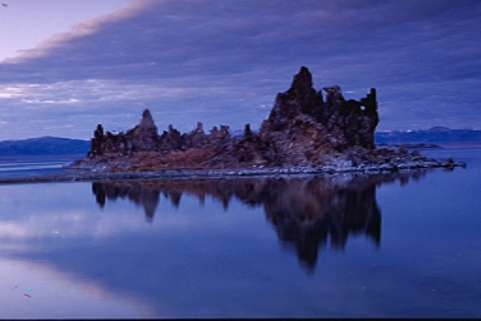}
{(f) DPS}
{33.74 / 0.934 / 0.0230}
&
\imgcell{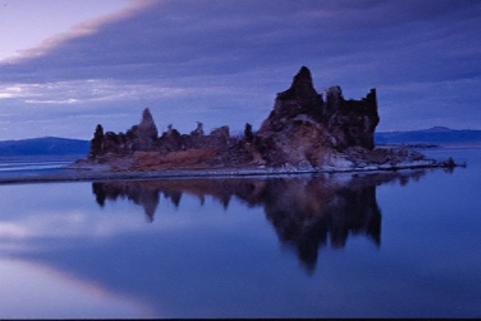}
{(g) DAPS}
{33.76 / 0.913 / 0.0297}
&
\imgcell{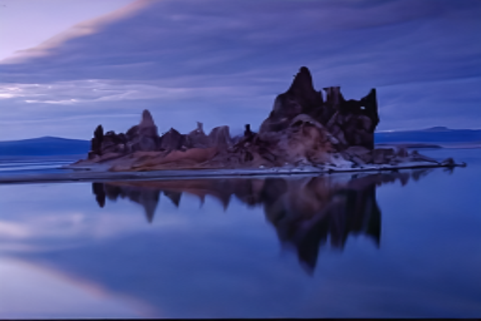}
{(h) RED-diff}
{34.30 / 0.928 / 0.0503}
&
\imgcell{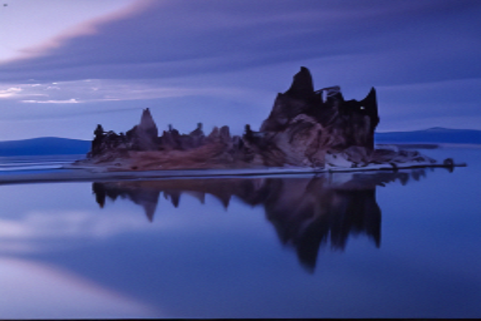}
{(i) PnP-DM}
{32.04 / 0.877 / 0.0716}
&
\imgcell{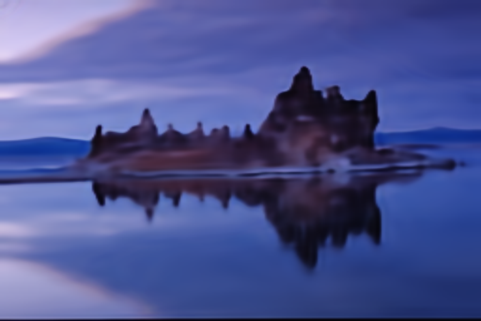}
{(j) DIP}
{31.49 / 0.871 / 0.1762}
\\[2.0em]

\imgcell{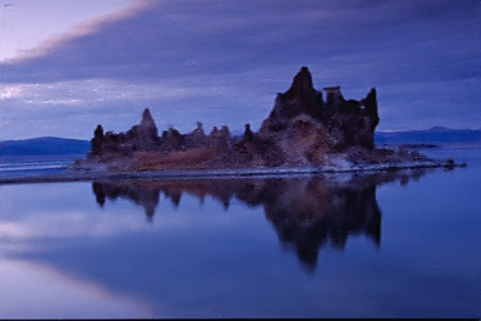}
{(k) ASeqDIP}
{34.92 / 0.940 / 0.0347}
&
\imgcell{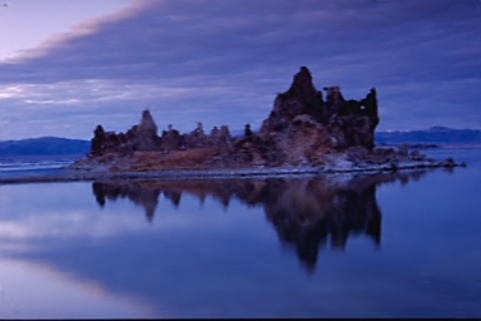}
{(l) Self-Guided}
{\underline{36.31} / \textbf{0.967} / \textbf{0.0161}}
&
\imgcell{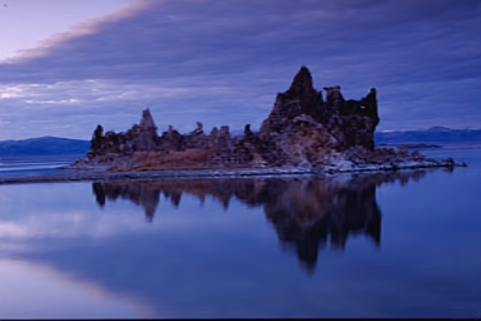}
{(m) SDI}
{36.02 / 0.961 / \underline{0.0176}}
&
\imgcell{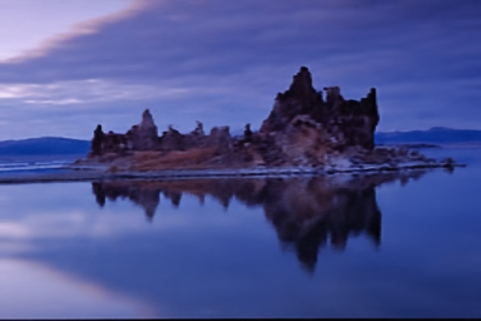}
{(n) UGoDIT}
{34.15/0.927/0.0586}
&
\imgcell{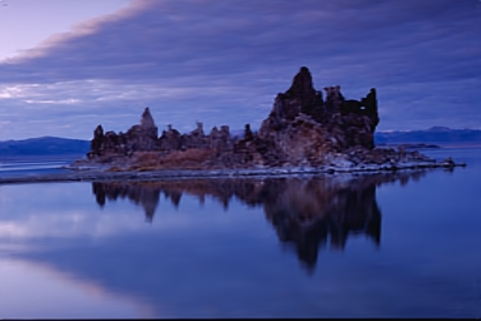}
{(o) \textbf{TRACE}}
{\textbf{36.63} / \underline{0.965} /0.0248}
\end{tabular}

\vspace{0.35em}
\caption{
Additional qualitative result for image inpainting with $50\%$ missing pixels.
The numbers below each image denote PSNR$\uparrow$ / SSIM$\uparrow$ /
LPIPS$\downarrow$; the best and second-best values are shown in \textbf{bold} and
\underline{underlined}, respectively. TRACE achieves the highest PSNR on this example and produces
a visually consistent reconstruction in the missing regions.
}
\label{fig:qual_inpainting}
\end{figure}

\providecommand{\imgcell}[3]{}
\renewcommand{\imgcell}[3]{%
\begin{minipage}[t]{0.185\linewidth}
    \centering
    \includegraphics[width=\linewidth]{#1}\par
    \vspace{-0.04in}
    {\scriptsize #2}\par
    \vspace{-0.04in}
    {\scriptsize #3}
\end{minipage}%
}

\begin{figure}[t]
\centering
\setlength{\tabcolsep}{2pt}
\renewcommand{\arraystretch}{1.0}

\textbf{Additional Result: Inpainting with $70\%$ Missing Pixels}
\par\vspace{0.45em}

\begin{tabular}{@{}ccccc@{}}
\imgcell{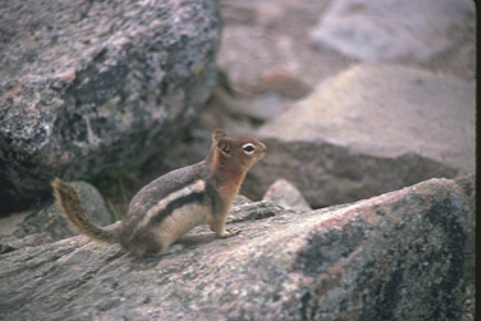}
{(a) GT}
{ }
&
\imgcell{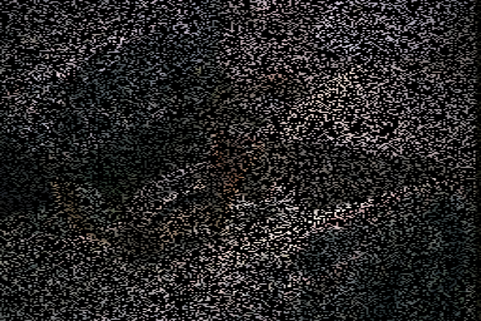}
{(b) Observation}
{8.39/0.057/0.9367}
&
\imgcell{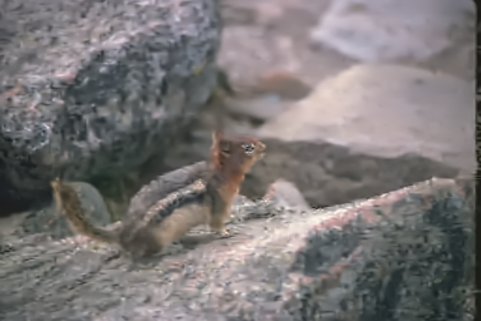}
{(c) PnP-DnCNN}
{29.06/0.878/0.1466}
&
\imgcell{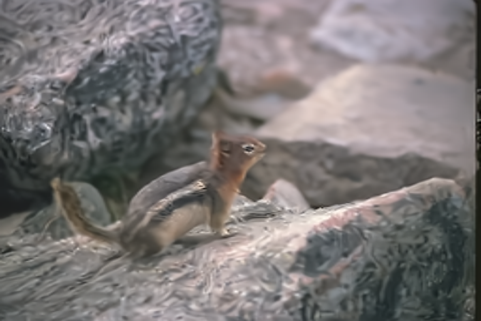}
{(d) PnP-FFDNet}
{\underline{29.39}/\underline{0.886}/0.1628}
&
\imgcell{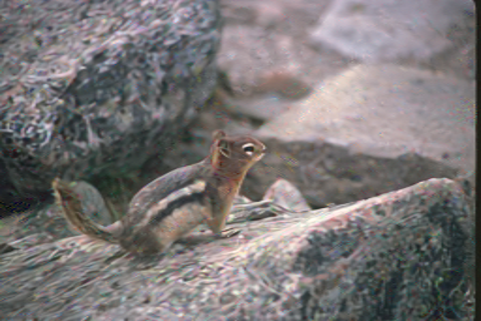}
{(e) SITCOM}
{27.82/0.844/0.1191}
\\[2.0em]

\imgcell{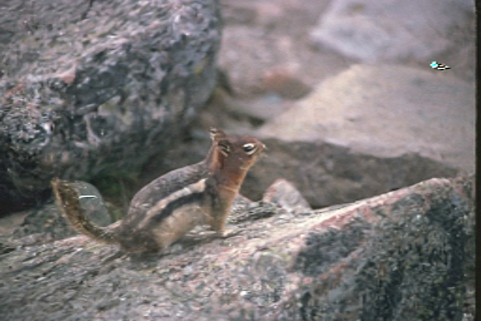}
{(f) DPS}
{26.88/0.831/0.1256}
&
\imgcell{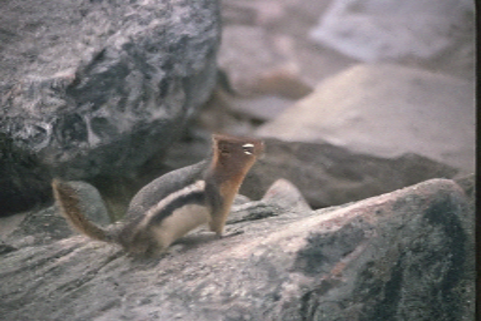}
{(g) DAPS}
{27.69/0.801/0.1453}
&
\imgcell{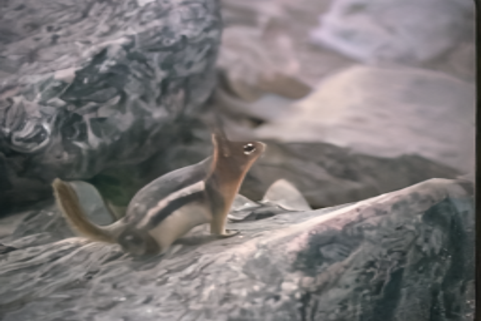}
{(h) RED-diff}
{28.32/0.825/0.2127}
&
\imgcell{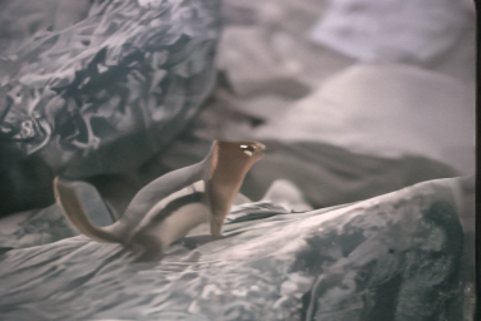}
{(i) PnP-DM}
{26.97/0.746/0.3038}
&
\imgcell{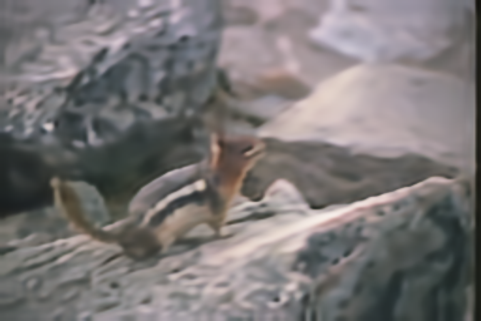}
{(j) DIP}
{26.97/0.737/0.4258}
\\[2.0em]

\imgcell{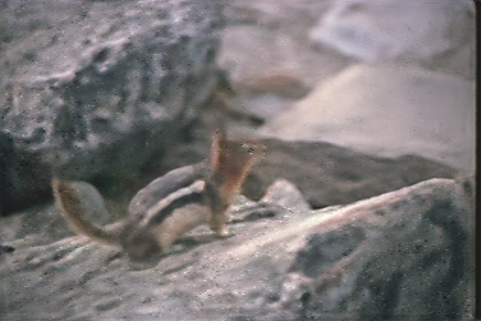}
{(k) ASeqDIP}
{27.75/0.794/0.2269}
&
\imgcell{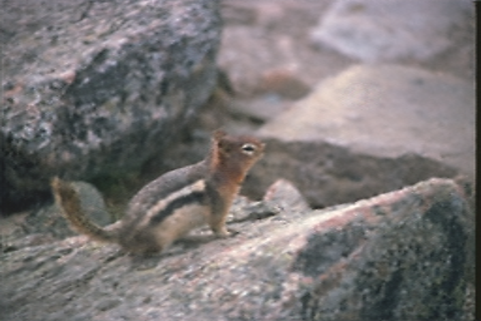}
{(l) Self-Guided}
{29.06/0.881/\textbf{0.0767}}
&
\imgcell{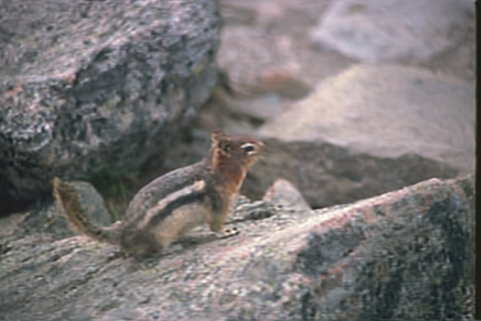}
{(m) SDI}
{28.75/0.878/\underline{0.0818}}
&
\imgcell{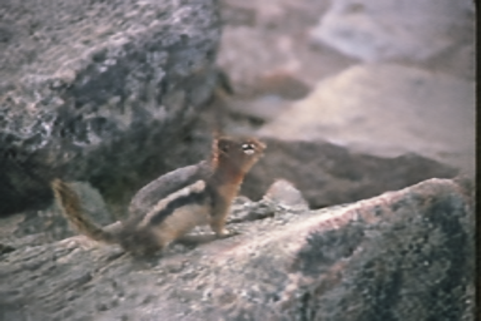}
{(n) UGoDIT}
{28.64/0.842/0.1393}
&
\imgcell{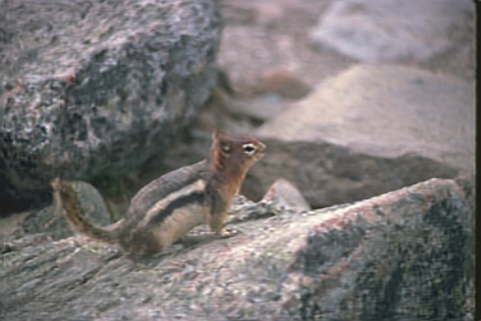}
{(o) \textbf{TRACE}}
{\textbf{29.88}/\textbf{0.888}/0.1026}
\end{tabular}

\vspace{0.35em}
\caption{
Representative image inpainting result under $70\%$ missing pixels.
The numbers below each image denote PSNR$\uparrow$ /
SSIM$\uparrow$ / LPIPS$\downarrow$; the best and second-best values are shown in \textbf{bold} and
\underline{underlined}, respectively.
TRACE achieves the highest PSNR and SSIM on this challenging example,
indicating improved pixel-level and structural fidelity.
}
\label{fig:inpaint70_example2}
\end{figure}

\section{Additional Details and Results for Natural Image}
\label{app:moreresults}

In this section, we provide additional qualitative results to complement the
main experiments. These examples further illustrate the behavior of TRACE across
different natural-image inverse problems, including image inpainting,
super-resolution, and deblurring. For all visual comparisons, the numbers below
each image denote PSNR$\uparrow$ / SSIM$\uparrow$ / LPIPS$\downarrow$.

\subsection{Image Inpainting}
\label{app:inpainting}

We provide an additional qualitative comparison for image inpainting with
$50\%$ random masking. Inpainting removes a large portion of pixels and therefore
requires the reconstruction method to infer missing structures from incomplete
observations. This provides a complementary test to deblurring and
super-resolution, where the measurements are corrupted but spatially complete.

Figures~\ref{fig:qual_inpainting} and \ref{fig:inpaint70_example2} show that TRACE produces a visually consistent
reconstruction in the missing regions. Compared with DIP and ASeqDIP, TRACE
reduces artifacts caused by directly fitting the incomplete observation. Compared
with stronger baselines such as Self-Guided DIP and SDI, TRACE achieves the
highest PSNR on these examples while maintaining competitive structural and
perceptual quality.

\subsection{Image Super-Resolution}
\label{app:sr}

We further evaluate image super-resolution under $\times2$ and $\times4$ downsampling
setting. Super-resolution requires recovering high-frequency details from a
low-resolution observation, making it sensitive to over-smoothing and unstable
texture generation.

Figures ~\ref{fig:qual_sr} and \ref{fig:qual_sr4} show additional SR examples. TRACE achieves the best
PSNR and LPIPS among the compared methods and matches the best SSIM. This
indicates that the proposed trajectory-constrained formulation improves
pixel-level fidelity while preserving perceptual quality. The result further
supports the conclusion in the main paper that TRACE is effective for recovering
fine image structures under resolution loss.

\providecommand{\imgcell}[3]{}
\renewcommand{\imgcell}[3]{%
\begin{minipage}[t]{0.185\linewidth}
    \centering
    \includegraphics[width=\linewidth]{#1}\par
    \vspace{-0.04in}
    {\scriptsize #2}\par
    \vspace{-0.04in}
    {\scriptsize #3}
\end{minipage}%
}

\begin{figure*}[t]
\centering
\setlength{\tabcolsep}{2pt}
\renewcommand{\arraystretch}{1.0}

\textbf{Additional Result: Super-resolution $\times2$}
\par\vspace{0.45em}

\begin{tabular}{@{}ccccc@{}}
\imgcell{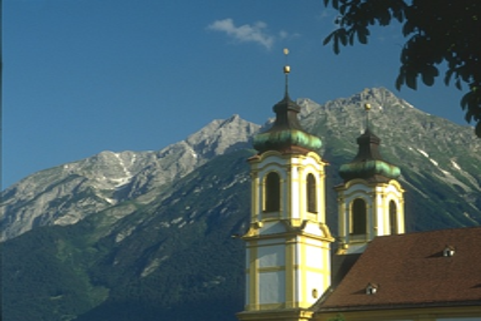}
{(a) GT}
{}
&
\imgcell{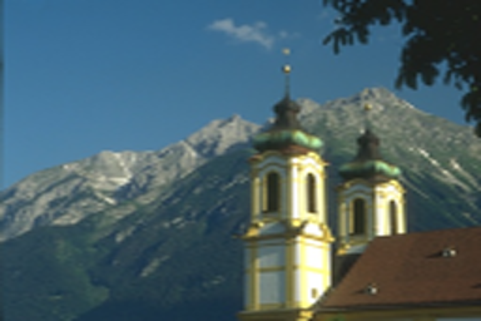}
{(b) Observation}
{30.01/0.903/0.1573}
&
\imgcell{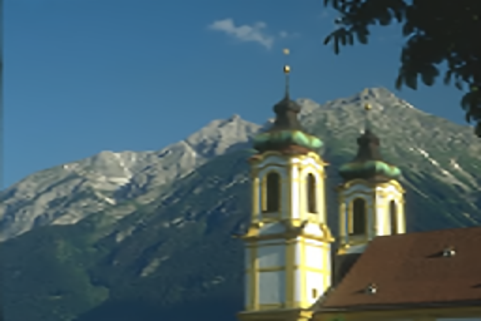}
{(c) PnP-DnCNN}
{30.63/\textbf{0.916}/0.0950}
&
\imgcell{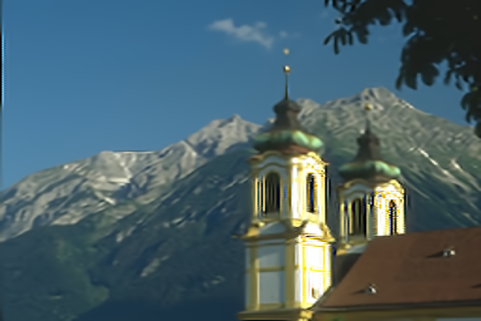}
{(d) PnP-FFDNet}
{26.94/0.902/0.1237}
&
\imgcell{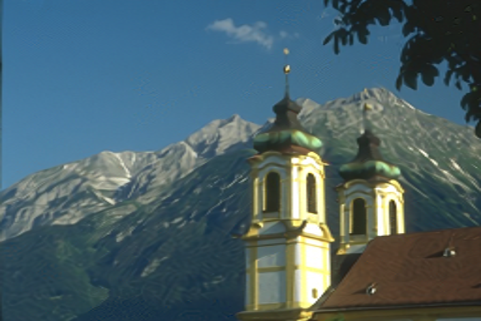}
{(e) SITCOM}
{\underline{30.82}/\underline{0.915}/0.0766}
\\[2.0em]

\imgcell{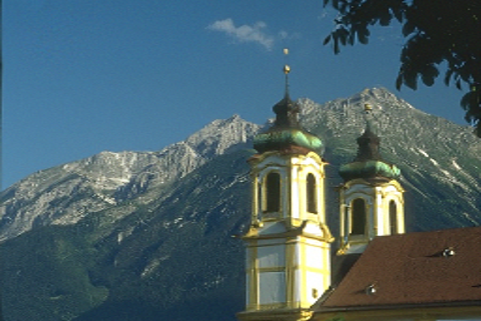}
{(f) DPS}
{27.61/0.826/\underline{0.0585}}
&
\imgcell{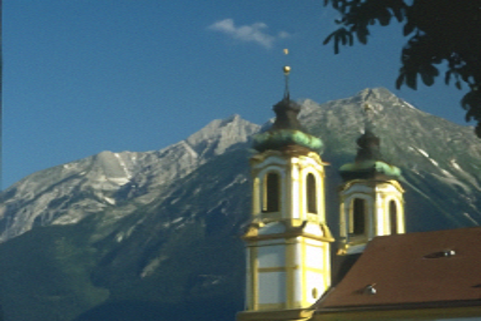}
{(g) DAPS}
{29.80/0.858/0.0615}
&
\imgcell{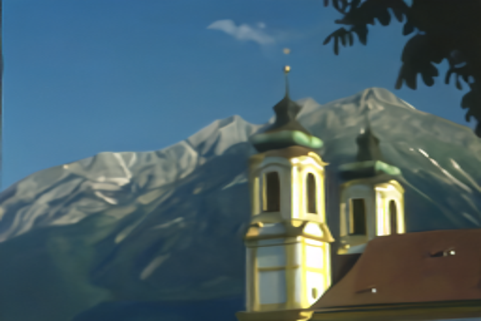}
{(h) RED-diff}
{29.40/0.853/0.1568}
&
\imgcell{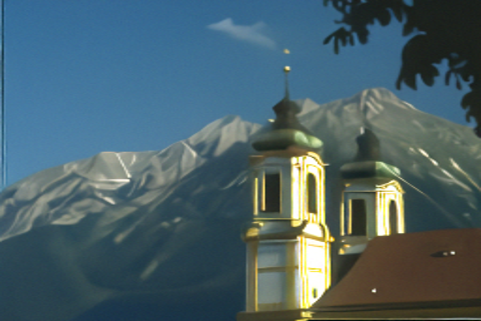}
{(i) PnP-DM}
{25.94/0.808/0.1539}
&
\imgcell{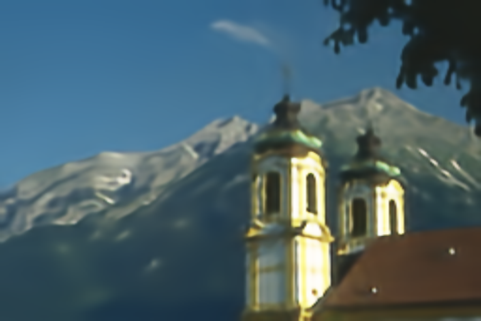}
{(j) DIP}
{28.09/0.808/0.2740}
\\[2.0em]

\imgcell{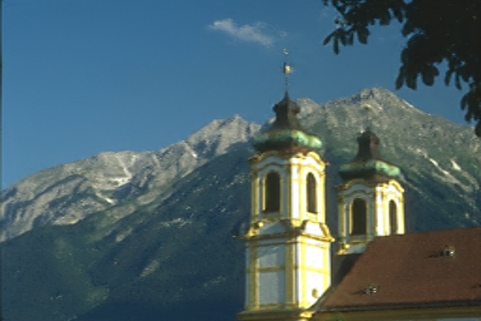}
{(k) ASeqDIP}
{29.32/0.843/0.2079}
&
\imgcell{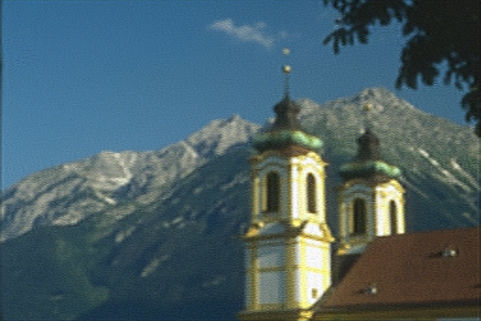}
{(l) Self-Guided}
{28.29/0.752/0.1111}
&
\imgcell{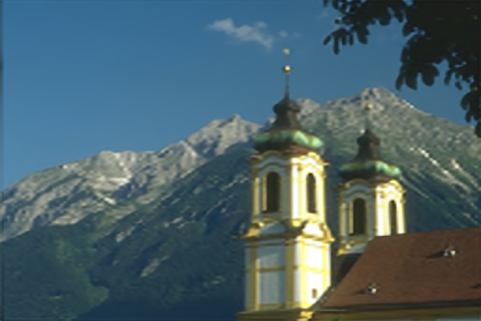}
{(m) SDI}
{30.74/\underline{0.915}/0.0853}
&
\imgcell{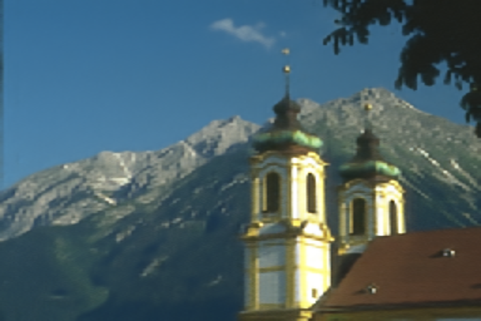}
{(n) UGoDIT}
{30.09/0.883/0.0836}
&
\imgcell{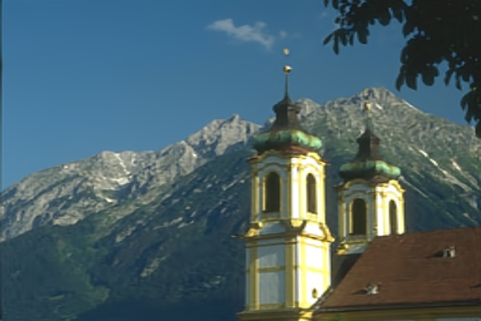}
{(o) \textbf{TRACE}}
{\textbf{31.03}/\underline{0.915}/\textbf{0.0439}}
\end{tabular}

\vspace{0.35em}
\caption{
Additional qualitative result for image super-resolution under $\times2$
downsampling. The numbers below each image denote PSNR$\uparrow$ /
SSIM$\uparrow$ / LPIPS$\downarrow$; the best and second-best values are shown in \textbf{bold} and
\underline{underlined}, respectively. TRACE achieves the best PSNR and LPIPS and matches the second-best SSIM
on this example.
}
\label{fig:qual_sr}
\end{figure*}

\providecommand{\imgcell}[3]{}
\renewcommand{\imgcell}[3]{%
\begin{minipage}[t]{0.185\linewidth}
    \centering
    \includegraphics[width=\linewidth]{#1}\par
    \vspace{-0.04in}
    {\scriptsize #2}\par
    \vspace{-0.04in}
    {\scriptsize #3}
\end{minipage}%
}

\begin{figure*}[t]
\centering
\setlength{\tabcolsep}{2pt}
\renewcommand{\arraystretch}{1.0}

\textbf{Additional Result: Super-resolution $\times4$}
\par\vspace{0.45em}

\begin{tabular}{@{}ccccc@{}}
\imgcell{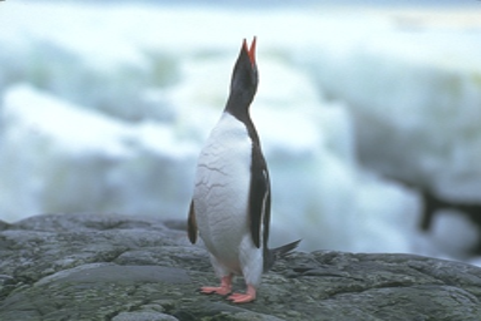}
{(a) GT}
{}
&
\imgcell{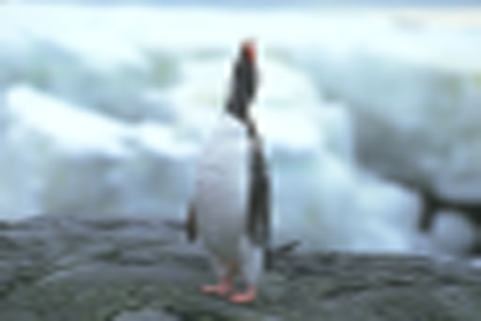}
{(b) Observation}
{29.35/0.863/0.2937}
&
\imgcell{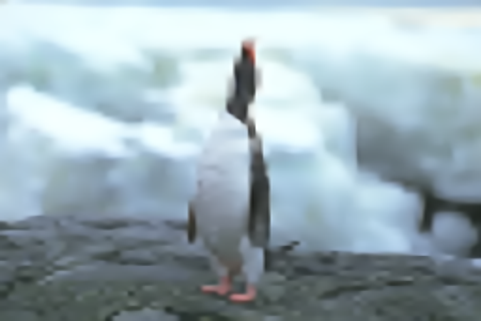}
{(c) PnP-DnCNN}
{29.62/0.873/0.2199}
&
\imgcell{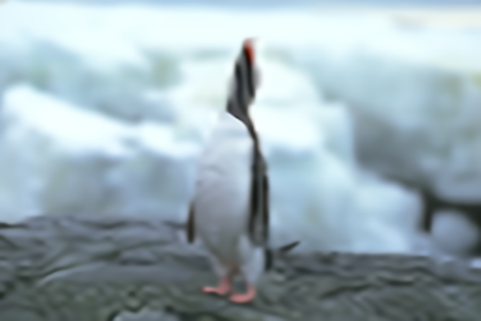}
{(d) PnP-FFDNet}
{29.71/\underline{0.878}/0.2269}
&
\imgcell{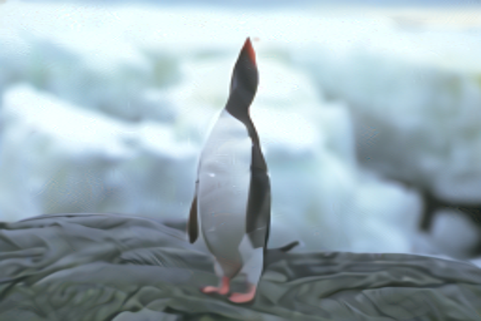}
{(e) SITCOM}
{30.21/0.864/\underline{0.1334}}
\\[1.0em]

\imgcell{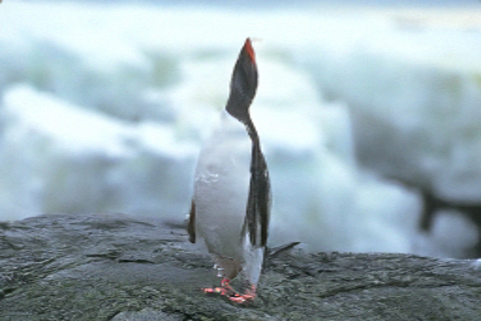}
{(f) DPS}
{27.10/0.798/\textbf{0.0743}}
&
\imgcell{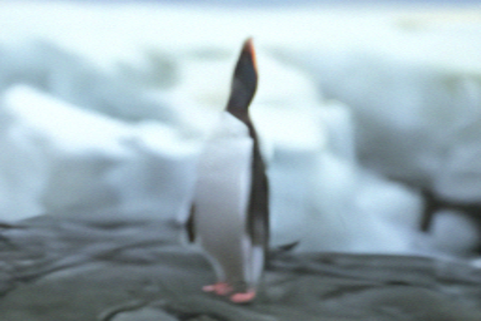}
{(g) DAPS}
{29.35/0.815/0.2048}
&
\imgcell{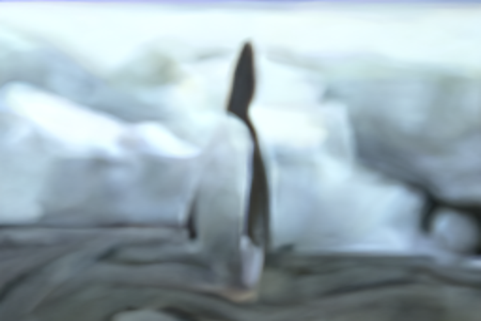}
{(h) RED-diff}
{27.70/0.783/0.3257}
&
\imgcell{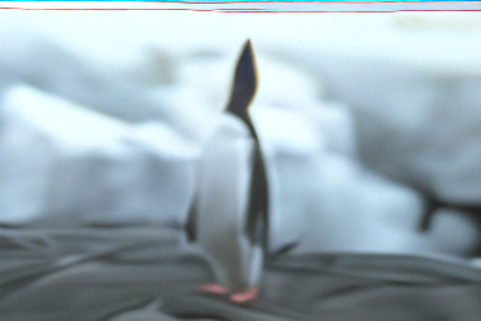}
{(i) PnP-DM}
{25.08/0.770/0.3239}
&
\imgcell{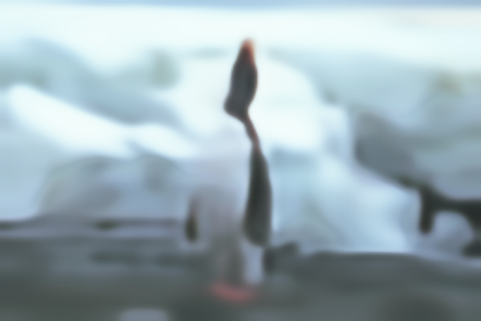}
{(j) DIP}
{27.83/0.781/0.3242}
\\[1.0em]

\imgcell{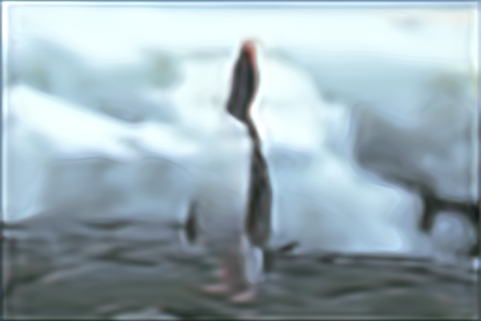}
{(k) ASeqDIP}
{27.74/0.804/0.3543}
&
\imgcell{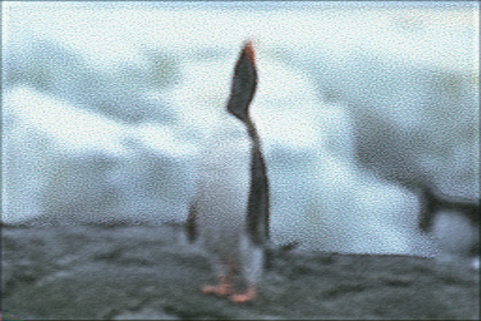}
{(l) Self-Guided}
{23.35/0.458/0.7414}
&
\imgcell{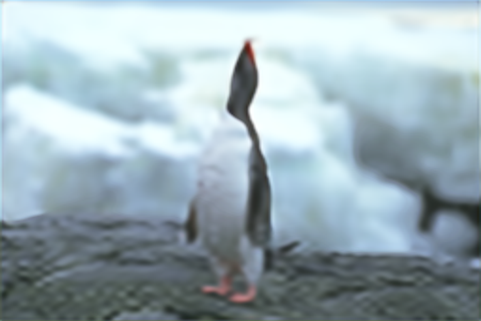}
{(m) SDI}
{\underline{30.23}/\underline{0.878}/0.2021}
&
\imgcell{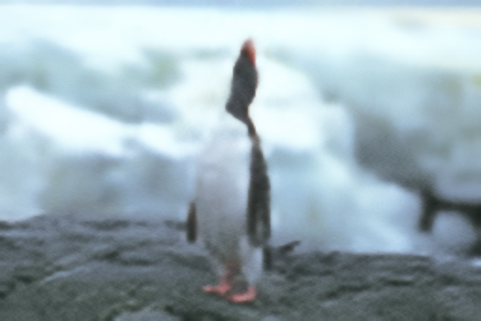}
{(n) UGoDIT}
{29.91/0.853/0.1630}
&
\imgcell{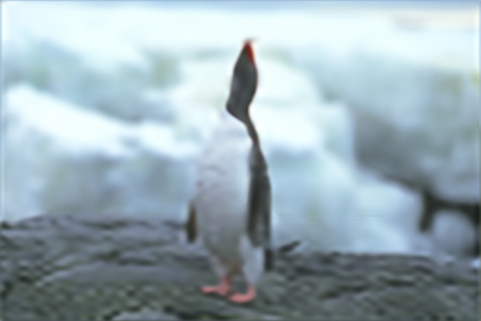}
{(o) \textbf{TRACE}}
{\textbf{30.29}/\textbf{0.880}/0.2134}
\end{tabular}

\vspace{0.35em}
\caption{
Additional qualitative result for image super-resolution under $\times4$
downsampling. The numbers below each image denote PSNR$\uparrow$ /
SSIM$\uparrow$ / LPIPS$\downarrow$; the best and second-best values are shown in \textbf{bold} and
\underline{underlined}, respectively. TRACE achieves the highest PSNR and SSIM on this example.
}
\label{fig:qual_sr4}
\end{figure*}

\subsection{Image Deblurring}
\label{app:deblurring}

We provide an additional qualitative comparison for motion deblurring. Deblurring
is highly ill-conditioned because high-frequency information is attenuated by
the blur operator, and iterative reconstruction methods can easily introduce
ringing or texture artifacts.

For both deblurring settings, we synthesize the observation as
\(
y = A(x_{\mathrm{true}})
\)
without adding extra noise, where \(x_{\mathrm{true}}\) is the resized RGB image
(default \(256\times256\)). The blur operator is implemented as channel-wise
(depthwise) convolution with shared kernels across three channels and symmetric
padding (\(p=10\), kernel size \(21\times21\)), outputs are clipped to \([0,1]\).

\textbf{Motion blur.}
We use a spatially invariant linear motion kernel of length 21 and direction
\(45^\circ\). The kernel is constructed by drawing a line through the kernel center
along the motion direction and normalizing it to unit sum.

\textbf{Nonlinear blur.}
For nonlinear deblurring, we follow DPS~\citep{chung2023dps} and use the BKSE forward
model~\citep{tran2021bkse}: the measurement is generated as $y=\mathcal{F}_\phi(x,k)+n$, where
$\mathcal{F}_\phi$ is a blur network pretrained on the GoPro dataset and $k$ is a randomly sampled latent blur
kernel, fixed across all test images, with additive Gaussian noise of standard deviation $0.05$.

Figure~\ref{fig:qual_motion} shows that TRACE achieves the second-best PSNR, SSIM, and
LPIPS on the displayed example. Compared with DIP and ASeqDIP, TRACE recovers
sharper structures and avoids severe blur-related artifacts. Compared with
Self-Guided DIP, UGoDIT, and SDI, TRACE further improves both distortion-based
metrics and perceptual quality. These additional results are consistent with the
main-paper observation that trajectory regularization improves stability and
reconstruction reliability in ill-conditioned inverse problems.

\providecommand{\imgcell}[3]{}
\renewcommand{\imgcell}[3]{%
\begin{minipage}[t]{0.185\linewidth}
    \centering
    \includegraphics[width=\linewidth]{#1}\par
    \vspace{-0.04in}
    {\scriptsize #2}\par
    \vspace{-0.04in}
    {\scriptsize #3}
\end{minipage}%
}

\begin{figure*}[t]
\centering
\setlength{\tabcolsep}{2pt}
\renewcommand{\arraystretch}{1.0}

\textbf{Additional Result: Motion Deblurring}
\par\vspace{0.45em}

\begin{tabular}{@{}ccccc@{}}
\imgcell{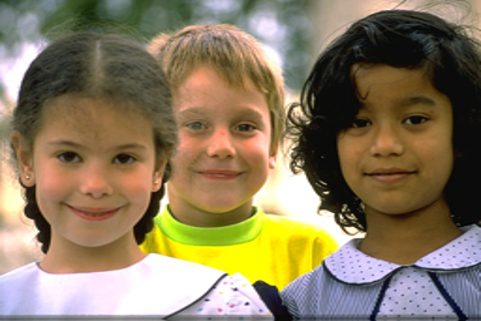}
{(a) GT}
{}
&
\imgcell{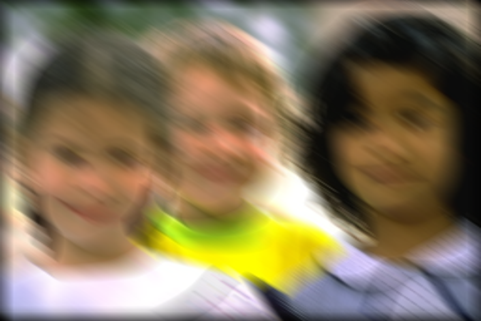}
{(b) Observation}
{18.85/0.581/0.4676}
&
\imgcell{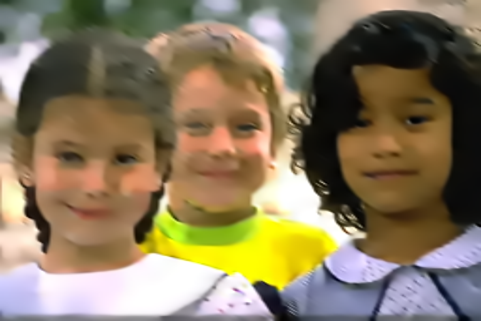}
{(c) PnP-DnCNN}
{26.69/0.812/0.1768}
&
\imgcell{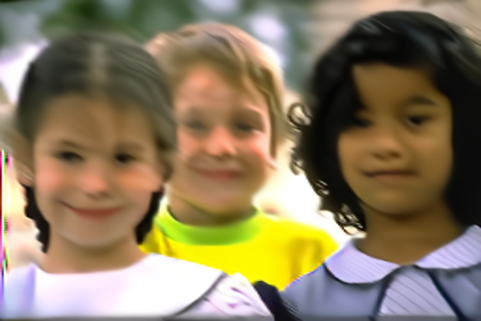}
{(d) PnP-FFDNet}
{25.65/0.804/0.2212}
&
\imgcell{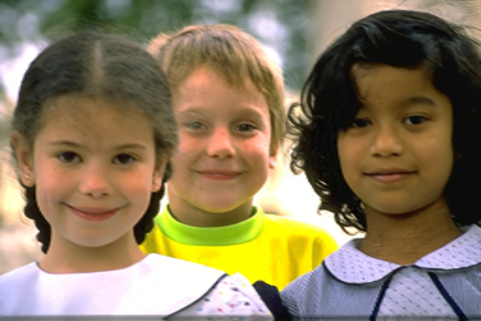}
{(e) SITCOM}
{\textbf{34.12}/\textbf{0.952}/\textbf{0.0290}}
\\[2.0em]

\imgcell{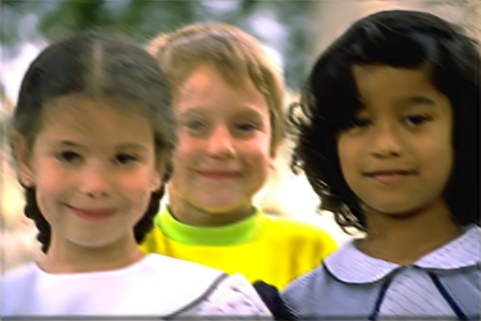}
{(f) DPS}
{29.30/0.878/0.1062}
&
\imgcell{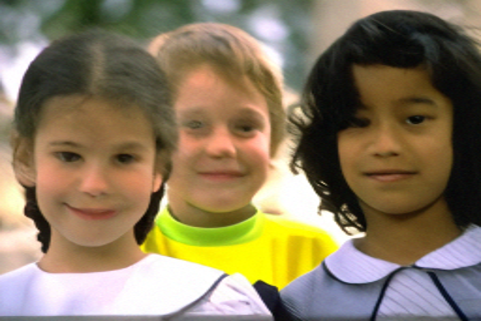}
{(g) DAPS}
{27.59/0.827/0.0972}
&
\imgcell{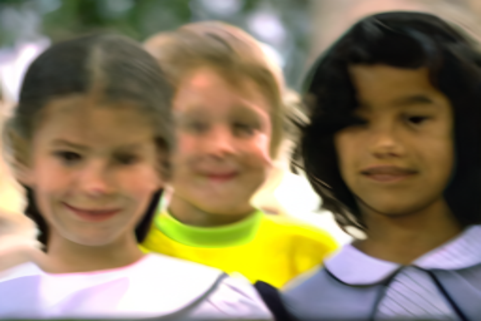}
{(h) RED-diff}
{24.76/0.761/0.2508}
&
\imgcell{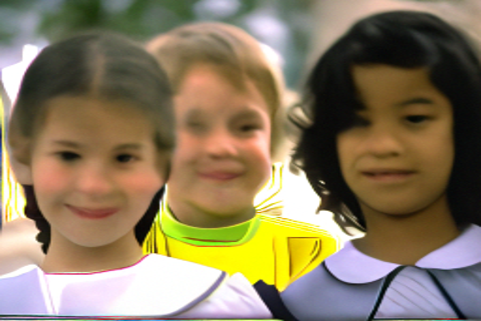}
{(i) PnP-DM}
{23.67/0.758/0.2169}
&
\imgcell{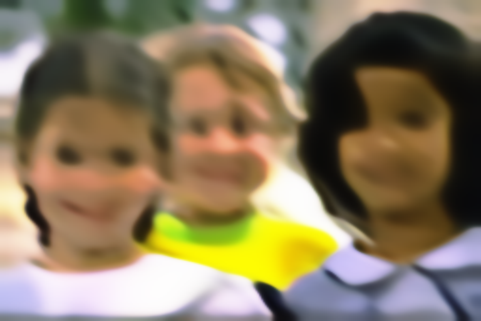}
{(j) DIP}
{22.11/0.643/0.4484}
\\[2.0em]

\imgcell{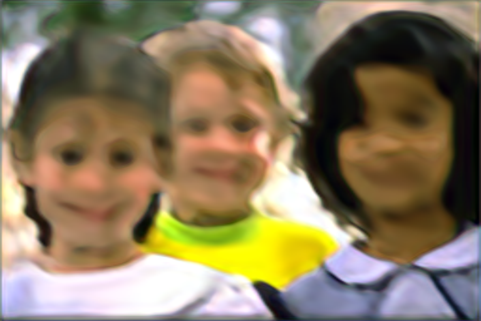}
{(k) ASeqDIP}
{23.39/0.687/0.3456}
&
\imgcell{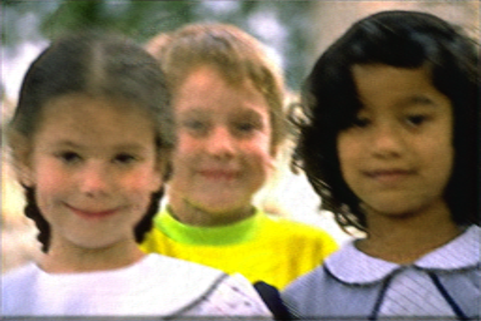}
{(l) Self-Guided}
{27.24/0.829/0.1555}
&
\imgcell{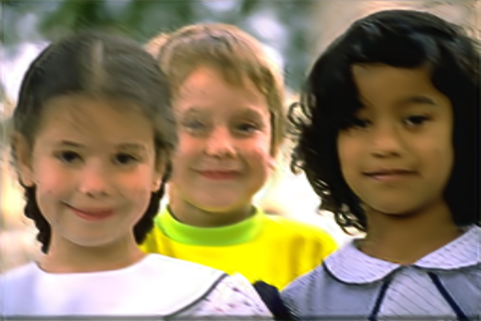}
{(m) SDI}
{28.65/0.872/0.1263}
&
\imgcell{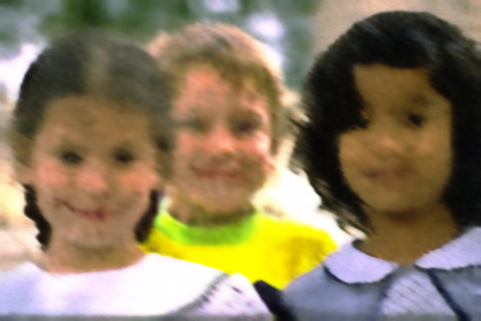}
{(n) UGoDIT}
{24.23/0.727/0.2896}
&
\imgcell{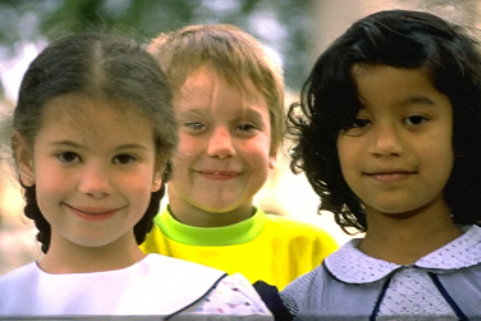}
{(o) \textbf{TRACE}}
{\underline{30.37}/\underline{0.892}/\underline{0.0329}}
\end{tabular}

\vspace{0.35em}
\caption{
Additional qualitative result for motion deblurring.
The numbers below each image denote PSNR$\uparrow$ / SSIM$\uparrow$ /
LPIPS$\downarrow$; the best and second-best values are shown in \textbf{bold} and \underline{underlined}, respectively. TRACE achieves the second-best PSNR, SSIM, and LPIPS on this
example.
}
\label{fig:qual_motion}
\end{figure*}

\section{Additional Ablations and Computational Analysis}
\label{app:additional_ablation}
This section provides additional analyses of practical design choices and computational cost in TRACE.
We first study how a fixed optimization budget is allocated between the number of outer trajectory states and the number of inner network updates.
We then examine the effect of using a pretrained network initialization instead of the random initialization adopted in the main experiments.
Finally, we report runtime, memory usage, and network-evaluation cost for TRACE and the compared methods.

\subsection{Outer and Inner Optimization Budget}
\label{app:TK}

TRACE contains two optimization scales: the number of outer trajectory
states $T$ and the number of inner network updates $K$ performed at
each state.
To study their relative allocation independently of the total number of
network updates, we fix
\[
N = TK = 6000
\]
and vary $T$ and $K$ accordingly.

We evaluate three Set14 images under $50\%$ missing-pixel inpainting.
Table~\ref{tab:iteration_ablation} shows that the reconstruction quality
depends on how the fixed optimization budget is distributed between the
outer and inner loops.
Increasing $T$ from $10$ to $40$ generally improves PSNR and SSIM,
with the best PSNR and SSIM obtained at $T=40$ and $K=150$.
Further increasing $T$ to $50$ or $60$, and consequently reducing the
number of inner updates per state, does not provide additional gains.
These results motivate the default choice $T=40$ and $K=150$ used in
the main experiments.
\begin{table}[t]
\centering
% \scriptsize
\setlength{\tabcolsep}{7pt}
\renewcommand{\arraystretch}{1.10}

\caption{
Sensitivity to the allocation of a fixed optimization budget between
the number of outer trajectory states $T$ and inner updates $K$.
The total number of network updates is fixed at $N=TK=6000$.
Best and second-best values are shown in \textbf{bold} and
\underline{underlined}, respectively.
}
\label{tab:iteration_ablation}

\begin{tabular}{@{}ccccc@{}}
\toprule
\rowcolor{headergray}
\textbf{$T$}
&
\textbf{$K$}
&
\textbf{PSNR $\uparrow$}
&
\textbf{SSIM $\uparrow$}
&
\textbf{LPIPS $\downarrow$}
\\
\midrule

10 & 600
& 31.65 & 0.946 & 0.0286
\\

20 & 300
& 32.26 & \underline{0.951} & \textbf{0.0252}
\\

30 & 200
& 31.92 & 0.946 & 0.0315
\\

\rowcolor{tracebg}
40 & 150
& \textbf{32.65} & \textbf{0.952} & \underline{0.0259}
\\

50 & 120
& \underline{32.43} & 0.948 & 0.0273
\\

60 & 100
& 32.21 & 0.948 & 0.0323
\\

\bottomrule
\end{tabular}
\end{table}

% \subsection{Perturb Inner Iteration}
% Inspired by the DIP-based methods \cite{ulyanov2018dip}, we add the additive noise $\boldsymbol{\epsilon}_{in}\sim\mathcal{N}(0,I)$ with noise level $\sigma_{in}$ into the network input during optimization. 

% $x\_t\_preturbed = x_t + torch.zeors\_like(x\_t).normal(std=\sigma_{in})$\\ $x\_star = self.net(x\_t\_preturbed)$

% \section{Robustness under Increasing Degradation}\label{app:degradation}

% We evaluate whether the benefit of trajectory regularization becomes more pronounced as the inverse problem becomes more challenging. 
% Specifically, we vary the level of degradation for different tasks: Gaussian noise level for denoising, number of projection views for sparse-view CT.

% Figure~\ref{fig:degradation_robustness} shows reconstruction quality and average trajectory variation under increasing degradation. 
% As the degradation becomes more severe, baseline methods exhibit noticeable performance degradation and increased trajectory variation, indicating unstable reconstruction framework. 
% In contrast, the proposed method maintains more stable trajectories and degrades more gracefully.

% \begin{figure}[t]
% \centering
% % \includegraphics[width=0.9\linewidth]{figs/degradation_robustness.pdf}
% \caption{
% Robustness under increasing degradation. 
% Left: reconstruction quality. 
% Right: average trajectory variation. 
% The proposed method maintains more stable trajectories and degrades more gracefully under challenging conditions.
% }
% \label{fig:degradation_robustness}
% \end{figure}

\subsection{Effect of Pretrained Initialization}
\label{app:pretrain}

The main TRACE experiments use random initialization and therefore do
not require a pretrained network.
As an additional experiment, we examine whether the same trajectory
formulation can also be used when a pretrained initialization is
available.

Specifically, we initialize the network parameters using the publicly
released unconditional $256\times256$ Guided Diffusion checkpoint
(\texttt{256x256\_diffusion\_uncond.pt}) from OpenAI
\citep{openai_guided_diffusion}, and then optimize the network using the
same TRACE procedure.
The pretrained model is used only to initialize the parameters; it is
not kept fixed as a reconstruction prior during optimization.

Table~\ref{tab:pretrained_prior} reports the average results on three
Set14 images under $50\%$ missing-pixel inpainting.
Pretrained initialization increases PSNR from $32.59$ to $32.75$ dB and
SSIM from $0.949$ to $0.952$, while LPIPS changes from $0.0279$ to
$0.0301$.
The differences are modest, indicating that pretrained initialization
can be incorporated into TRACE when available, while the main method
does not depend on it.

\begin{table}[t]
\centering
% \scriptsize
\setlength{\tabcolsep}{7pt}
\renewcommand{\arraystretch}{1.10}

\caption{
Effect of network initialization on $50\%$ missing-pixel inpainting.
The pretrained model is used only for initialization before TRACE
optimization and is not used as a fixed reconstruction prior.
}
\label{tab:pretrained_prior}

\begin{tabular}{@{}lccc@{}}
\toprule
\rowcolor{headergray}
\textbf{Initialization}
&
\textbf{PSNR $\uparrow$}
&
\textbf{SSIM $\uparrow$}
&
\textbf{LPIPS $\downarrow$}
\\
\midrule

\rowcolor{tracebg}
Random (default)
& 32.59
& 0.949
& \textbf{0.0279}
\\

Pretrained
& \textbf{32.75}
& \textbf{0.952}
& 0.0301
\\

\bottomrule
\end{tabular}
\end{table}

\ifdefined\providecolor\providecolor{tracebg}{RGB}{232,240,254}\fi
\providecommand{\rowcolor}[2][]{}% no-op if colortbl is not loaded
\begin{table}[t]
\centering
\small
\setlength{\tabcolsep}{4pt}
\renewcommand{\arraystretch}{1.08}
\caption{
Computational cost and reconstruction quality on SR $4\times$
(BSDS68, mean over 5 images $\times$ 3 random seeds).
All methods run one image at a time on a single, otherwise idle
NVIDIA RTX PRO 6000 Blackwell GPU with the configuration used for the main results.
\emph{Time} is the per-image reconstruction time (mean\,$\pm$\,std over the 15 runs),
CUDA-synchronised, excluding model loading, I/O and metric logging;
\emph{GPU-h} is the corresponding GPU time for 100 images.
\emph{Fwd / Bwd} counts forward / backward passes through the network;
\emph{Params} is the size of the network being optimised or evaluated.
Diffusion-prior methods additionally rely on a natural-image pretrained model,
whose training cost is not included.
Best and second-best quality values are in bold and underlined, respectively.
}
\label{tab:runtime_SR4x}
\begin{adjustbox}{max width=\linewidth}
\begin{tabular}{lrrrrrccc}
\toprule
Method
& Params (M)
& Mem.\ (GB)
& Fwd / Bwd
& Time (s)
& GPU-h
& PSNR$\uparrow$
& SSIM$\uparrow$
& LPIPS$\downarrow$
\\
\midrule
\multicolumn{9}{l}{\textit{Untrained networks (no training data)}} \\
UGoDIT
& 2.2
& 0.48
& 600 / 600
& 3.3\,$\pm$\,0.1
& 0.09
& 28.33
& 0.736
& 0.421
\\
ASeqDIP
& 1.0
& 0.37
& 1200 / 1200
& 4.3\,$\pm$\,0.0
& 0.12
& 22.93
& 0.596
& 0.570
\\
DIP
& 2.2
& 0.55
& 3000 / 3000
& 16.9\,$\pm$\,0.2
& 0.47
& 30.00
& 0.830
& 0.219
\\
Self-Guided DIP
& 69.2
& 1.89
& 3600 / 3600
& 33.8\,$\pm$\,0.1
& 0.94
& 23.30
& 0.619
& 0.614
\\
SDI
& 88.1
& 1.66
& 6000 / 6000
& 76.3\,$\pm$\,0.1
& 2.12
& \underline{30.16}
& 0.831
& 0.262
\\
\rowcolor{tracebg}
\textbf{TRACE (ours)}
& 88.1
& 1.66
& 6000 / 6000
& 76.3\,$\pm$\,0.1
& 2.13
& \textbf{30.46}
& \textbf{0.852}
& 0.261
\\
\midrule
\multicolumn{9}{l}{\textit{Pretrained CNN denoiser + plug-and-play}} \\
PnP-FFDNet
& 0.9
& 0.04
& 200 / 0
& 0.3\,$\pm$\,0.0
& 0.01
& 29.41
& \underline{0.851}
& 0.295
\\
PnP-DnCNN
& 0.7
& 0.09
& 200 / 0
& 0.5\,$\pm$\,0.0
& 0.01
& 30.10
& 0.848
& 0.268
\\
\midrule
\multicolumn{9}{l}{\textit{Pretrained diffusion prior}} \\
RED-diff
& 93.6
& 0.65
& 1000 / 0
& 13.8\,$\pm$\,0.6
& 0.38
& 27.70
& 0.736
& 0.391
\\
DAPS
& 93.6
& 0.93
& 1000 / 0
& 18.8\,$\pm$\,0.6
& 0.52
& 29.33
& 0.780
& 0.267
\\
SITCOM
& 93.6
& 4.42
& 600 / 600
& 19.7\,$\pm$\,0.2
& 0.55
& 30.13
& 0.829
& \underline{0.186}
\\
DPS
& 93.6
& 2.32
& 1000 / 1000
& 26.0\,$\pm$\,0.7
& 0.72
& 26.98
& 0.738
& \textbf{0.165}
\\
PnP-DM
& 93.6
& 0.94
& 4121 / 0
& 58.1\,$\pm$\,0.9
& 1.61
& 26.85
& 0.738
& 0.364
\\
\bottomrule
\end{tabular}
\end{adjustbox}
\end{table}

\subsection{Computational Cost}
\label{app:runtime}

We additionally compare the computational cost of TRACE with the
methods used in the natural-image experiments.
Table~\ref{tab:runtime_SR4x} reports network size, peak GPU memory,
forward/backward evaluations, reconstruction time, and the corresponding
reconstruction quality for $4\times$ super-resolution.

TRACE requires substantially more computation than feed-forward
plug-and-play CNN methods and several pretrained diffusion solvers,
because the network is optimized directly for each measurement instance.
Among the untrained neural methods with comparable optimization scale,
TRACE has nearly the same measured runtime and memory usage as SDI
($76.6$ vs.\ $76.3$ seconds per image and $1.66$ GB for both), while
achieving higher PSNR and SSIM in this benchmark.
These results highlight the computational trade-off of the training-free
setting: TRACE avoids external model training but performs
instance-specific optimization at reconstruction time.

% \input{sections/Effect of Pretrained Initialization}
% Appendix paragraph referenced by the main SPC table (tab:spc_main). Carries every detail that was in the two
% former appendix-table captions. Put it in the appendix; the label app:spc_details is defined here.
\section{Experimental details for SPC on diffusion solvers}
\label{app:spc_details}
% \paragraph{Protocol.}
% For each solver, Base is the released implementation with the SPC weight set to zero, which reproduces the
% original algorithm exactly; +SPC is the same implementation with the SPC term switched on. Apart from the SPC
% term, the two runs are identical: same pretrained prior, hyperparameters, random seed, initial noise and
% measurements, and all images of a task are processed in one paired batch. One sample is drawn per image or slice.
% The SPC term is the proximal penalty $\tfrac{\beta_t}{2}\|z_t-z_{\mathrm{prev}}\|_2^2$ added to the solver's
% update, where $z_t$ is the solver's current clean-image estimate (for PnP-DM and RED-diff, its current iterate) and
% $z_{\mathrm{prev}}$ is the same quantity from the previous outer step. For DAPS and PnP-DM the penalty is
% additionally divided by $\sigma_t^2$, the solver's current annealing noise level, matching the scaling of their
% existing prior term; for RED-diff, $z_{\mathrm{prev}}$ is refreshed every $20$ iterations, since consecutive
% iterates are nearly identical. The weight follows a linear schedule over the $T$ outer steps,
% $\beta_t = w\,[\,1+(r-1)\,t/T\,]$, i.e., it goes from $w$ at the first step to about $w\,r$ at the last step.
\paragraph{Protocol.}
For each solver, Base is the released implementation with the SPC
weight set to zero, which reproduces the original algorithm exactly;
+SPC is the same implementation with the SPC term switched on. Apart
from the SPC term, the two runs are identical: they use the same
pretrained prior, hyperparameters, random seed, initial noise, and
measurements, and all images of a task are processed in one paired
batch. One sample is drawn per image or slice.

We index the solver's $T$ outer steps in execution order by
$s=0,\ldots,T-1$. For $s\geq1$, the SPC term is the proximal penalty
\[
    \frac{\beta_s}{2}
    \left\|z_s-z_{s-1}\right\|_2^2
\]
added to the solver's update, where $z_s$ is the solver's current
clean-image estimate and $z_{s-1}$ is the corresponding estimate from
the previous outer step. For PnP-DM and RED-diff, $z_s$ denotes the
current iterate. The SPC term is inactive at $s=0$ because no previous
outer-step estimate is available.

For DAPS and PnP-DM, the penalty is additionally divided by
$\sigma_s^2$, where $\sigma_s$ is the solver's current annealing noise
level, matching the scaling of their existing prior term. For RED-diff,
$z_{s-1}$ is replaced by an anchor that is refreshed every $20$
iterations because consecutive iterates are nearly identical.

The weight follows a linear schedule in execution order:
\[
    \beta_s
    =
    w\left[1+(r-1)\frac{s}{T}\right],
    \qquad s=0,\ldots,T-1.
\]
The schedule assigns $\beta_0=w$ at the beginning and
\[
    \beta_{T-1}
    =
    w\left[1+(r-1)\frac{T-1}{T}\right]
    \approx wr
\]
at the final step.

\paragraph{FFHQ deblurring.}
The task is to recover a $256\times256$ FFHQ face image from a measurement blurred by a $61\times61$ Gaussian
kernel (standard deviation $3.0$) and corrupted by additive Gaussian noise with standard deviation $0.05$ (images
scaled to $[-1,1]$). All solvers use the same unconditional FFHQ-256 diffusion model, and results are averaged over
the first $20$ images of the FFHQ-256 set (indices $0$--$19$). $(w,r)$ is $(0.06,\,2.0)$ for DPS, $(60,\,0.5)$ for
DAPS, $(0.3,\,0.8)$ for RED-diff, $(0.3,\,0.3)$ for PnP-DM and $(0.3,\,0.3)$ for SITCOM. For PnP-DM we set the likelihood parameter
$\tau=0.05$ instead of $0.01$ (for both Base and +SPC), since $0.01$ produces hatching artifacts in flat
background regions on this task. PSNR is computed by the
InverseBench~\cite{zheng2025inversebench} image evaluator (prediction and target clipped to $[0,1]$).

\paragraph{Multi-coil MRI.}
The task is to reconstruct a complex-valued $320\times320$ image from retrospectively undersampled, noiseless
$15$-coil k-space data (fastMRI knee; sensitivity maps provided with the data), using a random Cartesian
phase-encoding mask with a fully sampled central band of $8\%$ of the lines at $4\times$ ($80$ of $320$ lines) and
$4\%$ at $8\times$ ($40$ of $320$ lines); the same mask is used for all slices. All solvers share an unconditional
complex-valued diffusion model pretrained on multi-coil fastMRI knee images. The test set consists of $10$
consecutive slices (slices 005--014) of a single fastMRI knee volume (file1000033). $(w,r)$ is $(30,\,0.5)$ for
DAPS, $(0.5,\,0.6)$ for RED-diff, $(1.0,\,0.8)$ for PnP-DM and $(0.1,\,0.1)$ for SITCOM at both accelerations,
and $(10^{-3},\,1)$ at $4\times$ and $(10^{-4},\,10)$ at $8\times$ for DPS. At $8\times$, DPS uses
guidance scale $1$ instead of its default $10$ (for both Base and +SPC). PSNR is computed on the magnitude image,
with the dynamic range set to the maximum of the ground-truth magnitude (InverseBench MRI evaluator).

\end{document}